\documentclass[lettersize,journal]{cls/IEEEtran} % Comment this line out if you need a4paper
\usepackage{enumitem}
\usepackage[table,usenames,dvipsnames]{xcolor}      % color
\usepackage{amsmath,amssymb,amsfonts,amsthm,dsfont,bbm} % math
\usepackage[linesnumbered,ruled,noend]{algorithm2e}% algorithms
\usepackage{graphicx,tabularx,adjustbox}
\usepackage{booktabs}
\usepackage{multirow}
\usepackage[font={footnotesize}]{caption}   %onehalfspacing
\usepackage{subcaption}
\usepackage{thmtools,thm-restate}
\usepackage[breaklinks=true, colorlinks, bookmarks=true, citecolor=Black, urlcolor=Violet,linkcolor=Black]{hyperref}
\usepackage[table]{xcolor}
\usepackage{siunitx}
\usepackage[nameinlink,capitalise]{cleveref}
\crefname{line}{line}{lines}
\crefname{figure}{Fig.}{Figs.}
\Crefname{figure}{Fig.}{Figs.}
\crefname{equation}{Eq.}{Eqs.}
\Crefname{equation}{Eq.}{Eqs.}
\crefname{section}{Sec.}{Secs.}
\Crefname{section}{Sec.}{Secs.}
\crefname{definition}{Def.}{Defs.}
\Crefname{definition}{Def.}{Defs.}
\crefname{algorithm}{Alg.}{Algs.}
\Crefname{algorithm}{Alg.}{Algs.}
\crefname{subassumption}{Asm.}{Asms.}
\Crefname{subassumption}{Asm.}{Asms.}
\Crefname{problem}{Problem}{Problems}
\crefname{problem}{Problem}{Problems}

\usepackage[activate={true,nocompatibility},final,kerning=true,spacing=true,factor=1500,stretch=10,shrink=10]{microtype}

\usepackage{csquotes}
\usepackage[
	maxbibnames=99,
	maxcitenames=2,
	natbib=true,
	style=numeric-comp,
	backend=biber,
	sorting=none,
	giveninits=true,
	url=false,
	doi=false,
	eprint=false,
	isbn=false,
]{biblatex}

\usepackage{pbalance}

\newtheorem{theorem}{Theorem}

\newtheorem*{proposition*}{Proposition}

\newtheorem*{corollary*}{Corollary}

\theoremstyle{definition}
\newtheorem{definition}{Definition}[]
\newtheorem{assumption}{Assumption}
\newtheorem*{assumption*}{Assumption}
\newtheorem*{problem*}{Problem}
\newtheorem{problem}{Problem}
\theoremstyle{remark}

\newtheorem*{solution*}{Solution}
\newtheorem{example}{Example}
\newtheorem*{example*}{Example}

\newcommand{\scaleMathLine}[2][1]{\resizebox{#1\linewidth}{!}{$\displaystyle{#2}$}}
\SetKwComment{Comment}{/* }{ */}

\DeclareMathOperator{\argmin}{argmin}

\DeclareMathOperator{\rank}{rank}

\newcommand{\dof}{\textsc{d\scalebox{0.8}{o}f}\xspace}

\newcommand{\eg}{\emph{e.g.},\xspace}
\newcommand{\ie}{\emph{i.e.},\xspace}

\newcommand{\NEW}[1]{{#1}}
\newcommand{\REV}[1]{{#1}}
\newcommand{\REVV}[1]{{#1}}
\newcommand{\REVVV}[1]{{#1}}
\newcommand{\MATHREV}[1]{{{#1}}}
\newcommand{\MATHREVV}[1]{{{#1}}}
\newcommand{\BLUE}[1]{{\color{magenta}#1}}
\newcommand{\BETTER}[1]{{\boldsymbol{#1}}}
\newcommand{\RED}[1]{{\underline{\color{red}\boldsymbol{#1}}}}

\newcommand{\calA}{{\cal A}}
\newcommand{\calB}{{\cal B}}
\newcommand{\calC}{{\cal C}}

\newcommand{\calE}{{\cal E}}

\newcommand{\calG}{{\cal G}}

\newcommand{\calS}{{\cal S}}
\newcommand{\calT}{{\cal T}}
\newcommand{\calU}{{\cal U}}

\newcommand{\calW}{{\cal W}}
\newcommand{\calX}{{\cal X}}

\newcommand{\calZ}{{\cal Z}}

\newcommand{\bfa}{\mathbf{a}}
\newcommand{\bfb}{\mathbf{b}}
\newcommand{\bfc}{\mathbf{c}}
\newcommand{\bfd}{\mathbf{d}}
\newcommand{\bfe}{\mathbf{e}}
\newcommand{\bff}{\mathbf{f}}

\newcommand{\bfh}{\mathbf{h}}

\newcommand{\bfp}{\mathbf{p}}
\newcommand{\bfq}{\mathbf{q}}
\newcommand{\bfr}{\mathbf{r}}

\newcommand{\bft}{\mathbf{t}}
\newcommand{\bfu}{\mathbf{u}}
\newcommand{\bfv}{\mathbf{v}}
\newcommand{\bfw}{\mathbf{w}}
\newcommand{\bfx}{\mathbf{x}}
\newcommand{\bfy}{\mathbf{y}}
\newcommand{\bfz}{\mathbf{z}}

\newcommand{\bfalpha}{\boldsymbol{\alpha}}
\newcommand{\bfbeta}{\boldsymbol{\beta}}
\newcommand{\bfgamma}{\boldsymbol{\gamma}}

\newcommand{\bfpi}{\boldsymbol{\pi}}

\newcommand{\bfsigma}{\boldsymbol{\sigma}}
\newcommand{\bftau}{\boldsymbol{\tau}}

\newcommand{\bfomega}{\boldsymbol{\omega}}

\newcommand{\bfA}{\mathbf{A}}
\newcommand{\bfB}{\mathbf{B}}
\newcommand{\bfC}{\mathbf{C}}

\newcommand{\bfG}{\mathbf{G}}

\newcommand{\bfI}{\mathbf{I}}
\newcommand{\bfJ}{\mathbf{J}}

\newcommand{\bfM}{\mathbf{M}}

\newcommand{\bfR}{\mathbf{R}}

\newcommand{\bfZ}{\mathbf{Z}}

\newcommand{\bbE}{\mathbb{E}}

\newcommand{\bbG}{\mathbb{G}}

\newcommand{\bbR}{\mathbb{R}}

\newcommand{\bbV}{\mathbb{V}}

\newcommand{\scFK}{\textsc{fk}\xspace}
\newcommand{\scCC}{\textsc{cc}\xspace}
\newcommand{\scNN}{\textsc{nn}\xspace}
\newcommand{\scSIMD}{\textsc{simd}\xspace}
\newcommand{\FLASK}{\textbf{\texttt{FLASK}}\xspace}
\newcommand{\FLASKRRTC}{{\texttt{FLASK-RRTConnect}}\xspace}
\newcommand{\FLASKRRT}{{\texttt{FLASK-RRT}}\xspace}
\newcommand{\FLASKSST}{{\texttt{FLASK-SST}$^*$}\xspace}

\begin{document}
\title{Ultrafast Sampling-based Kinodynamic Planning via Differential Flatness} 
% \title{Ultrafast SIMD-Accelerated Sampling-based Kinodynamic Planning via Differential Flatness} % or CPU-Accelerated, we ant CPU + fast...
%\title{CPU-Accelerated Sampling-based Kinodynamic Planning in Milliseconds via Differential Flatness}

% \author{Author Names Omitted for Anonymous Review}
\author{Thai~Duong$^1$, Clayton W. Ramsey$^1$, Zachary Kingston$^{1*}$, Wil Thomason$^{1*}$, and Lydia E. Kavraki$^1$% <-this % stops a space
	%\thanks{We gratefully acknowledge support from ...}%
	\thanks{$^1$Thai Duong, Clayton W. Ramsey, Zachary Kingston and Wil Thomason are with the Department of Computer Science, Rice University, 
	Houston, TX, 77005, USA. Lydia E. Kavraki is with the Department of Computer Science, Rice University, Houston, TX, 77005 USA, and the Ken Kennedy Institute, Houston, TX, 77005 USA. Correspondence: thaiduong@rice.edu, kavraki@rice.edu.}
\thanks{$^*$Work by Zachary Kingston and Wil Thomason was done while they were at Rice University.}
    \thanks{Videos and software will be made available at \url{https://kavrakilab.org/flask}.}
        \thanks{Thai Duong and Lydia E. Kavraki were supported in part by NSF 2336612, NSF 2411219 and ERDC W912HZ2320003. Clayton W. Ramsey was supported by NASA NSTGRO 80NSSC24K1391.}
}

\markboth{IEEE Transactions on Robotics}%
{}

\maketitle
% \thispagestyle{empty}
% \pagestyle{empty}

%%%%%%%%%%%%%%%%%%%%%%%%%%%%%%%%%%%%%%%%%%%%%%%%%%%%%%%%%%%%%%%%%%%%%%%%%%%%%%%%
\begin{abstract}
Motion planning under dynamics constraints, \ie kinodynamic planning, enables safe robot operation by generating dynamically feasible trajectories that the robot can accurately track. For high-\dof robots such as manipulators, sampling-based motion planners are commonly used, especially for complex tasks in cluttered environments. However, enforcing constraints on robot dynamics in such planners requires solving either challenging two-point boundary value problems (BVPs) or propagating robot dynamics, both of which \REV{cause} computational~bottlenecks that drastically increase planning times. Meanwhile, recent efforts have shown that sampling-based motion planners can generate plans in microseconds using parallelization, but are limited to geometric paths. This paper develops \REVV{\normalsize{\FLASK}}, a fast parallelized sampling-based kinodynamic motion planning \REVV{framework} for a broad class of differentially flat robot systems, including manipulators, ground and aerial vehicles, and more. Differential flatness allows us to transform the motion planning problem from the original state space to a flat output space, where an analytical time-parameterized solution of \REV{the BVP problem} can be obtained. A trajectory in the flat output space is then converted back to a closed-form dynamically feasible trajectory in the original state space, enabling fast validation via ``single instruction, multiple data" parallelism. Our framework is fast, exact, and compatible with any sampling-based motion planner, \REV{while offering theoretical guarantees on probabilistic \REVVV{exhaustivity} and asymptotic optimality based on the closed-form BVP~solutions}. We extensively verify the effectiveness of our approach in both simulated benchmarks and real experiments with cluttered and dynamic environments, requiring mere microseconds to milliseconds of planning time.%
\end{abstract}
%%%%%%%%%%%%%%%%%%%%%%%%%%%%%%%%%%%%%%%%%%%%%%%%%%%%%%%%%%%%%%%%%%%%%%%%%%%%%%%%
\section{Introduction}

Motion planning is critical for safe and accurate robot operation in many applications, \emph{e.g.}, transportation~\cite{claussmann2020review}, environment monitoring~\cite{honig2018trajectory, zhou2022swarm}, warehouses~\cite{eppner2016lessons}, healthcare~\cite{riek2017healthcare}, and home assistance~\cite{jenamani2025feast}. In such applications, the robot has to react quickly to changes in its environment, promptly (re)plan a collision-free trajectory, and safely track the trajectory to reach a goal state, using a controller. This task requires fast motion planning, subject to both collision avoidance and robot dynamics constraints, to generate a dynamically feasible trajectory from a start to a goal that the robot is able to follow. While recent advances have significantly reduced motion planning times for geometric planning via parallelization techniques~\cite{thomason2024vamp, ramsey2024-capt}, sampling-based motion planning under dynamics constraints, \emph{i.e.}, kinodynamic planning, remains a challenge for real-time applications especially for high degree-of-freedom (\dof) robots such as manipulators. In this paper, we address this problem by leveraging the \emph{differential flatness} property of many common robot systems, such as ground and aerial vehicles, manipulators, and more, to enable ultrafast sampling-based kinodynamic motion planning via parallelization techniques.

\begin{figure}
\centering
\begin{subfigure}[t]{0.243\textwidth}
        \centering
\includegraphics[width=\textwidth]{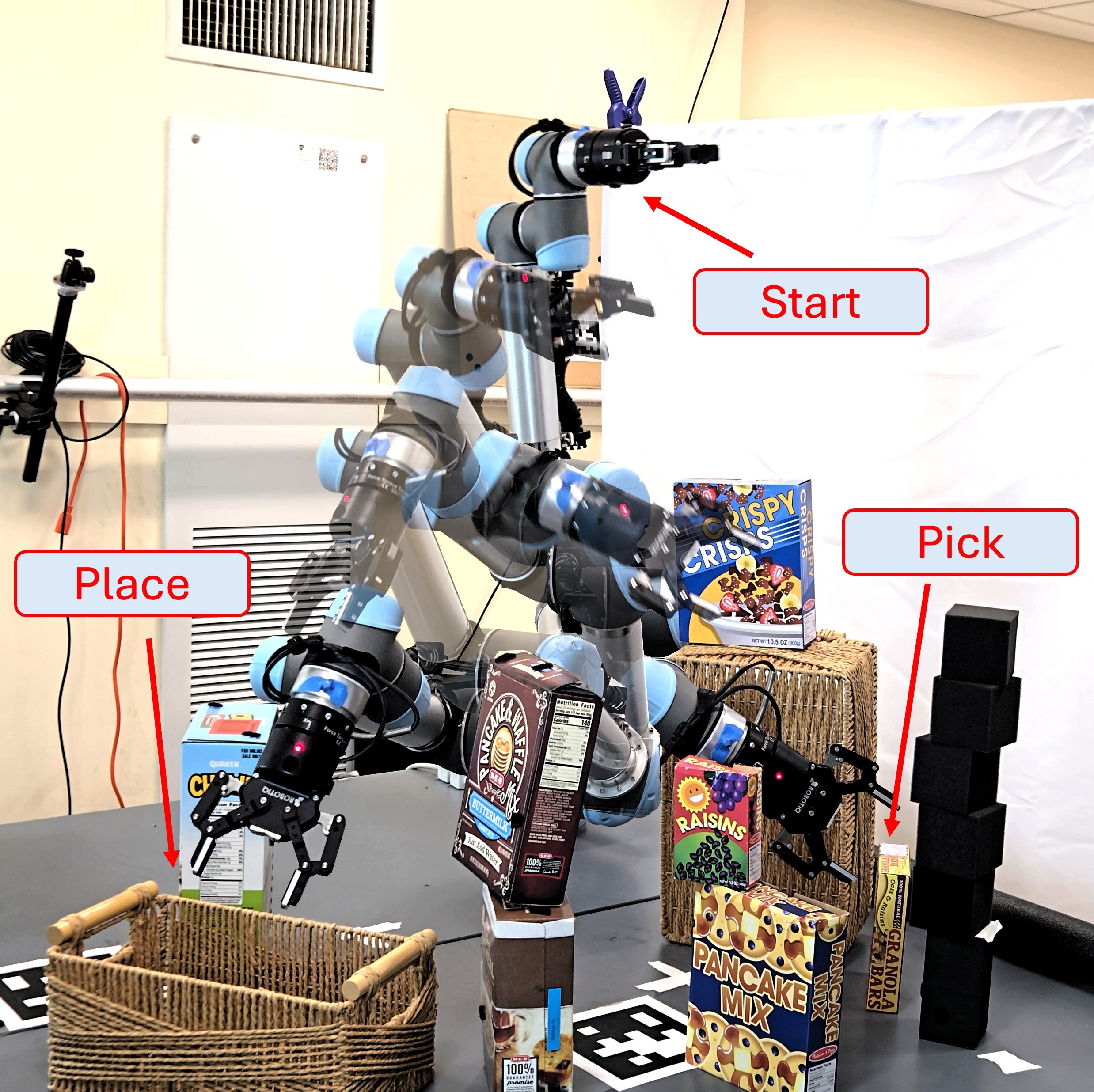}
\caption{Tracking our trajectory}
        \label{fig:demo_kino_plan}
\end{subfigure}%
\hfill
\begin{subfigure}[t]{0.243\textwidth}
        \centering
\includegraphics[width=\textwidth]{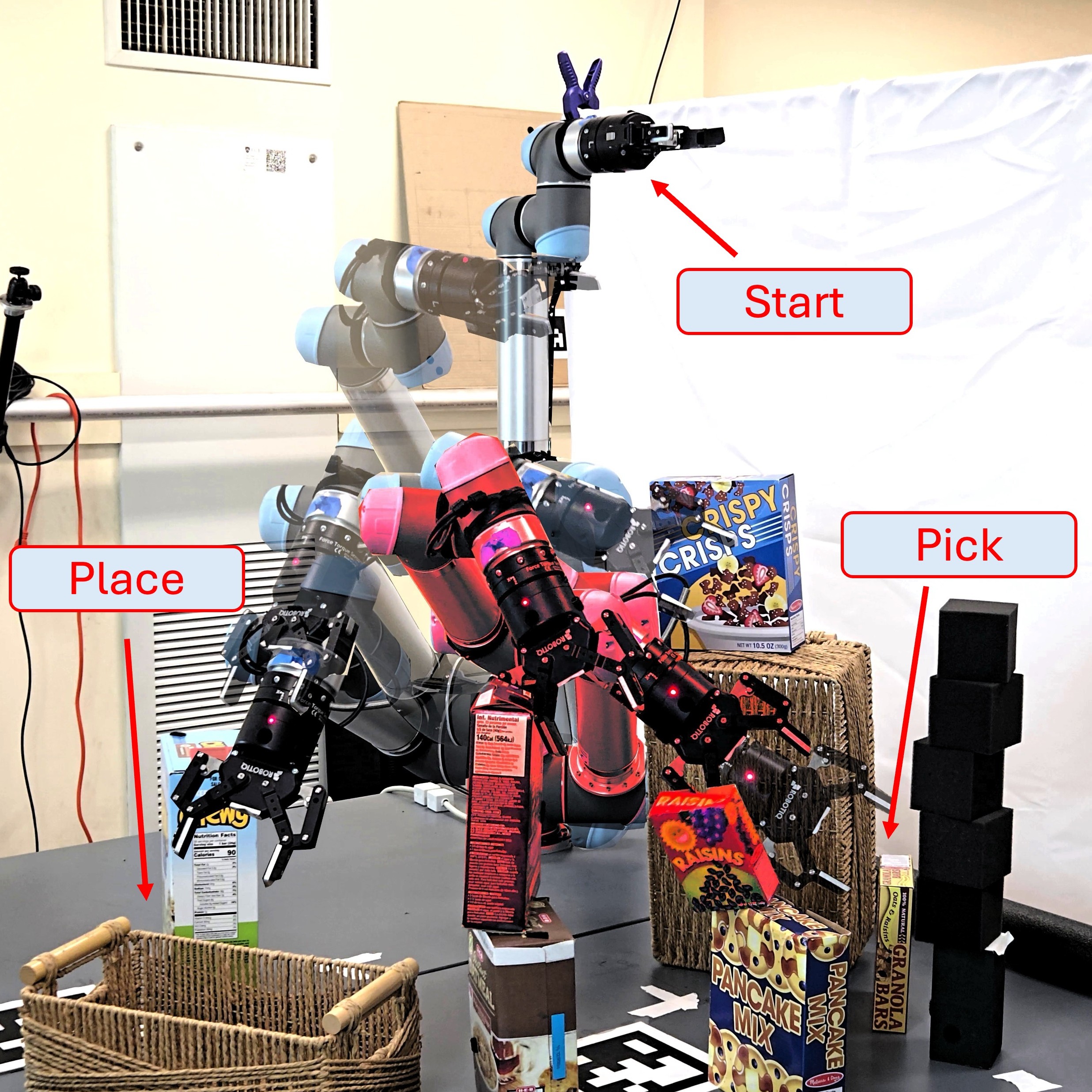}
\caption{Tracking a geometric path}
        \label{fig:demo_geo_plan}
\end{subfigure}%

\caption{
Motion planning for a ``\emph{pick and place}" task in a cluttered environment: a dynamically feasible trajectory (a) generated from our \FLASK framework can be accurately tracked by a UR5 robot. Meanwhile, tracking a geometric path (b) leads to collisions (shown in red) that topple the nearby boxes. Multiple intermediate states are overlaid to illustrate the robot's motion. Our planning framework is real-time and generates trajectories in $\sim 90\mu s$ by leveraging differential flatness and ``single instruction, multiple data" (\scSIMD) parallelism. 
}

\label{fig:ur5_pickandplace_annotated}
\end{figure}

Kinodynamic planning~\cite{lavalle2001randomized, hsu2022randomizekinodynamic, lavalle2006planning} generates dynamically feasible trajectories by directly enforcing dynamics constraints in the planner. Without such constraints, controllers often struggle to accurately track motion plans, especially with high-speed maneuvers, potentially leading to unsafe behavior. Optimization-based approaches, \emph{e.g.},~\cite{augugliaro2012generation, schulman2014trajopt, bonalli2019gusto, tassa2012synthesis, Mastalli2020Crocoddyl, ortizharo2025iDbAstar}, formulate kinodynamic planning as a nonlinear optimization problem and solve for the solution by minimizing a trajectory cost subject to constraints such as collision avoidance, robot dynamics, and joint angle and velocity limits. While these approaches can work well in high-dimensional state spaces, they are often susceptible to local minima and are sensitive to initial solutions. Meanwhile, search-based approaches, \emph{e.g.},~\cite{pivtoraiko2005efficient, pivtoraiko2011kinodynamic, cohen2010search, liu2017search, ajanovic2018searchbased}, build a graph in the state space on a grid or lattice and search for a sequence of motion primitives that connects the start and the goal. Search-based methods provide optimality guarantees but suffer from the curse of dimensionality, thus requiring complex heuristics or domain knowledge to guide the search. On the other hand, sampling-based approaches randomly sample to expand from the start to the goal, constructing a tree \cite{lavalle2001randomized, webb2013kinodynamicRRT, karaman2010kinorrt, hauser2016aorrt, li2016sst, verginis2023kdf} or sometimes a graph for simple systems~\cite{vandenberg2007kinoroadmap}.

A key challenge in enforcing dynamics constraints in sampling-based planning is to solve challenging two-point boundary value problems (BVPs) for dynamically feasible \emph{local paths} between robot states. 
Many planners~\cite{lavalle2001randomized, hsu2022randomizekinodynamic, li2016sst} avoid this by instead sampling an \REVV{often sub-optimal control} input and then propagating the dynamics forward using numerical integration~\cite{butcher2016numerical}. \REVV{Dynamic propagation with randomly sampled control often causes the planner to ``wander" to irrelevant parts of the state space and might lead to longer paths and longer planning times}. Other works solve the BVP problem for simple systems~\cite{karaman2010kinorrt} and linearized robot dynamics~\cite{webb2013kinodynamicRRT, Perez2012LQRRRT}, or approximate the solution using a neural network~\cite{wolfslag2018rrt, chiang2019rl, zheng2021sampling}. Instead, we obtain an \emph{exact analytical solution} of the BVP, thanks to the \emph{differential flatness} of many robot platforms such as mobile robots and manipulators~\cite{murray1995differential, mellinger2011minimumsnap}.%

Differential flatness~\cite{murray1995differential} is a powerful system property that, similar to feedback linearization~\cite{khalil2002nonlinear}, simplifies motion planning and control designs by converting the nonlinear robot dynamics to an equivalent linear system. For differentially flat systems, the robot states and control inputs can be described by a set of carefully chosen flat outputs and their derivatives (see \cref{def:diff_flat} for more details). A trajectory in the flat output space can be converted to a dynamically feasible trajectory in the robot's original state space, allowing us to simplify motion planning problems via a change of variables. This technique has been used mainly to plan trajectories for mobile robots, most notably quadrotors~\cite{mellinger2011minimumsnap, mellinger2012_trajgen, liu2017search}. However, the planning time is still hindered by computationally expensive subroutines such as forward kinematics and collision checking, which are particularly worse with high-\dof robots such as manipulators.

Recently, advances in parallelization~\cite{sundaralingam2023curobo, thomason2024vamp} have improved planning time to the range of microseconds and milliseconds. However, these works focus on either geometric~\cite{thomason2024vamp} or kinematic~\cite{sundaralingam2023curobo} planning problems. One family of these techniques uses ``fine-grained" parallelism via ``single instruction, multiple data" (\scSIMD) instructions,  available on consumer CPUs, to perform collision checking on multiple samples at the same time~\cite{thomason2024vamp, ramsey2024-capt}. However, directly enforcing dynamics constraints in parallel is non-trivial, as it requires sampling a trajectory at different time steps in advance while obeying robot dynamics. We address this problem by transforming the kinodynamic planning problem from the original state space to a flat output space, where closed-form \emph{local paths} are time-parameterized polynomials and hence, amenable to \scSIMD parallelization for fast collision checking. The resulting flat output trajectory is converted to a dynamically feasible collision-free trajectory in the original state space for the robot to execute. Our approach is fast, exact and general for most sampling-based planners.

In summary, we develop \FLASK\footnote{\REV{The code will be made publicly available at \url{https://kavrakilab.org/flask}.}}, a ``Differential \textbf{Fl}atness-based \textbf{A}ccelerated \textbf{S}ampling-based \textbf{K}inodynamic Planning" framework that:
\begin{itemize}
    \item generates dynamically feasible trajectories in the range of milliseconds for nonlinear differentially-flat robot systems, such as  manipulators, and mobile robots,
    \item constructs a planning tree or graph by solving the boundary value problem or dynamics propagation for continuous closed-form trajectories in the flat output space,
    \item performs fast collision checking of the closed-form trajectories using fine-grained \scSIMD parallelism.
\end{itemize}
\REVV{We provide a theoretical analysis on our approach's probabilistic \REVVV{exhaustivity}, a concept stronger than completeness introduced in \cite{schmerling2015optimal_driftless, kavraki1998analysis}, and optimality guarantees.}
We extensively verify our approach with challenging motion planning problems on low- and high-\dof, fully- and under-actuated robot systems in both simulated and real experiments. 
\REV{\FLASK is shown to achieve planning times of merely a few milliseconds, even on high-\dof robot platforms in cluttered environments, while respecting dynamics constraints by design.
Our framework is general and can be integrated in the context of many common sampling-based motion planners, such as RRT-Connect~\cite{kuffner2000rrtconnect}, SST*~\cite{li2016sst} and others, effective turning geometric planners into kinodynamic planners with theoretical guarantees.}%

\section{Related Work}
\subsection{Motion Planning with Dynamics Constraints}
\label{subsec:kino_mp}
Kinodynamic motion planning~\cite{schmerling2019kinodynamic}, \emph{i.e.}, motion planning under dynamics constraints, is a challenging task where the robot plans a dynamically feasible trajectory from a start, \emph{e.g.},~a robot state, containing its joint angles and derivatives, to a goal region while satisfying the robot dynamics and avoiding collision with obstacles in the environment. There are three main approaches: optimization-based, search-based, and sampling-based kinodynamic planning.

\emph{Optimization-based kinodynamic planning}, such as \mbox{TrajOpt~\cite{schulman2014trajopt}}, GuSTO~\cite{bonalli2019gusto}, and Crocoddyl~\cite{Mastalli2020Crocoddyl}, generates robot trajectories by formulating and solving an optimization problem, subject to dynamics and collision-avoidance constraints, with an objective function measuring the cost of the trajectory. This optimization problem is often nonlinear and can be solved via sequential convex programming~\cite{augugliaro2012generation, schulman2014trajopt, chen2015decoupled, bonalli2019gusto}, iterative linear quadratic regulators~\cite{tassa2012synthesis}, differential dynamic programming~\cite{Howell2019Altro, Mastalli2020Crocoddyl}, augmented Lagrangian methods~\cite{toussaint2017tutorial}, or general solvers~\cite{l2022whole,beck2025vitro}. In general, optimization-based trajectory planners provide smooth trajectories, but often get stuck in a local minimum and require good initialization. While trajectory optimization with a graph of convex sets~\cite{marcucci2023motion, vonwrangel2024gcs, graesdal2024_contactplanning} can avoid local minima, it requires expensive precomputation of the graph in the state space.

\emph{Search-based kinodynamic planning} ~\cite{pivtoraiko2005efficient, pivtoraiko2011kinodynamic, cohen2010search, liu2017search, ajanovic2018searchbased, mishani2025srmp} instead constructs a graph, often on a predefined grid or lattice, where each edge is chosen from a precomputed, discrete set of motion primitives, generated by propagating the robot dynamics for a short period of time under a set of control inputs. A motion primitive is valid if it does not collide with an obstacle. A search algorithm, such as $A^*$~\cite{hart1968formal}, can be used to find the shortest path on the graph, providing a trajectory as a sequence of motion primitives connecting the start with the goal. A major challenge of search-based approaches is the need to precompute dynamics propagation, where a numerical approximation with fine lattice resolution is required for high accuracy. They also suffer from the curse of dimensionality and require a good heuristic to guide the search.

\emph{Sampling-based kinodynamic planning}~\cite{orthey2024-review-sampling, webb2013kinodynamicRRT, karaman2010kinorrt, lavalle2001randomized, hsu2022randomizekinodynamic, hauser2016aorrt, li2016sst, verginis2023kdf} uses sampling to discretize the high-dimensional state space and build a tree (or, less often, a graph for simple systems such as car-like robots~\cite{vandenberg2007kinoroadmap}), growing from the start towards the goal region. To find a feasible trajectory connecting two samples, a difficult \emph{two-point boundary value problem} (BVP) has to be solved~\cite{lavalle2006planning}, posing a major challenge for sampling-based kinodynamic planning. A common approach to avoid solving a BVP problem for tree expansion is to sample the control input space, and propagate the robot dynamics, \emph{e.g.}, using a numerical integrator~\cite{lavalle2001randomized, hsu2022randomizekinodynamic, li2016sst} or physics-based models~\cite{gao2025parallel}, for a short period of time, with asymptotic optimality guarantees analyzed in~\cite{li2016sst}. 
Other approaches only solve the BVP problem for simple robot dynamics with low-dimensional state space~\cite{karaman2010kinorrt, webb2013kinodynamicRRT}, or linearized dynamics~\cite{webb2013kinodynamicRRT, Perez2012LQRRRT}. Meanwhile, learning-based kinodynamic motion planning uses neural networks to approximate the control input and the steering cost of expanding the tree towards a new node~\cite{wolfslag2018rrt, chiang2019rl, zheng2021sampling, ichter2019latentspacemp, li2021mpc-mpnet}.

BVPs and dynamics propagation \REV{cause} major computational bottlenecks for kinodynamic planning. While BVPs can be analytically solved for linear or simple systems~\cite{webb2013kinodynamicRRT, Perez2012LQRRRT}, this is not true in general for most nonlinear systems and it is computationally expensive to obtain an approximate solution.
Kinodynamic RRT$^*$~\cite{webb2013kinodynamicRRT, Perez2012LQRRRT} linearizes the dynamics around an operating point and demonstrates that BVPs can be solved in closed form for certain robots, \emph{e.g.}, quadrotors around a hovering position. However, the approximated solution only works well around the operating point, limiting aggressive maneuvers.  \REV{To avoid solving difficult BVPs, most existing kinodynamic planners~\cite{li2016sst} resort to dynamics propagation, where a constant value of the control is sampled instead. A numerical integrator such as Euler's or Runge-Kutta methods~\cite{butcher2016numerical} is then used to sequentially calculate the next robot state over multiple small time steps to maintain high accuracy. \REVV{Typically, the sampled control is suboptimal, and therefore, causes the planner to wander in the state space before reaching the goal}. Existing planners often mitigate this ``wandering" effect via best-state selection and tree pruning~\cite{li2016sst}, but still require long planning times to find a dynamically feasible trajectory.}

Optimization-based, search-based and sampling-based approaches can be combined to improve trajectory generation~\cite{ortizharo2024iDbRRT, ortizharo2025iDbAstar, natarajan2024pinsat, natarajan2023torque, natarajan2021interleaving, sakcak2019sampling, shome2021asymptotically, kamat2022bitkomo, choudhury2016regionally, alwala2021joint}. For example, a path or trajectory from sampling-based or search-based planners can be used as an initial solution for optimization-based ones~\cite{ortizharo2024iDbRRT, ortizharo2025iDbAstar}. INSAT planners~\cite{natarajan2024pinsat, natarajan2023torque, natarajan2021interleaving} interleave between search-based planning on a low-dimensional subspace and optimization-based planning on the full-dimensional space to improve planning times and success rates. Meanwhile, a library of precomputed motion primitives from search-based planning can be sampled to expand the planning tree in sampling-based motion planning~\cite{sakcak2019sampling, shome2021asymptotically}. Furthermore, optimized local paths can be used to generate collision-free edges and bias the sampling regions in a sampling-based planner~\cite{kamat2022bitkomo, choudhury2016regionally}. Another approach is to fit a geometric path with a time-parameterized trajectory using trajectory optimization such as TOPP-RA~\cite{pham2018toppra} or Ruckig~\cite{berscheid2021ruckig}, typically without considering collision avoidance constraints.

Our method performs fast sampling-based kinodynamic planning by leveraging \scSIMD parallelism (\cref{subsec:prelim_simd}) and the differential flatness of the dynamics of common robot platforms~\cite{allen2019real, liu2017search, bascetta2017flat, welde2021dynamically} to \REV{directly tackle the BVP problems}. A system is called \emph{differentially flat} if there exist variables, called flat outputs, whose values and derivatives dictate the robot state and control inputs (see \cref{def:diff_flat} for details). This approach has been applied to sampling-based~\cite{bascetta2017flat, ye2022efficient, wang2024differential, seemann2014exact}, search-based~\cite{liu2017search}, and optimization-based motion planning~\cite{mellinger2011minimumsnap, welde2021dynamically, han2023efficient, hao2005differential, beaver2024optimal}, but primarily for low-dimensional robot platforms \REV{(often with simple geometry shapes such as spheres or boxes for collision checking)}, \emph{e.g.}, quadrotors~\cite{mellinger2011minimumsnap}, \REV{unicycles and 2-link arms~\cite{beaver2024optimal}}, or for a specific system, \emph{e.g.}, gantry cranes~\cite{vu2022sampling}. Instead, we integrate differential flatness with sampling-based kinodynamic planning for generic high-\dof robots, where \REV{forward kinematics and collision checking are complex and time-consuming} besides enforcing dynamics constraints. While most existing methods approximate \REV{a solution to the BVP problems}, differential flatness allows us to obtain a closed-form fixed- or minimum-time polynomial BVP solution. Such~closed-form trajectories are amenable to parallelized collision checking using \emph{fine-grained} parallelization based on \scSIMD instructions (\cref{subsec:prelim_simd}), enabling trajectory generation in microseconds to milliseconds. \REVV{It is also important to note that existing flatness-based sampling-based planners either do not prove completeness and optimality~\cite{wang2024differential, seemann2014exact, ye2022efficient} or loosely mention the existing guarantees of geometric planners, \eg those of RRT*~\cite{bascetta2017flat, vu2022sampling}. Achieving such theoretical guarantees with closed-form BVP solutions is nontrivial as it requires careful consideration of nonlinear local paths rather than linear edges on the graph. We address this by offering a theoretical analysis of probabilistic exhaustivity, which is stronger than completeness, and asymptotic optimality of our approach in \cref{sec:theoretical_analysis}.}
\subsection{Hardware-accelerated Motion Planning}
\label{subsec:prelim_simd}
Motion planning can be time-consuming as it relies on multiple computationally expensive subroutines, such as forward kinematics (\scFK), collision checking (\scCC) and nearest neighbor (\scNN) search. With recent advances in parallel computing, much progress has been made to improve these subroutines and enhance planning performances via both \emph{coarse-grained} and \emph{fine-grained} parallelization.

\emph{Coarse-grained} parallelization techniques typically run multiple subroutines or even multiple instances of the planners at the thread or process levels. Early work focuses on improving motion plans by merging and averaging out the paths from different instances of the planners~\cite{raveh2011little}, or by adapting existing planners to run their subroutines in parallel~\cite{amato1999probabilistic, ichnowski2012parallel}. Parallelized motion planning can also be achieved by partitioning the configuration space~\cite{jacobs2012scalable,werner2025gcs} or the planning tree construction~\cite{plaku2005sampling, vu2022sampling, perrault2025kino}. Closely related to our work,~\cite{vu2022sampling} also employs differential flatness to generate dynamically feasible trajectories, however, by coarsely growing multiple planning subtrees in parallel for a specific gantry crane system. Recently, the prevalence of GPUs enables impressive improvements in \REV{motion planning~\cite{bhardwaj2022storm, sundaralingam2023curobo, fishman2023motion, le2025global, le2025model}} but suffers from costly GPU resources and communication overhead between CPUs and GPUs. GPU-based planning methods often grow a planning tree in parallel by propagating the robot dynamics, \mbox{\emph{e.g.}, Kino-PAX~\cite{perrault2025kino}}, or via approximate dynamic
programming recursion, \emph{e.g.}, GMT*~\cite{ichter2017gmt}, and are shown to generate a robot trajectory in milliseconds for low-\dof systems with simple forward kinematics. For high-\dof robots, cuRobo~\cite{sundaralingam2023curobo} generates geometric paths as seeds for a parallelized trajectory optimization solver under kinematics constraints such as velocity, acceleration and~jerk limits.

\emph{Fine-grained} parallelization techniques focus on parallelizing primitive operations in a motion planning algorithm, \emph{e.g.},  via ``single instruction/multiple data" (\scSIMD) instructions on consumer-grade CPUs. This has been shown to provide extremely fast geometric motion planning subject to collision avoidance constraints, with planning times ranging from microseconds to milliseconds~\cite{thomason2024vamp, ramsey2024-capt, wilson2024nearest}. To check an edge (a line segment) for collision, VAMP~\cite{thomason2024vamp} uses \scSIMD instructions to efficiently perform parallelized forward kinematics and collision checking on multiple configurations, generated via linear interpolation. While this approach is promising, it is challenging to enforce additional requirements via parallelized primitive operations such as dynamics and non-holonomic constraints. \REV{Particularly, enforcing nonlinear dynamics constraints with fine-grained parallelism is non-trivial due to the intractability of BVP problems, as mentioned in \cref{subsec:kino_mp}}. We instead leverage the differential flatness property to obtain time-parameterized solutions that can be discretized at arbitrary times and hence, amenable to SIMD-based primitive operations. \REV{Our deliberate fusion of differential flatness and SIMD parallelism effectively brings the benefits of ``fine-grained" parallelized forward kinematics and collision checking to kinodynamic planning. \REVV{Thanks to the closed-form BVP solution in the flat output space, we will later show that our approach achieves planning times of merely a few milliseconds.}}%,
\section{Problem Formulation}
Consider a robot with state $\bfx \in \calX$ and control $\bfu \in \calU$. For example, the state of a manipulator can include the joint configuration $\bfq$ and possibly its derivatives, while the control input can be the torques being applied on the robot joints. Let $\calX_{free}$ and $\calX_{obs} = \calX \setminus \calX_{free}$ be the free and occupied spaces, respectively, which can be generated from robot constraints such as collision avoidance or joints limits. The robot motion is governed by a  nonlinear dynamics function $\bff$ of the state $\bfx$ and the control $\bfu$ as follows:
\begin{equation} \label{eq:dynamics}
    \dot{\bfx} = \bff (\bfx, \bfu).
\end{equation}
A time-parameterized trajectory $\bfsigma: [0,1] \rightarrow \calX$ for time $t \in [0,1]$ is called dynamically feasible if there exists a \mbox{time-parameterized} control input $\bfu: [0,1] \rightarrow \calU$ such that the robot dynamics is satisfied by the trajectory: 
\begin{equation}
    \dot{\bfsigma}\REV{(t)} = \bff(\bfsigma\REV{(t)}, \bfu\REV{(t)}), \quad \REV{\forall t \in [0,1]}.
\end{equation}
Given an initial robot state $\bfx_s$ and a goal region $\calG\REV{\subset \calX}$, \REV{as a subset of $\calX$}, the kinodynamic motion planning problem aims to find a dynamically feasible trajectory $\bfsigma(t)$ with control input $\bfu(t)$, connecting the initial state $\bfx_s$ to the goal region $\calG$ in the free space $\calX_{free}$  as described in Problem \ref{problem:kinodyn_planning}. \REVV{The goal region $\calG$ commonly represents a desired state or a region that a goal state can be sampled from.}

\begin{problem} \label{problem:kinodyn_planning}
    Given an initial robot state $\bfx_s$ and a goal region $\calG \REV{\subset \calX}$, find a control input $\bfu(t)$ that generates a dynamically feasible trajectory $\bfsigma(t)$ such that:
    \begin{equation} \label{eq:problem_formulation}
        \begin{aligned}
            &\dot{\bfsigma}\REV{(t)} = \bff(\bfsigma\REV{(t)}, \bfu\REV{(t)}), \\
            &\bfsigma(0) = \bfx_s, \bfsigma(1) \in \calG, \\
            &\bfsigma(t) \in \calX_{free}, \bfu(t) \in \calU \quad \forall t \in [0,1].
        \end{aligned}
    \end{equation}
\end{problem}

In the remainder of the paper, we solve \cref{problem:kinodyn_planning} by developing \textbf{\FLASK}, a parallelized kinodynamic motion planning \REVV{framework} for differential flat robot systems, including common platforms such as ground and aerial vehicles, and manipulators. Our approach offers ultrafast planning time and exact dynamically feasible trajectory solution without the need to approximate the robot dynamics. Occasionally, we will drop the notation of time dependence for readability.

\section{Preliminaries} \label{sec:prelim}
In this section, we provide a brief review and necessary background on sampling-based kinodynamic planning, differential flatness and fine-grained parallelization that will be useful for the derivation of our approach in \cref{sec:technical_approach}.

\subsection{Sampling-based Kinodynamic Motion Planning} \label{subsec:prelim_sbkmp}
Most sampling-based geometric planners, such as RRT-connect~\cite{kuffner2000rrtconnect} and PRM~\cite{kavraki2002probabilistic} consider the robot configuration $\bfq$ as the state $\bfx$ and approximate the robot's configuration space with a tree or graph $\bbG$ with a set of nodes $\bbV$ and a set of edges $\bbE$.
Though individual sampling-based planners differ wildly in their exact approach, they all have roughly the same structure for their main search loop.
At each iteration, such planners attempt to add a new sample $\bfx_f$ to $\bbV$, then connect $\bfx_f$ to some existing $\bfx_0 \in \bbV$, and if successful, add the resulting edge to~$\bbE$. This is called a \textbf{\textsc{Connect}} or \textbf{\textsc{Extend} subroutine}.
To construct edges, all geometric sampling-based planners require a \emph{local planner} to produce a local path between configurations.

However, in kinodynamic motion planning, where robots are subjected to dynamics constraints, finding \textbf{a local path} $\bfx_{loc}(t)$ to reach $\bfx_f$ with control input $\bfu_{loc}(t)$ in duration $T$, requires solving a \emph{boundary value problem} (BVP), subject to the robot dynamics with initial state $\bfx_0$ and terminal state $\bfx_f$ as follows:
\begin{equation}
\begin{aligned}
   &\dot{\bfx}_{loc} = \bff(\bfx_{loc}, \bfu_{loc}),\quad 
    \bfx_{loc}(0) = \bfx_0, \bfx_{loc}(T) = \bfx_f.
\end{aligned}
\end{equation}

Solving the BVP in practice is often computationally intractable, so most kinodynamic planners (\emph{e.g.},~\cite{lavalle2001randomized, hsu2022randomizekinodynamic, li2016sst}) instead take a propagation-based approach: they integrate a sampled control input $\bfu_0$ from a reachable state for a short time $T$ to generate a new state $\bfx_f$ rather than connecting sampled states. In this case, the \textbf{local path} $\bfx_{loc}(t)$ and the new state $\bfx_f$ are defined as:
\begin{equation}
\begin{aligned}
   &\bfx_{loc}(t) = \bfx_0 + \int_{0}^t \bff(\bfx(\tau), \bfu_0)d\tau, \quad \bfx_f = \bfx_{loc}(T).
\end{aligned}
\end{equation}

\subsection{Differential Flatness} \label{subsec:diff_flat}
\begin{definition} [Differential Flatness] \label{def:diff_flat}
    A dynamical system \eqref{eq:dynamics} is called ``\emph{differentially flat}"~\cite{murray1995differential} if there exists an $n$-dimensional output:
\begin{equation} \label{eq:flat_output_def}
    \bfy = \bfh (\bfx, \bfu, \dot{\bfu}, \ldots, \bfu^{(k)}) \in \bbR^{n},
\end{equation}
% $$\bfx_T = \int_{0}^T \bff(\bfx, \bfu)dt$$
such that the robot state $\bfx$ and control input $\bfu$ can be described in terms of the output $\bfy$ and its derivatives $\bfy^{(\cdot)}$:
\begin{equation} \label{eq:state_control_from_flat_output}
    \begin{aligned}
        \bfx &= \bfalpha(\bfy, \dot{\bfy}, \ldots, \bfy^{(l)}), \\
        \bfu &= \bfbeta(\bfy, \dot{\bfy}, \ldots, \bfy^{(m)}),
    \end{aligned}
\end{equation}
for non-negative derivative orders $k,l$ and $m$. The output $\bfy$ must be differentially independent, \emph{i.e.}, there does not exist any differential relationship among the components of $\bfy$.
The exact form of the functions $\bfalpha$ and $\bfbeta$ depends on the robot system, several of which can be found in~\cite{murray1995differential, mellinger2011minimumsnap}. 
\end{definition}

We provide three common examples of differentially flat fully-actuated and under-actuated robot platforms, as follows.

\begin{example} \label{example:manipulator}
    Consider a \emph{fully-actuated manipulator} with joint angles $\bfq$, and control input $\bfu$, \emph{e.g.}, the joint torques. This is typically the case for common manipulators such as Franka or KUKA platforms. The robot dynamics is described by the Euler-Lagrange equation of motions:
    \begin{equation}
        \bfM(\bfq)\ddot{\bfq} + \bfC(\bfq, \dot{\bfq}) = \bfB(\bfq)\bfu,
    \end{equation}
    where the control gain matrix $\bfB(\bfq)$ is typically invertible, \emph{i.e.}, the system is fully actuated. The robot dynamics can be expressed in the form of Eq. \eqref{eq:dynamics} with the robot state $\bfx = (\bfq, \dot{\bfq})$.

    For this system, the flat output is the same as the configuration: $\bfy = \bfq$. The state $\bfx$ and the control input $\bfu$ can be derived from the flat output $\bfy$ as follows,
    \begin{equation} \label{eq:flat_manipulator}
        \begin{aligned}
            \bfx &= \bfalpha(\bfy, \dot{\bfy}) =  (\bfq, \dot{\bfq}) \\ 
            \bfu &= \bfbeta(\bfy, \dot{\bfy}, \ddot{\bfy}) = \bfB^{-1}(\bfq)\left(\bfM(\bfq)\ddot{\bfq} + \bfC(\bfq, \dot{\bfq})\right).
        \end{aligned}
    \end{equation}
As the control gain $\bfB(\bfq)$ is invertible, the control input $\bfu$ in \eqref{eq:flat_manipulator} is guaranteed to exist.
\end{example}
\begin{example} \label{example:unicycle}
    Consider a \emph{unicyle} robot whose state $\bfx$ is defined as $\bfx = (x, y, \theta)$ where $(x,y)$ is the position and $\theta$ is the heading angle of the vehicle. The control input $\bfu = (v, \omega)$ includes the speed $v$ and angular velocity $\omega$. The flat output $\bfy$ is defined as the position: $\bfy = (x, y)$. The yaw angle $\theta$ and the control input $\bfu$ can be determined from $\bfy$ as follows:
    \begin{equation}
    \begin{aligned}
                \bfx &= \bfalpha(\bfy, \dot{\bfy}) = (x, y, \text{arctan2}(\dot{y}, \dot{x}) + \kappa\pi), \\
                \bfu &= \bfbeta(\bfy, \dot{\bfy}, \ddot{\bfy}) = \left((-1)^\kappa \sqrt{\dot{x}^2 + \dot{y}^2 }, \frac{\dot{x}\ddot{y} - \ddot{x}\dot{y}}{\dot{x}^2 + \dot{y}^2}\right),
    \end{aligned}
    \end{equation}
where $\kappa \in \{0,1\}$ depends on whether the vehicle is moving forward or backward, respectively.
\end{example}

\begin{example} \label{example:quadrotor}
    Consider an \emph{under-actuated quadrotor} whose state $\bfx$ is defined as $\bfx = (\bfp, \bfR, \bfv, \bfomega)$, where $\bfp = (x, y, z) \in \bbR^3$ is the position of the center of mass, $\bfR \in SO(3)$ is the rotation matrix, $\bfv$ is the linear velocity, and $\bfomega$ is the angular velocity. The control input $\bfu = (f, \bftau)$ consists of a thrust $f\in \bbR_{\geq 0}$ and a torque $\bftau \in \bbR^3$, generated from the motors. The flat output for quadrotor systems is $\bfy = (\bfp, \psi) \in \bbR^4$, where $\psi$ is the yaw angle of the robot~\cite{mellinger2011minimumsnap}. 
    
    The state $\bfx = \bfalpha(\bfp, \dot{\bfp}, \ddot{\bfp}, \bfp^{(3)}, \psi, \dot{\psi})$ can be expressed in terms of the flat outputs and their derivatives as follows. 
    Clearly, the position $\bfp$ is already part of the flat output $\bfy$, and hence the linear velocity is calculated as $\bfv = \dot{\bfp}$.
    Let $m$ and $\bfJ$ be the mass and inertia matrix of the quadrotor, respectively. Let $\bft = m(\ddot{\bfp} + g\bfe_z)$ be the thrust vector applied on the quadrotor's center of mass, which coincides with the $z-$axis of the body frame where $g$ is the gravitational acceleration, and $\bfe_z = \begin{bmatrix}
        0 & 0 & 1
    \end{bmatrix}^\top$ is the $z-$axis unit vector in the world frame. The rotation matrix $\bfR~=~\begin{bmatrix}
        \bfr_x & \bfr_y & \bfr_z
    \end{bmatrix}$ can be calculated~as:
    \begin{equation}
            \bfr_z = \frac{\bft}{\Vert \bft\Vert}, \quad 
            \bfr_x = \frac{\bfr_{\psi} \times \bfr_z}{\Vert \bfr_{\psi} \times \bfr_z \Vert}, \quad            \bfr_y = \bfr_z \times \bfr_x,
    \end{equation}
    where $\bfr_\psi = [-\sin{\psi}, \cos{\psi}, 0]$. The derivative of the rotation matrix is $\dot{\bfR}~=~\begin{bmatrix}
        \dot{\bfr}_x & \dot{\bfr}_y & \dot{\bfr}_z
    \end{bmatrix}$ with:
    \begin{equation}
    \begin{aligned}
            \dot{\bfr}_x &= \bfr_x \times \frac{\dot{\bfr}_{\psi} \times \bfr_z + \bfr_{\psi} \times \dot{\bfr}_z}{\Vert \bfr_{\psi} \times \bfr_z \Vert} \times \bfr_x, \\
            \dot{\bfr}_y &= \dot{\bfr}_z \times \bfr_x + \bfr_z \times \dot{\bfr}_x,             \quad \dot{\bfr}_z = \bfr_z \times \frac{\dot{\bft}}{\Vert \bft\Vert} \times \bfr_z.
    \end{aligned}
    \end{equation}
    The angular velocity $\bfomega$ is calculated as: $\bfomega = (\bfR^\top \dot{\bfR})^\vee$,
    where the $(\cdot)^\vee$ operator maps a skew-symmetric vector \mbox{$\hat{\bfomega} \in \frak{so}(3)$} to a vector $\bfomega\in \bbR^3$. 
    Meanwhile, the control input $\bfu~=~\bfbeta(\bfp, \dot{\bfp}, \ddot{\bfp}, \bfp^{(3)}, \bfp^{(4)}, \psi, \dot{\psi}, \ddot{\psi})$ is described as:
    \begin{equation}
        \bfu = (f, \bftau)  = \left(\Vert \bft \Vert, \bfJ (\dot{\bfR}^\top \dot{\bfR} + \bfR^\top \ddot{\bfR})^\vee + \bfR^\top \dot{\bfR} \bfJ \bfomega\right).
    \end{equation}
    We refer the readers to~\cite{mellinger2011minimumsnap, zhou2014vectorfield, liu2018search} for the detailed derivation. 
\end{example}

\begin{figure*}[t]
\centering

\begin{subfigure}[t]{0.32\textwidth}
        \centering
\includegraphics[width=0.75\textwidth]{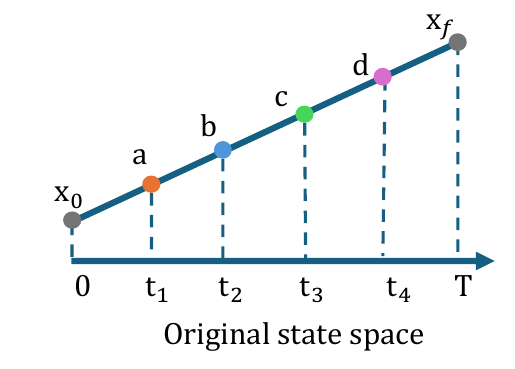}%
 \caption{Samples from a linear path~\cite{thomason2024vamp}.}
        \label{fig:cc_vamp}
\end{subfigure}%
\begin{subfigure}[t]{0.55\textwidth}
        \centering
\includegraphics[width=\textwidth]{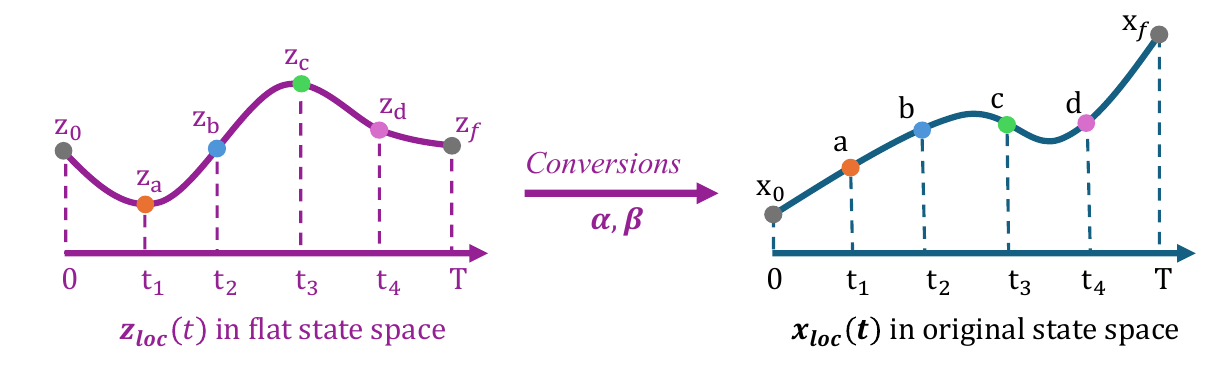}%
 \caption{Samples from a nonlinear local path $\bfx_{loc}(t)$.}
        \label{fig:cc_flat}
\end{subfigure}%
\caption{Configuration samples $\bfa, \bfb, \bfc$ and $\bfd$, discretized from a linear path (a), as in VAMP~\cite{thomason2024vamp}, and from our closed-form time-parameterized motions (b), can be efficiently checked for collision using \scSIMD parallelism.}
\label{fig:simd_cc}
\end{figure*}

\subsection{CPU fine-grained parallelism} \label{subsec:fine_grained_parallelism}

Fine-grained parallelization techniques such as ``single instruction, multiple data" (\scSIMD) are ubiquitous on consumer CPUs and have recently shown significant improvement in sampling-based geometric motion planning by performing forward kinematics and collision checking on multiple configurations in parallel. Vectorization via \scSIMD allows simultaneously applying the same primitive operations on multiple variables. VAMP~\cite{thomason2024vamp}, a \scSIMD-accelerated planning method, uses this technique to perform parallelized collision checking of multiple robot configurations, evenly sampled along a line segment~(\cref{fig:cc_vamp}) connecting two nodes on a planning tree or graph. Given the set of configurations, the robot's geometric shape, described by a set of \REV{SIMD-compatible shapes such as capsules or spheres} in the workspace, is calculated via branchless forward kinematics and checked for collision against the obstacle geometry.

\scSIMD-based collision checking requires access to all the states along a local path, typically through an analytical form of the motion. However, under dynamics constraints, the nonlinear local path $\bfx_{loc}(t)$ that \REV{solves the BVP problem (\cref{subsec:prelim_sbkmp}) is computationally intractable, prohibiting configuration sampling at multiple arbitrary times in advance}. In the next section, we address this problem by deriving a closed-form time-parameterized local path $\bfx_{loc}(t)$ based on the differential flatness property and sampling multiple states at different times, as illustrated in \cref{fig:cc_flat}.

\section{Technical Approach}
\label{sec:technical_approach} 
In this section, we present our \REVV{kinodynamic planning framework, \textbf{\FLASK},} by showing that the flat output evolves as a linear system (\cref{subsec:flat_linear}), and hence, allows us to convert \cref{problem:kinodyn_planning} from the original state space $\calX$ to a flat state space (\cref{problem:flatspace_planning} in \cref{subsec:flat_planning_problem}).
\REVV{In the flat state space, we develop primitive subroutines, \textsc{FlaskExtend} for planning graph construction in \cref{subsec:flat_sbmp} and \textsc{FlaskCC} for parallelized forward kinematics and collision checking in \cref{subsec:vectorized_cc}, that work with any sampling-based planning method. These subroutines are the core of our framework in \cref{alg:main_proposed_alg}, which effectively turns any geometric motion planner into a kinodynamic version. Instead of being tied to a specific planner, our general framework gives rise to a new class of ultrafast kinodynamic planners for a broad class of differentially flat robot systems.} 
Finally, we discuss trajectory postprocessing in our approach in \cref{subsec:traj_simplification}.

%%%%%%%%%%%%%%%%%%%%%%%%%%%%%%%%%%%%%%%%%%%%%%%%%%%%%%%%%%%%%%%%%%%%%%%%%%%%%%%%%%
\subsection{Flat Output Dynamics as a Linear System} \label{subsec:flat_linear}
Consider the $n$-dimensional flat output $\bfy$ in \cref{def:diff_flat}, that can be used to recover the state $\bfx$ and control input $\bfu$ via Eq.~\eqref{eq:state_control_from_flat_output}. Let $\bfw \in \REV{\calW \subset \bbR^n}$ be a pseudo-control input, defined as the $r$th derivative of the output $\bfy$:
\begin{equation} \label{eq:pseudo_control}
    \bfw = \bfy^{(r)},\quad r\MATHREV{\;\geq\;} \max(l,m),
\end{equation}
where the derivative orders $l,m$ are defined in~\cref{eq:state_control_from_flat_output}. Let~$\bfz~\in~\calZ \subset \bbR^{rn}$ be the flat state, defined as:
\begin{equation} \label{eq:z_def}
    \bfz = (\bfy, \dot{\bfy}, \ldots, \bfy^{(r-1)}).
\end{equation}
The state $\bfz$ satisfies linear dynamics with the pseudo-control input $\bfw$ as follows:
\begin{equation}\label{eq:linear_dynamics}
    \dot{\bfz} = \bfA \bfz + \bfB \bfw \in \bbR^{rn},
\end{equation}
\begin{equation*}
\bfA =\begin{bmatrix}
    \bf0 & \bfI_n & \bf0 & \cdots & \bf0 \\
    \bf0 & \bf0 & \bfI_n & \cdots & \bf0 \\
    \vdots & \vdots & \vdots & \ddots & \vdots \\
    \bf0 & \bf0 & \bf0 & \cdots & \bfI_n \\
    \bf0 & \bf0 & \bf0 & \cdots & \bf0
\end{bmatrix},
\qquad
\bfB=\begin{bmatrix}
    \bf0 \\ \bf0 \\ \bf0 \\ \vdots\\ \bfI_n
\end{bmatrix},
\end{equation*}
with $\bfA \in \bbR ^{rn\times rn}$, $\bfB \in \bbR ^{rn\times n}$, and the identity matrix $\bfI_n~\in~\bbR^{n\times n}$. Instead of enforcing the nonlinear dynamics constraint~\eqref{eq:dynamics} in a motion planning problem, we can implicitly enforce a much simpler linear dynamics \eqref{eq:linear_dynamics} in the flat state space via the conversion \eqref{eq:state_control_from_flat_output}. Note that the matrix $\bfA$ is nilpotent with index $r$, \emph{i.e.},~$\bfA^r = 0$.

%%%%%%%%%%%%%%%%%%%%%%%%%%%%%%%%%%%%%%%%%%%%%%%%%%%%%%%%%%%%%%%%%%%%%%%%%%%%%%%%%%
\subsection{Sampling-based Kinodynamic Planning in Flat State Space} \label{subsec:flat_planning_problem}

Due to the non-linearity of the robot dynamics \eqref{eq:dynamics}, it is challenging to plan robot motions in the original state space $\calX$, as most sampling-based motion planning algorithms require a challenging \textsc{Extend} subroutine that either solves a boundary value problem to connect two states $\bfx_0$ and $\bfx_f$ or propagates the dynamics to predict the next state $\bfx_f$ given a constant control input $\bfu_0$. As shown in \cref{subsec:flat_linear}, the flat state $\bfz$ satisfies a linear dynamics in \eqref{eq:linear_dynamics} with a nilpotent matrix $\bfA$, potentially leading to much simpler BVP problem and dynamics propagation. This motivates us to perform motion planning in the flat state space $\calZ$ instead of the original state space $\calX$, thanks to the conversions \eqref{eq:flat_output_def} and \eqref{eq:state_control_from_flat_output}. Our approach is illustrated in \cref{fig:problem_conversion}.

\begin{figure}[t]
\centering
\includegraphics[width=0.9\linewidth]{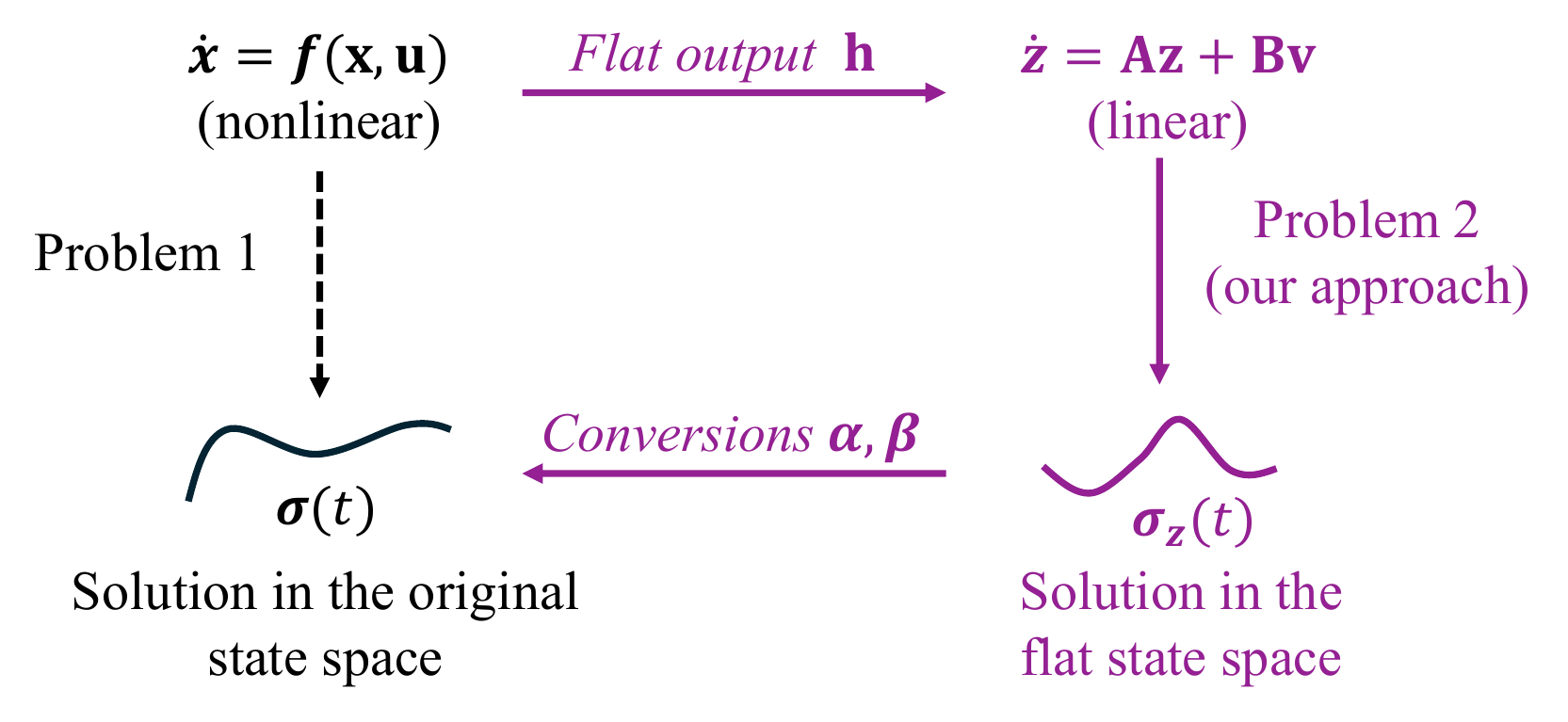}
\caption{Formulating the kinodynamic planning problem in the flat state space with linear dynamics.}
\label{fig:problem_conversion}
\end{figure}

Given an initial state $\bfx_s$ with an initial control input $\bfu_s$, and a goal region $\calG$ with a desired control $\bfu_\calG$, the initial flat output and goal region are calculated as $\bfy_s = \bfh (\bfx_s, \bfu_s, \dot{\bfu}_s, \ldots, \bfu^{(k)}_s)$ and $\calG_\bfy = \{\bfh(\bfx, \bfu_\calG, \dot{\bfu}_\calG, \ldots, \bfu^{(k)}_\calG) \;|\; \bfx \in \calG\}$, respectively. Note that in many common robot systems such as \cref{example:manipulator,example:unicycle,example:quadrotor}, the flat output $\bfy$ does not depend on $\bfu$ and therefore, the values of $\bfu_s$ and $\bfu_\calG$ need not be specified.
The original kinodynamic motion planning (\cref{problem:kinodyn_planning}) becomes finding a dynamically feasible trajectory $\bfsigma_\bfz(t)$ in the flat state space $\calZ$ that connects the start $\bfz_s$ to the goal region $\calG_\bfz$, and satisfies the linear dynamics \eqref{eq:linear_dynamics}, as summarized in \cref{problem:flatspace_planning}.
\begin{problem} \label{problem:flatspace_planning}
    Given an initial flat state $\bfz_s$ and a goal region $\calG_\bfz$, calculated from $\bfy_s$ and $\calG_\bfy$ via \eqref{eq:z_def}, the original kinodynamic motion planning (\cref{problem:kinodyn_planning}) is equivalent to finding a control input function $\bfw(t)$ that generates a dynamically feasible trajectory $\bfsigma_\bfz(t)$:
    \begin{equation} \label{eq:problem_formulation2}
        \begin{aligned}
            %&\min_{\bfw} c(\bfalpha(\bfsigma_\bfz, \bfw), \bfbeta(\bfsigma_\bfz, \bfw)) \\
           & \dot{\bfsigma}_\bfz = \bfA \bfsigma_\bfz + \bfB \bfw, \\
           &\bfsigma_\bfz(0) = \bfz_s, \bfsigma_\bfz(1) \in \calG_\bfz, \\
            & \bfu(t) = \bfbeta \left(\bfsigma_\bfz(t), \bfw(t)\right) \in \calU, \\
            &\bfx(t) = \bfalpha(\bfsigma_\bfz(t), \bfw(t)) \in \calX_{free}, \forall t \in [0,1],
        \end{aligned}
    \end{equation}
    where the matrices $\bfA$ and $\bfB$ are defined in \eqref{eq:linear_dynamics}.
\end{problem}

\begin{algorithm}[t]
    \caption{\textsc{\textbf{\REVV{\FLASK}}: \normalsize Sampling-based Kinodynamic Motion Planning via Differential Flatness}}
    \label{alg:main_proposed_alg}
     \DontPrintSemicolon
    \KwIn{Initial state $\bfx_s$, initial control $\bfu_s$, goal region $\calG$ with control input $\bfu_\calG$, maximum iterations \REV{$I$}} % 
    \KwOut{A collision-free dynamically feasible trajectory $\bfsigma(t)$ with control input $\bfu(t)$.}

   \Comment*[l]{\small \BLUE{Convert to the flat state space}}
   $\bfy_0 \gets \bfh (\bfx_0, \bfu_0, \dot{\bfu}_0, \ldots, \bfu^{(k)}_0)$\\
   $\calG_\bfy \gets \{\bfh(\bfx, \bfu_\calG, \dot{\bfu}_\calG, \ldots, \bfu^{(k)}_\calG): \bfx \in \calG\}$\\
    $\bfz_0 \gets (\bfy_0, \dot{\bfy}_0, \ldots, \bfy^{(r-1)}_0)$\\
    $\calG_z = \{(\bfy, \dot{\bfy}, \ldots, \bfy^{(r-1)}): \bfy \in \calG_\bfy\}$\\

    Create a planning graph or tree $\bbG_\bfz = (\bbV_\bfz, \bbE_\bfz)$ with set of vertices $\bbV_\bfz$ and set of edges $\bbE_\bfz$.

    $\bbV_\bfz \gets \{ \bfz_0 \}$
    
    \While{Iteration $i < I$}
    {   \Comment*[l]{\small \BLUE{Grow $\bbG_\bfz$ on the flat state space}}
        $\bbG_\bfz \gets$ \textsc{FlaskExtend}($\bbG_\bfz$) 
        
        \Comment*[l]{\small \BLUE{Return $\bfsigma(t)$ if reaching \REV{$\calG_\bfz$}}}
        \If{Goal region $\calG_\bfz$ is reached}
        {
            Find the trajectory $\bfsigma_\bfz(t)$ with pseudo-input $\bfw(t)$ from $\bbG_\bfz$.

            \Comment*[l]{\small \BLUE{Convert the trajectory $\bfsigma_z(t)$ to the original state space}}
            $\bfsigma(t) \gets \bfalpha(\bfsigma_\bfz(t), \bfw(t))$\\
            $\bfu(t) \gets \bfbeta(\bfsigma_\bfz(t), \bfw(t))$\\
            
            \Return Trajectory $\bfsigma(t)$ with control $\bfu(t)$.
        }
    }
    \Return Trajectory not found.
\end{algorithm}

To solve \cref{problem:flatspace_planning}, \REVV{we develop primitive subroutines in the flat state space $\calZ$ for a sampling-based kinodynamic planning \REVV{framework proposed} in \cref{alg:main_proposed_alg}}, which is consistent with most existing sampling-based motion planners~\cite{karaman2010kinorrt, webb2013kinodynamicRRT, li2016sst}. Leveraging the linear dynamics~\eqref{eq:linear_dynamics}, we develop an efficient \textsc{FlaskExtend} subroutine (see \cref{subsec:flat_sbmp}) to construct a graph or tree $\bbG_\bfz = (\bbV_\bfz, \bbE_\bfz)$ in the flat state space $\calZ$, where $\bbV_\bfz$ and $\bbE_\bfz$ denote the sets of nodes and edges, respectively, by solving the BVP and dynamics propagation problems for a polynomial \textbf{local ``flat" path} $\bfz_{loc}(t)$. The resulting trajectory $\bfsigma_\bfz(t)$ is a continuous piecewise-polynominal, consisting of $M$ local ``flat" paths found on $\bbG_\bfz$:
\begin{equation} \label{eq:piecewise_poly_traj}
\begin{aligned}
        \bfsigma_\bfz(t) &= \left\{(\bfz_i(t
), t_i)\right\}^{M}_{i=1}, \\
    \bfw_\bfz(t) &= \left\{ \bfw_i(t), t_i)\right\}^{M}_{i=1},
\end{aligned}
\end{equation}
with $0 = t_0 < t_1 < \ldots < t_{M}$ where the pseudo-control $\bfw_i(t)$ is applied in the time interval $[t_{i-1}, t_i]$, for $i = 1, \ldots, M$.
The trajectory $\bfsigma_\bfz(t)$ is converted back to a trajectory $\bfsigma(t)$ with control $\bfu(t)$ as follows:
\begin{equation}\label{eq:sigma_conversion}
\begin{aligned}
    \bfsigma(t) &= \bfalpha(\bfsigma_\bfz(t), \bfw(t)), \\
    \bfu(t) &= \bfbeta(\bfsigma_\bfz(t), \bfw(t)),
\end{aligned}    
\end{equation} 
which is possible because the pseudo-control input $\bfw$ is the $r$-th order derivative of $\bfy$ with $r \MATHREV{\;\geq\;} \max(l,m)$ in Eq. \eqref{eq:pseudo_control}.

%%%%%%%%%%%%%%%%%%%%%%%%%%%%%%%%%%%%%%%%%%%%%%%%%%%%%%%%%%%%%%%%%%%%%%%%%%%%%%%%%%
\subsection{\textsc{FlaskExtend} Subroutine on Flat Output Space}
\label{subsec:flat_sbmp}
The goal of the \textsc{FlaskExtend} subroutine, described in \cref{alg:steering}, is to find a collision-free dynamically feasible \textbf{local ``flat" path} $\bfz_{loc}(t), t\in [0,T]$ with a time duration $T$, that connects an existing node $\bfz_0 \in \bbV$ to a new node $\bfz_f$. A motion $\bfz_{loc}(t)$ can be generated by either sampling $\bfz_f$ and solving a BVP problem (\cref{subsubsec:solving_flat_bvp}) or sampling a pseudo-control input $\bfw_0$ and integrating the flat state dynamics (\cref{subsubsec:flat_dyn_prop}). For either case, we show that a closed-form expression of $\bfz_{loc}(t)$ can be obtained for parallelized collision checking in \textsc{FlaskCC} subroutine in \cref{subsec:vectorized_cc}. If $\bfz_{loc}(t)$ is valid, the node $\bfz_f$ is added to $\bbV_\bfz$ while the edge $(\bfx_0, \bfx_f)$ associated with the local ``flat" path $\bfz_{loc}(t)$ is added to $\bbE_\bfz$.

\begin{algorithm}[t]
    \caption{\textsc{FlaskExtend}}
    \label{alg:steering}
    \KwIn{The planning graph/tree $\bbG_\bfz = (\bbV_\bfz, \bbE_\bfz)$}
    \KwOut{Updated $\bbG_\bfz$}
    \DontPrintSemicolon
    
    \If{Sample $\bfz_f \in \calZ$}
        {
            Pick an existing node $\bfZ_{0} \in \bbV_\bfz$\\
            \Comment*[l]{\small \BLUE{Solve BVP in closed form}}
            $\bfz_{loc}(t) \gets$ Eq. \eqref{eq:optimal_traj} with a sampled $T$ or an optimal \REV{$T = T^*$} from \eqref{eq:optimal_time}.

            $\bfw_{loc}(t) \gets$ Eq. \eqref{eq:optimal_control}.
            
        }
    \If{Sample $(\bfz_0, \bfw_0, T) \in \bbV_\bfz \times \bbR^n \times \bbR_{+}$}
        {
            \Comment*[l]{\small \BLUE{Analytically propagate dynamics}} %Dynamics integration in closed form
            $\bfz_{loc}(t) \gets$ Eq. \eqref{eq:poly_traj_sampled_control}.
    
            $\bfw_{loc}(t) \gets \bfw_0$.
        }

    \If{\REVVV{\textbf{\textup{not }}}\textsc{FlaskCC}($\bfz_{loc}(t), \bfw_{loc}(t)$)}
    {
        $\bbV_\bfz \gets \bbV_\bfz \cup \{\bfz_f \}$
        
        $\bbE_\bfz \gets \bbE_\bfz \cup \{(\bfz_0, \bfz_f, (\bfz_{loc}(t), \bfw_{loc}(t), T)) \}$
    }
    \Return $\bbG$
\end{algorithm}

%%%%%%%%%%%%%%%%%%%%%%%%%%%%%%%%%%%%%%%%%%%%%%%%%%%%%%%%%%%%%%%%%%
\subsubsection{Solving the BVP Problem in Closed Forms} \label{subsubsec:solving_flat_bvp}
Given an existing node $\bfz_0$ and a sampled $\bfz_f$, we find a local ``flat" path $\bfz_{loc}(t), t\in [0,T]$ that connects $\bfz_0$ and $\bfz_f$ and satisfies the flat state dynamics \eqref{eq:linear_dynamics}. Consider a cost function of a motion $\bfz(t), t\in [0,T]$ with pseudo-control $\bfw(t)$ that accounts for the total control effort and the time duration~$T$ as follows:
\begin{equation} \label{eq:flat_cost}
    \calC(\bfz(t), \bfw(t), T) = \int_0^T \bfw^\top \bfR \bfw dt + \rho T,
\end{equation}
where $\bfR \succ 0$ is a user-defined symmetric positive definite weight matrix, and $\rho$ controls the trade-off between the control effort and the time it takes to finish the trajectory. We formulate a Linear Quadratic Minimum Time (LQMT) problem~\cite{Verriest1991QuadMinTime} to solve for our local ``flat" path $\bfz_{loc}(t)$ and control~$\bfw_{loc}(t)$:%
\begin{equation}
\begin{aligned}
    &\min_{\bfw(t), T}{\calC(\bfz(t), \bfw(t), T)}    \\
\text{s.t.}\quad& \dot{\bfz} = \bfA \bfz + \bfB \bfw,\\
&\bfz(0) = \bfz_0, \bfz(T) = \bfz_f,
\end{aligned} \label{eq:min_time_problem}
\end{equation}
where the matrices $\bfA$ and $\bfB$ are defined in \eqref{eq:linear_dynamics}.
Let us define  $\bfd_T = \bfz_f - e^{\bfA T} \bfz_0$ and the Gramian matrix 
\begin{equation} \label{eq:grammian}
    \bfG_T = \int_0^T e^{\bfA t} \bfB \bfR^{-1}\bfB^T e^{\bfA^\top t}dt.
\end{equation}
The duration $T$ can be either a) set to a fixed or sampled value or b) optimized for a minimum-time trajectory, as shown~next. \REV{Given a sample $\bfz_f$, the cost $\calC_{loc}$ of the local path $\bfz_{loc}(t)$ is commonly used to choose an existing node $\bfz_0$ on the graph $\bbG_\bfz$, \eg those in the neighborhood of $\bfz_f$ with $\calC_{loc}$ smaller than a threshold~$\zeta$ (see \cref{sec:theoretical_analysis} for how to choose the value of $\zeta$ as the number of nodes $N = |\bbV_\bfz|$ increases).}

\paragraph{Fixed-time Optimal Local Paths} For a fixed time duration $T$, the LQMT problem \eqref{eq:min_time_problem} has a closed-form solution, by following~\cite{Verriest1991QuadMinTime}, for the optimal pseudo-control input:
\begin{equation} \label{eq:optimal_control}
    \bfw_{loc}(t) = \bfR^{-1}\bfB^\top e^{\bfA^\top (T-t)}\bfG^{-1}_T \bfd_T.
\end{equation}
with the optimal cost:
\begin{equation}\label{eq:optimal_cost}
    \calC_{loc}(T) = \bfd_T^\top \bfG^{-1}_T\bfd_T + \rho T,
\end{equation}
and the optimal local ``flat" path:
\begin{equation}\label{eq:optimal_traj}
    \bfz_{loc}(t) = e^{\bfA t}\bfz_0 + \bfG_t e^{A^\top (T-t)} \bfG_T^{-1}\bfd_T.
\end{equation}
Since the matrix $\bfA$ is nilpotent, \emph{i.e.,} $\bfA^r = 0$, we have  $e^{\bfA t}~=~\sum_{j=0}^{r-1} \frac{\bfA^j t^j}{j!}$. Therefore, the Grammian $\bfG_T$ and $\bfd_T$ become polynomials of the time duration $T$. The optimal control input $\bfw_{loc}(t)$ becomes a $(r-1)$th order polynomial while the optimal flat output trajectory $\bfy_{loc}(t)$ is a $(2r-1)$th order polynomial, which is compatible with \scSIMD-based collision checking in  \cref{subsec:vectorized_cc}.

\paragraph{Minimum-time Optimal Local Paths} If we have the freedom to choose the duration $T$, we can solve the following equation for a minimum time \REV{$T = T^*$}~\cite{Verriest1991QuadMinTime}:
\begin{equation} \label{eq:optimal_time}
    \frac{d}{dT} \calC_{loc}(T) = 0, \qquad T > 0.
\end{equation}

For an arbitrary value of the pseudo-control order $r$, the minimum-time condition \eqref{eq:optimal_time} can be solved for a positive $T^*$ using a numerical root-finding solver~\cite{nocedal1999numerical}, that is amenable to \scSIMD parallelization such as L-BFGS~\cite{zhu1997lbfgsb}. More notably, many common systems such as manipulators (\cref{example:manipulator}) or unicycles (\cref{example:unicycle}) have $r=2$ while systems with higher $r$, such as quadrotors (\cref{example:quadrotor}), can reduce the pseudo-control order to $r=2$ to simplify motion planning in practice, as shown in~\cite{liu2017search}. For such cases, \cref{example:min_acc_time} below shows the optimal local ``flat" path, pseudo-control input, and optimal duration, where the condition \eqref{eq:optimal_time} is equivalent to solving a $4$th-order polynomial in \eqref{eq:solve_T_r2} for a positive root, which can also be done in closed forms.
As a result, solving for an optimal time \REV{$T = T^*$} is suitable for \scSIMD parallelization.

\begin{example}[Fixed-time and minimum-time local paths] \label{example:min_acc_time}
We consider a \emph{special case with $r = 2$} and $\bfR = \bfI_n$, \emph{e.g.}, for a fully-actuated manipulator or a unicycle. Given a flat output $\bfy$, the flat state is $\bfz = (\bfy, \dot{\bfy})$ with the pseudo-control input~$\bfw$ defined as the acceleration: $\bfw = \ddot{\bfy}$. Similar to~\cite{liu2017search} (for $n=3$), the flat state $\bfz$ follows the linear dynamics \eqref{eq:linear_dynamics} with:
\begin{equation}
    \bfA =\begin{bmatrix}
    \bf0 & \bfI_n \\
    \bf0 & \bf0 
\end{bmatrix}, \quad \bfB = \begin{bmatrix}
    \bf0 \\ \bfI_n
\end{bmatrix}.
\end{equation}
Since $\bfA^j = 0$ for all $j\geq 2$, we have $e^{\bfA t} = \bfI_{2n} + \bfA t$ and $e^{\bfA^\top t} = \bfI_{2n} + \bfA^\top t$, leading to:
\begin{equation*}
    \begin{aligned}
        \bfd_T &= \bfz_f - (\bfI_{2n} + \bfA T) \bfz_0 =  \bfz_f - \bfz_0 - T\bfA \bfz_0,\\
        \bfG_T & = \int_0^T e^{\bfA t} \bfB \bfR^{-1} \bfB^T e^{\bfA^\top t}dt = \begin{bmatrix}
                    \frac{T^3}{3}\bfI_n & \frac{T^2}{2}\bfI_n \\
                    \frac{T^2}{2}\bfI_n & T\bfI_n
                \end{bmatrix}.
    \end{aligned}
\end{equation*}
The matrix inverse of $\bfG_T$ can be derived using the Schur complement~\cite{petersen2008matrix}: $\bfG^{-1}_T =   \begin{bmatrix}
                    \frac{12}{T^3}\bfI_n & -\frac{6}{T^2}\bfI_n \\
                    -\frac{6}{T^2}\bfI_n & \frac{4}{T}\bfI_n
                \end{bmatrix}.$
The optimal control input \eqref{eq:optimal_control} becomes:
\begin{equation*}
\begin{aligned}
    \bfw_{loc}(t) &= \bfR^{-1}\bfB^\top e^{\bfA^\top (T-t)}\bfG^{-1}_T \bfd_T, \\
    &=\begin{bmatrix}
        -\frac{12\bfI_n}{T^3}  &  \frac{6\bfI_n}{T^2}
    \end{bmatrix} \bfd_T t + \begin{bmatrix}
        \frac{6\bfI_n}{T^2} &  - \frac{2\bfI_n}{T}
    \end{bmatrix} \bfd_T.
\end{aligned}
\end{equation*}
This leads to a minimum-acceleration local ``flat" path $\bfz_{loc}(t)$ as follows: 
\begin{equation*}
\begin{aligned}
        \bfz_{loc}(t) &= \bfG_t e^{A^\top (T-t)} \bfG_T^{-1}\bfd_T + e^{\bfA t}\bfz_0 ,\\
& = \scaleMathLine[0.85]{\begin{bmatrix}
    \begin{bmatrix}
        -\frac{2\bfI_n}{T^3}  &  \frac{\bfI_n}{T^2}
    \end{bmatrix} \bfd_T t^3 + \begin{bmatrix}
        \frac{3\bfI_n}{T^2} &  - \frac{\bfI_n}{T}
    \end{bmatrix} \bfd_T t^2 + \dot{\bfy}_0 t + \bfy_0 \\
    \begin{bmatrix}
        -\frac{6\bfI_n}{T^3}  &  \frac{3\bfI_n}{T^2}
    \end{bmatrix} \bfd_T t^2 + \begin{bmatrix}
        \frac{6\bfI_n}{T^2} &  - \frac{2\bfI_n}{T}
    \end{bmatrix} \bfd_T t + \dot{\bfy}_0
\end{bmatrix}.}
\end{aligned}
\end{equation*}
In other words, the optimal flat output is: 
\begin{equation*} 
    \bfy_{loc}(t) = \begin{bmatrix}
        -\frac{2\bfI_n}{T^3}  &  \frac{\bfI_n}{T^2}
    \end{bmatrix} \bfd_T t^3 + \begin{bmatrix}
        \frac{3\bfI_n}{T^2} &  - \frac{\bfI_n}{T}
    \end{bmatrix} \bfd_T t^2 + \dot{\bfy}_0 t + \bfy_0.
\end{equation*}
The optimal cost function \eqref{eq:optimal_cost} is calculated as:
\begin{equation*}
\begin{aligned}
\calC_{loc}(T) &= \bfd_T^\top \bfG^{-1}_T\bfd_T + \rho T, \\
    &= \frac{12\Vert \bfy_f - \bfy_0 \Vert^2}{T^3} - \frac{12 (\dot{\bfy}_f + \dot{\bfy}_0)^\top (\bfy_f - \bfy_0)}{T^2} \\
    & \qquad + \frac{4(\Vert \dot{\bfy}_0 \Vert^2 + \dot{\bfy}_0^\top \dot{\bfy}_f + \Vert \dot{\bfy}_f \Vert^2)}{T} + \rho T.
\end{aligned}
\end{equation*}

To find a minimum time $T$, we solve $d\calC_{loc}/dT$ for a positive real root $T^*$:
\begin{equation*}
\begin{aligned}
    \frac{d\calC_{loc}(T)}{dT} &= -\frac{36\Vert \bfy_f - \bfy_0 \Vert^2}{T^4} + \frac{24 (\dot{\bfy}_f + \dot{\bfy}_0)^\top (\bfy_f - \bfy_0)}{T^3} \\
    & \qquad - \frac{4(\Vert \dot{\bfy}_0 \Vert^2 + \dot{\bfy}_0^\top \dot{\bfy}_f + \Vert \dot{\bfy}_f \Vert^2)}{T^2} + \rho,
\end{aligned}
\end{equation*}
which is equivalent to solving a $4$th-order polynomial with closed-form solutions:
\begin{equation} \label{eq:solve_T_r2}
    \begin{aligned}
        &-36\Vert \bfy_f - \bfy_0 \Vert^2 + 24 (\dot{\bfy}_f + \dot{\bfy}_0)^\top (\bfy_f - \bfy_0)T  \\
        &\quad - 4(\Vert \dot{\bfy}_0 \Vert^2 + \dot{\bfy}_0^\top \dot{\bfy}_f + \Vert \dot{\bfy}_f \Vert^2)T^2 +\rho T^4 = 0.
    \end{aligned}
\end{equation}
The detailed derivation can be found in \cref{subsec:min_time_derivation}.
\end{example}

\paragraph{Remarks} \REV{We note that the closed-form polynomial solution \eqref{eq:optimal_traj} of the LQMT problem \eqref{eq:min_time_problem} is derived without considering additional constraints such as collision avoidance, velocity and acceleration limits. In our approach, such constraints \REVVV{are described by an occupied regions in the flat state space (see \cref{def:Z_obs} in \cref{sec:theoretical_analysis})} and enforced via collision checking in \cref{subsec:vectorized_cc}. If the trajectory $\bfz_{loc}(t)$ in \eqref{eq:optimal_traj} from $\bfz_0$ to $\bfz_f$ violates a constraint, it will be considered invalid. However, it does not necessarily mean that $\bfz_f$ is unreachable from $\bfz_0$. If there exists a constrained optimal trajectory $\bfz^*_{loc}(t)$, our \cref{theorem:prob_exhausitivity}~in~\cref{sec:theoretical_analysis} shows that as the number of samples $N = |\bbV_\bfz|$ goes to infinity, the probability of having a piecewise-polynomial trajectory that is close to $\bfz^*_{loc}(t)$ (see \cref{def:delta_clearance}) tends to $1$, \ie $\bfz_f$ can still be reached via a sequence of nodes.% with high probability.

Interestingly, recent results from optimal control~\cite{beaver2024optimal} find an optimal piecewise-polynomial local path under constraints for differentially flat systems by carefully switching between modes, each of which corresponds to a polynomial motion~primitive. This approach solves an optimality condition on the~cost function and the constraints to find an optimal mode schedule and is shown to generate an optimal trajectory for low-DOF systems with simple closed-form constraints in a few milliseconds. 
However, it is challenging to design such a mode-switching mechanism for high-DOF robots with complex geometry due to complicated forward kinematics, especially when the local path generation time is often restricted to the nanosecond range. Intriguingly, from a sampling-based perspective, our \cref{theorem:prob_exhausitivity} (see \cref{sec:theoretical_analysis}) shows that such an optimal piecewise-polynomial local path can be found on our planning graph with high probability, \ie a sequence of nodes on our graph is equivalent to the mode-switching schedule from~\cite{beaver2024optimal}. Nevertheless, this is an exciting direction for generating a constrained local path $\bfz_{loc}(t)$ in our framework that we leave for future work.
}

\subsubsection{Dynamics Propagation in Closed-Forms}
\label{subsubsec:flat_dyn_prop}
\REV{As BVP problems are challenging to solve, most existing work~\cite{li2016sst} resorts to dynamics propagation, where the planning tree is extended by sampling a constant control input and integrating the robot dynamics over small time steps using numerical methods. While our \emph{BVP closed-form solution} in~\cref{subsubsec:solving_flat_bvp} \emph{is the main focus of our paper}, this section shows that transforming the kinodynamic planning problem to the flat state space (\cref{problem:flatspace_planning}) also leads to closed-form dynamics propagation, and hence gets rid of numerical approximations with potentially high accumulated errors. }
Instead of sampling the original control input, we sample the pseudo-control input $\bfw_{loc}(t) = \bfw_0$ and the duration $T$, leading to a closed-form local path:
\begin{equation} \label{eq:poly_traj_sampled_control}
\bfz_{loc}(t)=\left(\bfy_{loc}, \dot{\bfy}_{loc}, \ldots, \bfy^{(r-1)}_{loc}\right),
\end{equation}
where $\bfy_{loc}(t) = \frac{\bfw_0}{r!}t^r + \sum_{i = 0}^{r-1} \frac{\bfy^{(i)}_0}{i!} t^{i}$.
These local ``flat" paths can be used in tree-based kinodynamic planners such as Stable Sparse RRT~\cite{li2016sst} and are also suitable for parallelized forward kinematics and collision checking in \cref{subsec:vectorized_cc}. 
\REV{Propagation-based planners tend to ``wander" in the state space due to \REVV{the suboptimality of random control with} small time steps $T$. Interestingly, our BVP solutions can help by finding shortcuts to the goal rather than keep expanding the tree, and therefore reduce the planning time significantly as illustrated in~\cref{subsec:compare_lowdim}.}

\subsection{Vectorized Collision Checking} \label{subsec:vectorized_cc}
To perform collision checking on a local ``flat" path $\bfz_{loc}(t)$  with pseudo-control $\bfw_{loc}(t)$~(\cref{alg:collision_checking_poly}), we convert it back to the original state space $\calX$ and obtain the corresponding \mbox{closed-form} local path and control: 
\begin{equation} \label{eq:orig_local_path}
    \begin{aligned}
        \bfx_{loc}(t) &= \bfalpha(\bfz_{loc}(t), \bfw_{loc}(t)) \\
        \bfu_{loc}(t) &= \bfbeta(\bfz_{loc}(t), \bfw_{loc}(t)).
    \end{aligned}
\end{equation}
We next discretize $\bfx_{loc}(t)$ and $\bfu_{loc}(t)$ at $N$ times: $t_1, t_2, \ldots, t_N~\in~[0,T]$, grouped into several spatially distributed batches of size $K$: $$\bfb_i~=~\left\{t_i, t_{\left\lceil\frac{N}{K}\right\rceil + i}, t_{2\left\lceil\frac{N}{K}\right\rceil+i}, \ldots, t_{(K-1)\left\lceil\frac{N}{K}\right\rceil+i}\right\},$$ for $i = 0, \ldots, \left\lceil\frac{N}{K}\right\rceil$. For each batch $\bfb_i$, we obtain a set of samples $\{\bfx_{loc}(t_j)\}_{t_j\in \bfb_i}$ and perform fast parallelized collision checking via \scSIMD instructions, as illustrated in \cref{fig:cc_flat}. \NEW{The batch size $K$ is the number of floats that a \scSIMD register can store, \emph{e.g.}, $K = 8$ for the commonly used AVX2 instruction set}. The robot's geometry is represented by a set of \REV{SIMD-compatible primitive shapes such as spheres or capsules}, whose poses are calculated via forward kinematics for multiple states and checked for collisions with obstacles using \scSIMD parallelism. If any state in batch $\bfb_i$ leads to collisions with an obstacle, we terminate the collision checking subroutine early and move on to other local paths. 

\begin{algorithm}[t]
    \caption{\textsc{FlaskCC}}
    \label{alg:collision_checking_poly}
         \DontPrintSemicolon
    \KwIn{Flat state path $\bfz_{loc}(t)$ with pseudo-control~$\bfw(t)$}
    \KwOut{Whether $\bfz_{loc}(t)$ collides with obstacles}
    \Comment*[l]{\small \BLUE{Convert to the original state space}}
    $\bfx_{loc}(t) \gets \bfalpha(\bfz_{loc}(t), \bfw_{loc}(t))$, 
    
    $\bfu_{loc}(t) \gets \bfbeta(\bfz_{loc}(t), \bfw_{loc}(t))$ 
    
    \For{$i \in \{0, \ldots, \left\lceil\frac{N}{K}\right\rceil\}$}
    {\Comment*[l]{\small \BLUE{Parallel checking via \scSIMD~\cite{thomason2024vamp}}}
        \If{$\exists j\in \bfb_i: \bfx_{loc}(t_j) \in \calX_{obs}$}{
        \Return true 
    } 
    \Comment*[l]{\small \BLUE{checking other constraints such as state and control limits}}
            \If{$\exists j\in \bfb_i$: $\bfx_{loc}(t_j), \bfu_{loc}(t_j)$ violate other constraints }{
        \Return true 
        }
    }
    \Return false
\end{algorithm}

Similarly, other constraints such as state and control limits can be checked in parallel for each batch $\bfb_i$ to validate $\bfz_{loc}(t)$. 
\REV{
In general, the conversions \eqref{eq:flat_output_def} and \eqref{eq:state_control_from_flat_output} are nonlinear and might complicate the exact conversions of state and control limits on $(\bfx, \bfu)$ to corresponding flat state and pseudo-control limits on $(\bfz, \bfw)$. For many common robot systems, an upper bound on the conversions \eqref{eq:state_control_from_flat_output} can be obtained to estimate the limits in the flat state space as the flat output $\bfy$ is often part of the original state $\bfx$ (see Examples~\ref{example:manipulator},~\ref{example:unicycle} and \ref{example:quadrotor}). For example, any limits on the joint angles and velocities of a manipulator in~\cref{example:manipulator} can be directly translated to the flat state or any limits on the speed $v$ of a unicyle in~\cref{example:unicycle} can be used to bound the first derivative of the flat out put $(\dot{x}, \dot{y})$. Furthermore, since the flat output describes certain physical properties of the robot, \eg the position and yaw angle of a quadrotor in~\cref{example:quadrotor}, users can also directly specify the flat state and pseudo-control constraints, \eg maximum linear velocity and yaw rates.
For general nonlinear systems, this issue can be addressed by using large limits in the flat state space, and checking the local path $(\bfx_{loc}(t),  \bfu_{loc}(t))$ in \eqref{eq:orig_local_path} against the original state and control limits on $(\bfx, \bfu)$ via sampling. 
However, there is a trade-off between the size of the flat state space and the risk of missing valid paths, which can be tuned based on the planner's performance, \eg the larger flat state space may lead to longer planning time but reduce the risk. 
}

\begin{algorithm}[t]
    \caption{\textsc{Trajectory \REV{Postprocessing}}}
    \label{alg:traj_simplification}
     \DontPrintSemicolon
    \KwIn{A collision-free piecewise-polynomial trajectory $\bfsigma_\bfz(t) = \left\{(\bfz_i(t
), t_i)\right\}^{M}_{i=1}$, with control inputs $\bfw_\bfz(t) = \left\{ \bfw_i(t), t_i)\right\}^{M}_{i=1},$}
    
    \For{$i \in \{1, \ldots, M\}$}
    {
        \For{$j \in \{M, \ldots, 1\}$}
        {
            Calculate $\bfz_{ij}(t)$ from Eq. \eqref{eq:optimal_traj} with $\bfz_0 = \bfz_i(t_{i-1})$, $\bfz_f = \bfz_j(t_{j})$, and a time $T_{ij} = t_j - t_{i-1}$ or an optimal $T_{ij}$ from \eqref{eq:optimal_time}.

            Calculate $\bfw_{ij}(t)$ from Eq. \eqref{eq:optimal_control}.
            
            \Comment*[l]{\small \BLUE{Bypass unnecessary motions \REV{if the trajectory $\bfz_{ij}(t)$ is valid}}}
            \If{\REVVV{\textbf{\textup{not }}}\textsc{FlaskCC}($\bfz_{ij}(t), \bfw_{ij}(t)$)}
                {
                    Replace $\left\{(\bfz_i(t), t_i)\right\}^{j}_{k=i}$ by $(\bfz_{ij}(t), t_{i-1} + T_{ij})$.
                    
                    Replace $\left\{(\bfw_i(t), t_i)\right\}^{j}_{k=i}$ by $(\bfw_{ij}(t), t_{i-1} + T_{ij})$.
                }
        }
    }
    % \Return false
\end{algorithm}

\subsection{Trajectory \REV{Postprocessing}} \label{subsec:traj_simplification}
Our kinodynamic planning approach in \cref{subsec:flat_planning_problem} returns a piecewise-polynomial trajectory: $\bfsigma_\bfz(t) = \left\{(\bfz_i(t
), t_i)\right\}^{M}_{i=1}$, for $0 = t_0 < t_1 < \ldots < t_{M}$ where the polynomial $\bfz_i(t
), t\in [t_{i-1}, t_i]$ corresponds to the pseudo-control $\bfw_i(t)$ and time duration $T_i = t_i - t_{i-1}$. 
Due to the sampling-based nature of our approach, the trajectory might contain unnecessary local paths and often requires further \REV{postprocessing}, \emph{e.g.}, by seeking shortcuts between nodes. \NEW{For completeness, a simple \REV{trajectory postprocessing} scheme is provided in \cref{alg:traj_simplification} such that} if there exists a collision-free trajectory $\bfz_{ij}(t)$ that connects two nodes $\bfz_i(t_{i-1})$ and $\bfz_j(t_{j})$, the local paths $\bfz_{k}$, $i\leq k\leq j$ in between can simply be bypassed by $\bfz_{ij}(t)$. \NEW{We note that our approach is general and therefore, compatible with other complex trajectory shortening schemes} such as~\cite{geraerts2007creating, hauser2010trajsim}.

\section{\REV{Probabilistic Exhaustivity and Optimality Analysis}} \label{sec:theoretical_analysis}
\REV{
In this section, we will examine the \emph{probabilistic exhaustivity} of our \textbf{\FLASK} framework (\cref{subsec:prob_exhaus}), which is a key concept for proof of optimality common among asymptotically optimal planners~\cite{bekris2020AOSurvey, karaman2011optimalmp, janson2015fmt}. 
As our approach is general and compatible with most sampling-based motion planners, we provide an outline of the \emph{optimality analysis} in \cref{subsec:optimality} based on the proven probabilistic exhaustivity property.
}
\subsection{Probabilistic Exhaustivity} \label{subsec:prob_exhaus}
\REV{
The probabilistic exhaustivity property, introduced in~\cite{kavraki1998analysis,schmerling2015optimal_driftless}, is stronger than probabilistic completeness~\cite{lavalle2006planning}. 
While probabilistic completeness only requires \emph{one} solution to be found, \emph{probabilistic exhaustivity} requires that \emph{any} trajectory~$\bfpi$, \REVV{with a $\delta$-clearance to the occupied regions of the state space (see \cref{def:delta_clearance})}, can be approximated arbitrarily well with high probability by a piecewise trajectory $\bfsigma$, composed of BVP solutions $\bfsigma(t) = \left\{(\bfx_i(t
), t_i)\right\}^{M}_{i=1}$ connecting $M+1$ nodes $\{\bar{\bfx}_i\}_{i=0}^M$ on the graph, as the number of nodes tends to infinity.
}
\REV{
Our analysis relies on the following assumptions and definitions.
\begin{assumption}\label{assumption:pseudocontrol}
    The pseudo-control input $\bfw = \bfy^{(r)}$ in \cref{eq:pseudo_control} has $r$ strictly larger than $l$ (defined in~\cref{eq:state_control_from_flat_output}), \ie the original state $\bfx$ can be recovered using \cref{eq:state_control_from_flat_output} solely from the flat state~$\bfz$:~$\bfx~=~\bfalpha(\bfz)$. This is the common case as seen in Examples~\ref{example:manipulator}, \ref{example:unicycle}, and \ref{example:quadrotor}. 
\end{assumption}

}
\REV{
\begin{definition}[\mbox{$\delta$-clearance} trajectory~\cite{bekris2020AOSurvey, schmerling2015optimal_driftless}]\label{def:delta_clearance}
A trajectory $\bfpi$ with duration $T_{\bfpi}$ in the state space $\calX$ is called ``\mbox{$\delta$-clearance}" if the state $\bfx(t)$ stays in the $\delta$-interior of the free space $\calX_{free}$: $$\calX^\delta_{free} = \{\bfx \in \calX_{free}: \min_{\bfy \in \calX_{obs}} \Vert \bfx - \bfy \Vert \geq \delta\} \textrm{ for a } \delta > 0.$$
\end{definition}
}

\begin{definition}[$(\varepsilon, \zeta, \eta)$-close piecewise trace~\cite{schmerling2015optimal_drift}]\label{def:trace}
Given~a trajectory $\bfpi(t)$ with control $\bfu_{\bfpi}(t)$ and duration $T_{\bfpi}$, a~piecewise trajectory $\bfsigma(t) = \left\{(\bfx_i(t
), t_i)\right\}^{M}_{i=1}$ with control $\bfu_{\bfsigma}(t)=\left\{(\bfu_i(t
), t_i)\right\}^{M}_{i=1}$ and duration $T_{\bfsigma}$, connecting $M+1$ points $\{\bar{\bfx}_i\}_{i=0}^M$, is called an ``$(\varepsilon, \zeta, \eta)$-close piecewise trace" of $\bfpi$ if:
\begin{enumerate}[label=(\roman*)]
    \item The cost of $\bfsigma$ is $\varepsilon$-close to that of $\bfpi$: $$\calC(\bfsigma, \bfu_{\bfsigma}, T_{\bfsigma}) \leq (1+\varepsilon)\calC(\bfpi, \bfu_{\bfpi}, T_{\bfpi}).$$
    \item The cost of the $i$-th segment is bounded by $\zeta$: $$\calC(\bfx_i(t), \bfu_i(t), T_i) \leq \zeta, \quad \forall i \in \{1, \ldots, M\}$$ 
    \text{with } $T_i = t_i - t_{i-1}$.
    \item The maximum distance from a point on $\bfsigma$ to $\bfpi(t)$ is bounded by $\eta$:
    \begin{equation}\label{eq:l2_dist_bound}
        \max_{t^{\bfsigma} \in [0, T_{\bfsigma}]} \min_{t^{\bfpi} \in [0, T_{\bfpi}]} \Vert \bfsigma(t^{\bfsigma}) - \bfpi(t^{\bfpi})\Vert \leq \eta.
    \end{equation}
\end{enumerate}
\end{definition}

\REVVV{
\begin{definition}[Occupied region $\calZ_{obs}$]\label{def:Z_obs}
Similar to the occupied region $\calX_{obs}$ in the original state space, we define an occupied region $\calZ_{obs}$ in the flat state space. If there are only state constraints in the original state space, i.e., collision avoidance described by $\calX_{obs}$, we can define $\calZ_{obs}$ as: $\calZ_{obs} = \{\bfz~\in~\calZ: \bfalpha(\bfz) \in \calX_{obs}\}$.

In many cases, there are additional constraints such as limits on the original control input $\bfu$ or on a higher-order derivative $\bfx^{(p)},~p~\geq~1$. Such constraints, often described as $\bfgamma(\bfu) \leq \bf0$ or $\bfgamma(\bfx^{(p)}) \leq \bf0$ for some function $\bfgamma$, can also be represented by $\calZ_{obs}$ as follows. Due to \cref{eq:state_control_from_flat_output}, we have $\bfu = \bfbeta(\bfy, \dot{\bfy}, \ldots, \bfy^{(m)})$ and $\bfx = \bfalpha(\bfy, \dot{\bfy}, \ldots, \bfy^{(l)})$, i.e., the $p$th-order derivative $\bfx^{(p)}$ is a function of $\bfy, \dot{\bfy}, \ldots, \bfy^{(l+p)}$.
Therefore, additional constraints on $\bfu$ or $\bfx^{(p)}$ can be described in terms of $\{\bfy, \dot{\bfy}, \ldots, \bfy^{\max{\{m,(l+p)\}}}\}$ as:
\begin{equation}\label{eq:additional_constraint}
    \bfgamma(\bfy, \dot{\bfy}, \ldots, \bfy^{\max{\{m,(l+p)\}}}) \leq \bf0.
\end{equation}
Since we have the freedom to choose the order $r$ of the pseudo-control input $\bfw = \bfy^{(r)}$ in \cref{eq:pseudo_control}, we can choose $r \geq \left(1 + \max{\{m,(l+p)\}}\right)$ so that the constraints in \cref{eq:additional_constraint} only depend on $\bfz$: $\bfgamma(\bfz) \leq \bf0$.
Consequently, this allows us to capture the additional constraints completely by $\calZ_{obs}$:
\begin{equation*}
    \calZ_{obs} = \{\bfz~\in~\calZ: \bfalpha(\bfz) \in \calX_{obs}, \bfgamma(\bfz) \leq \bf0\},
\end{equation*}
which is implemented as part of the collision checking routine (line 6) in \cref{alg:collision_checking_poly}. A collision-free trajectory $\bfsigma_\bfz$ in the flat state space leads to a trajectory $\bfsigma$ that satisfies the original state and control constraints.
Meanwhile, the pseudo-control input $\bfw~=~\bfy^{(r)}$ remains unconstrained in our theoretical analysis.
\end{definition}
}

\begin{theorem}[Probabilistic \REVVV{Exhaustivity}] \label{theorem:prob_exhausitivity}
    Consider our sampling-based kinodynamic planning framework, \FLASK, in ~\cref{alg:main_proposed_alg} where the \textsc{FlaskExtend} subroutine in \cref{alg:steering} uses the BVP solution with optimal time $T^*$ (lines 1-4) and the cost function in \cref{eq:flat_cost}. 
    Under \cref{assumption:pseudocontrol}, \REVVV{let $\calZ_{free} = \calZ \setminus \calZ_{obs}$ \mbox{denote} the free flat state space, where the invalid region $\calZ_{obs}$ is defined in \cref{def:Z_obs}}. 
    Let $\bfpi_\bfz$ with \REVVV{pseudo-control $\bfw_{\bfpi}$ and duration $T_{\bfpi}$} be a \mbox{$\delta$-clearance} trajectory for a $\delta > 0$ with respect to $\calZ_{free}$.    
    Let $N = |\bbV_\bfz|$ denote the number of nodes/samples in our planning graph $\bbG_\bfz = (\bbV_\bfz, \bbE_{\bfz})$. 
    Given a constant $C_\mu$, define $C_{\calZ_{free}} = C_{\mu}^{-1}D^{-1}6^{rn + r^2n/2}2^{rn/2}\mu(\calZ_{free})$ with the volume of the free space $\mu(\calZ_{free})$, a constant $D = (rn + r^2n)/2$ and a user-defined parameter $\kappa \geq 0$.
\begin{enumerate}[label=(\roman*)]
    \item In the flat state space, let $\calA_N$ define the event that there exists an $(\varepsilon, \zeta_N, \eta_N)$-close piecewise trace of $\bfpi_{\bfz}$, denoted as $\bfsigma_\bfz(t) = \left\{(\bfz_i(t), t_i)\right\}^{M}_{i=1}$ with:
    \begin{equation} \label{eq:exhausitivy_bound}
        \begin{aligned}
            \zeta_N &= (1+\kappa)^{(1/D)} C_{\calZ_{free}} \left( \frac{\log N}{N} \right)^{1/D}, \quad \eta_N = C_p \zeta_N,\\
        \end{aligned}
    \end{equation}
    for some constant $M$ and $C_p$. There exists a constant $C_{\mu}$ such that as $N\rightarrow \infty$, the probability that $\calA_N$ does not occur is bounded as:
    \begin{equation} \label{eq:probability_bound_flat}
        1 - P(\calA_N) = O\left(N^{-\kappa/D} \log^{-1/D}{N}\right).
    \end{equation}
    \item Under \cref{assumption:pseudocontrol}, let $\bfpi = \bfalpha(\bfpi_{\bfz})$ be the corresponding trajectory of $\bfpi_{\bfz}$ in the original state space $\calX$. If the conversion $\bfalpha:~\calZ~\rightarrow~\calX$ is Lipschitz-continuous on $\calZ$ with a Lipschitz constant $L_{\bfalpha}$, we define an event $\calB_N$ that there exists a $(\varepsilon, \zeta_N, L_{\bfalpha}\eta_N)$-close piecewise trace of $\bfpi$ where the bounds $\zeta_N$ $\eta_N$ are calculated in \eqref{eq:exhausitivy_bound}. Then, the probability that $\calB_N$ does not occur is bounded as:
    \begin{equation} \label{eq:probability_bound_orig}
        1 - P(\calB_N) = O\left(N^{-\kappa/D} \log^{-1/D}{N}\right).
    \end{equation}
\end{enumerate}
\end{theorem}
\begin{proof}
\REVV{Our proof follows a common intuition in \textit{geometric} motion planning~\cite{kavraki1998analysis, bekris2020AOSurvey}, that as the number of samples increases, there exists a sequence of spheres along the trajectory $\bfpi_\bfz$ and contains a piecewise-linear path arbitrarily $\varepsilon$-close to $\bfpi_\bfz$ with high probability. However, this is not trivial in our~case due to the dynamics constraints. Unlike a linear edge that can be bounded easily by a sphere, the nonlinear local path $\bfz_{loc}$ is bounded by an ellipsoid in our proof instead, as shown below.  }
\begin{enumerate}[label=(\roman*), wide, labelwidth=!]
    \item \textbf{Probabilistic exhaustivity in the flat state space}:
    Consider our linear system \eqref{eq:linear_dynamics} in the flat state space $\calZ \REVV{\subset \bbR^{rn}}$ with matrices $\bfA \in \bbR ^{rn\times rn}$, $\bfB \in \bbR ^{rn\times n}$, \REVVV{and unconstrained pseudo-control input $\bfw$}. Consider the controllability matrix:
    \begin{equation}
        \begin{aligned}
            \bfC(\bfA, \bfB) &= \begin{bmatrix}
                \bfB & \bfA \bfB & \bfA^2 \bfB & \cdots & \bfA^{rn} \bfB,
            \end{bmatrix},\\
            &= \begin{bmatrix}
                        \bf0 & \bf0 & \bf0 & \bf0 & \bfI_n & \bf0 &\cdots & \bf0 \\
                        \bf0 & \bf0 & \bf0 & \bfI_n & \bf0 & \bf0 & \cdots & \bf0 \\
                        \vdots & \vdots & \vdots & \vdots & \vdots & \vdots & \ddots & \vdots \\
                        \bf0 & \bfI_n & \bf0 & \bf0 & \bf0 & \bf0 & \cdots & \bf0 \\
                        \bfI_n & \bf0 & \bf0 & \bf0 & \bf0 & \bf0 & \cdots & \bf0
                        \end{bmatrix}.
        \end{aligned}
    \end{equation}
    Clearly, we have $\rank(\bfC(\bfA, \bfB)) = rn$ with $rn$ linearly independent column vectors: $\left\{\{
    \bfb_j, \bfA\bfb_j, \ldots, \bfA^{r-1}\bfb_j\}_{j = 1}^n\right\}$, where $\bfb_j$ denotes the $j$-th column of the matrix $\bfB$ in \eqref{eq:linear_dynamics}. Therefore, our linear system \eqref{eq:linear_dynamics} is controllable with the following controllability indices:
    \begin{equation} \label{eq:control_index}
        \nu_1 = \nu_2 = \ldots = \nu_n = r.
    \end{equation}
    \REVV{The controllability indices, as shown later, allow us to bound the volume of an ellipsoid containing our local path $\bfz_{loc}(t)$.}

\begin{figure}[t]
\centering
\begin{subfigure}[t]{0.4\textwidth}
        \centering
\includegraphics[width=\linewidth]{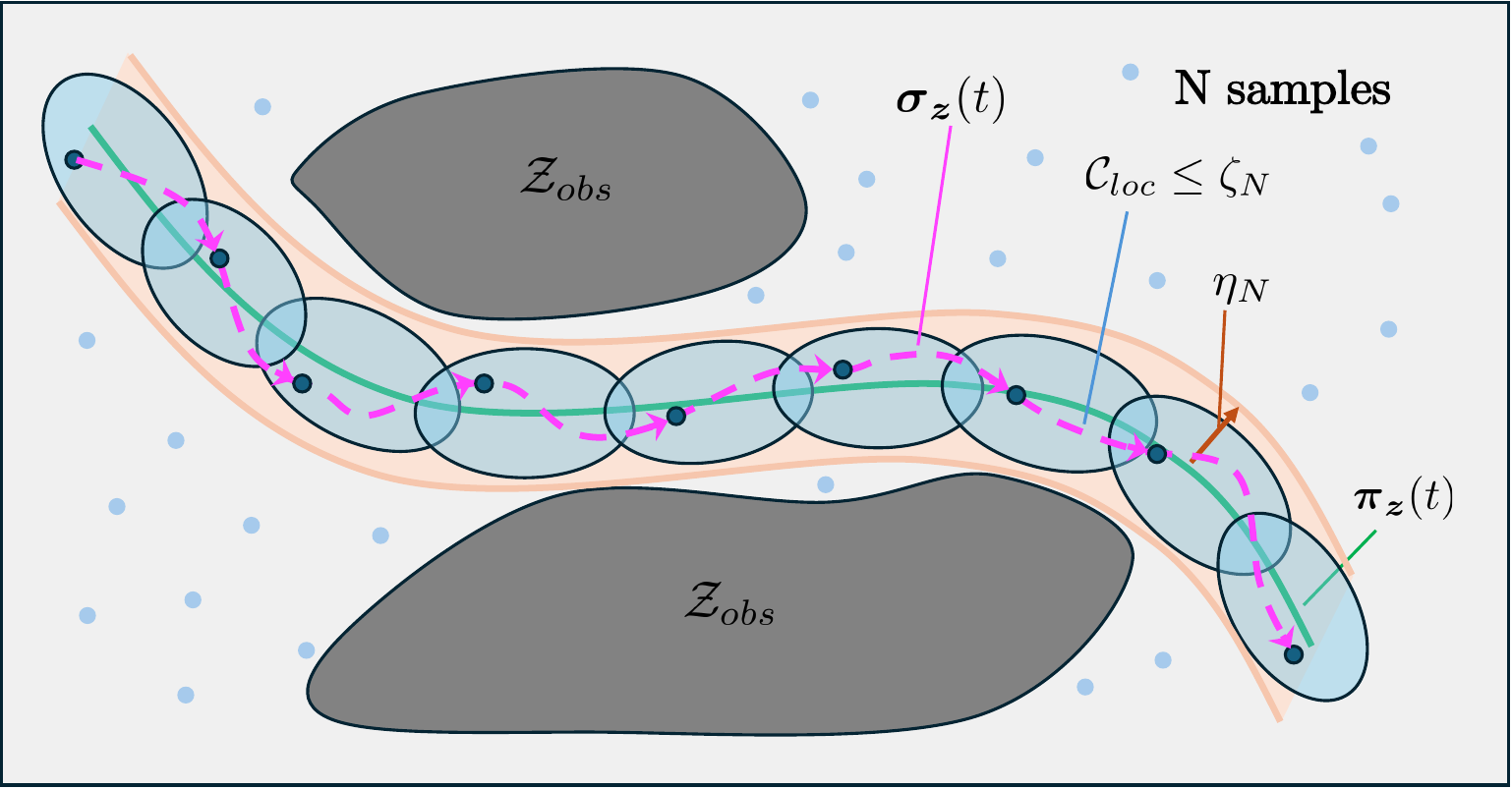}
\caption{}
\label{fig:proba_exhaus}
\end{subfigure}

\begin{subfigure}[t]{0.4\textwidth}
        \centering
\includegraphics[width=\linewidth]{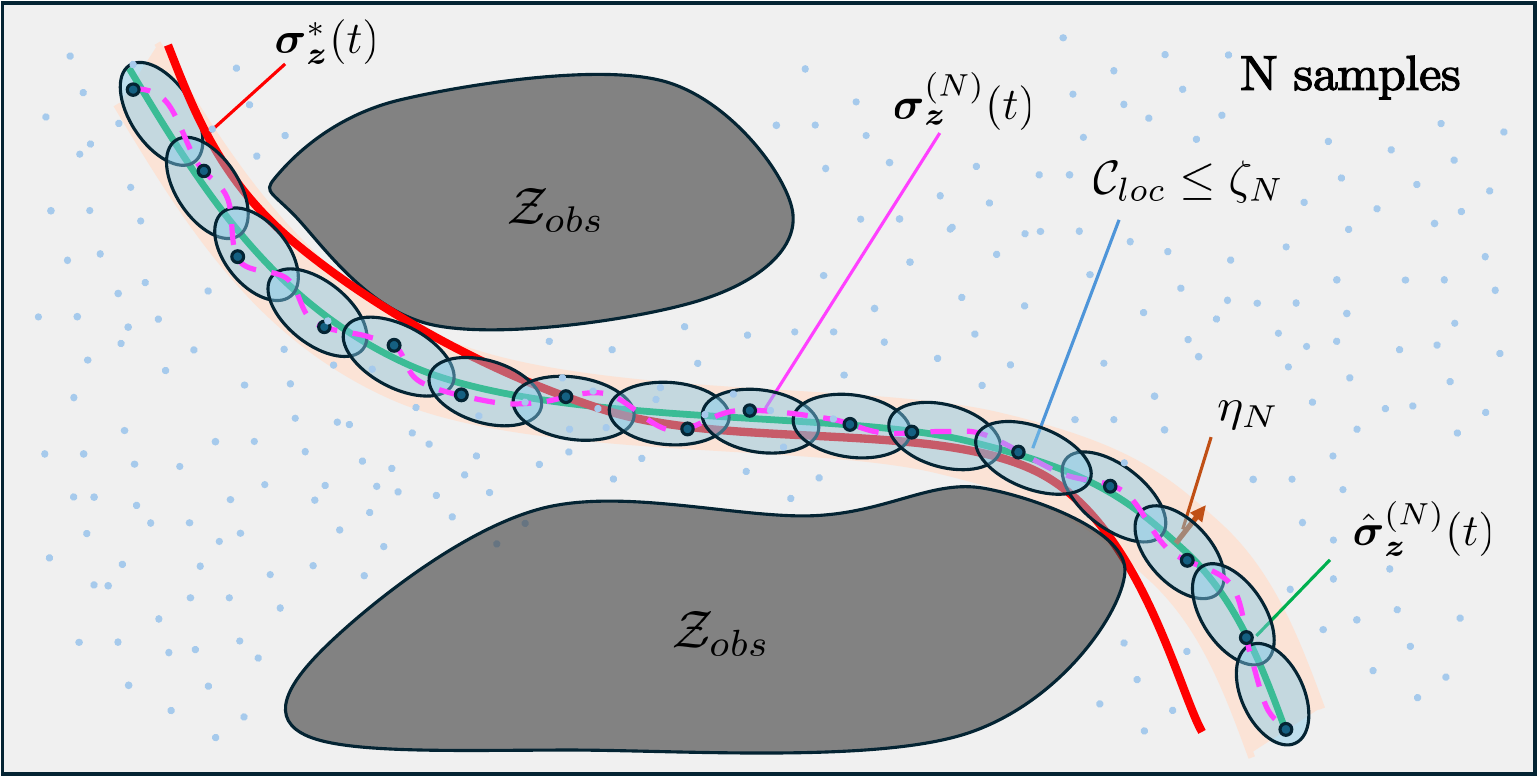}
\caption{}
\label{fig:optimality_paths}
\end{subfigure}
\caption{\REV{ Our theoretical analysis: (a) Probabilistic exhaustivity: as $N\rightarrow \infty$, the probability that there exists an $(\varepsilon, \zeta_N, \eta_N)$-close piecewise trace $\bfsigma_\bfz(t)$ (magenta), on our planning graph $\bbG$, of any trajectory $\bfpi_\bfz$ (green) with \mbox{$\delta$-clearance} approaches 1; (b) Optimality: our approach will find an $(\varepsilon, \zeta_N, \eta_N)$-close piecewise trace $\bfsigma^{(N)}_\bfz(t)$ (magenta) of a $\delta_N$-clearance trajectory $\hat{\bfsigma}^{(N)}_\bfz(t)$ (green). As $\delta_N\rightarrow 0$, the trajectory $\hat{\bfsigma}^{(N)}_\bfz(t)$ converges to the optimal trajectory $\bfsigma^*_\bfz(t)$ (red). Since $\zeta_N \rightarrow 0$, $\eta_N \rightarrow 0$ and $\delta_N \rightarrow 0$ as $N\rightarrow \infty$, our piecewise trajectory $\bfsigma^{(N)}_\bfz(t)$ converges to $\bfsigma^*_\bfz(t)$ in probability.}}
\label{fig:theoretical_guarantees}
\end{figure}

     As the number of samples $N$ goes to infinity, the duration $T$ of a local path $\bfz_{loc}$ approaches $0$ with high probability due to the higher sample density. Therefore, we are interested in how the cost $\calC_{loc}(T)$ in \cref{eq:optimal_cost} behaves near $T=0$. As the Gramian matrix $\bfG_T \succ 0$, we consider an ellipsoid $\calE_{\psi}\REVVV{(\bfz, T)}$ around a point $\bfz$, defined as $$\calE_{\psi}\REVVV{(\bfz, T)} = \{ \bfz: (\bfz' - \bfz)^\top \bfG^{-1}_T (\bfz' - \bfz) \leq \REVVV{\psi^2} \},$$
    which can be used to bound the quadratic component $\bfd_T^\top \bfG^{-1}_T\bfd_T$ of $\calC_{loc}(T)$.
    The volume of this ellipsoid is $\mu(\calE_{\psi}\REVVV{(\bfz, T)})~=~\psi^{rn}  \mu(\calS_{rn}) \sqrt{\det(\bfG_T)}$, where $\mu(\calS_{rn})$ is the volume of a unit ball $\calS_{rn}$ in $\bbR^{rn}$.
    As $T \rightarrow 0$, the value of $\det(\bfG_T)$ determines how fast the volume of $\calE_{\psi}\REVVV{(\bfz, T)}$ goes to~$0$.   
    Due to the controllability indices in \cref{eq:control_index}, the determinant of the Gramian matrix $\bfG_T$ \eqref{eq:grammian} is $\Theta(T^{\sum_{j=1}^n \nu_j^2}) = \Theta(T^{r^2n})$ as $T\rightarrow 0$ according to Lemma~III.4 in \cite{schmerling2015optimal_drift}, \ie there exist constants $ T_0, C_1, C_2 > 0$ such that
    \begin{equation*}
        C_1 T^{r^2n} \leq \det(\bfG_T) \leq C_2 T^{r^2n}, \quad\forall T < T_0.
    \end{equation*} As a result, we have:
    \begin{equation}\label{eq:lower_bound_vol}
        \mu(\calE_{\psi}\REVVV{(\bfz, T)}) \geq C_1 \mu(\calS_{rn}) \psi^{rn} T^{r^2n/2}, \quad\forall T < T_0.
    \end{equation}
    Denote $C_\mu = \sqrt{C_1}\mu(\calS_{rn})$. \cref{eq:lower_bound_vol} allows us to lower bound the rate of the quadratic cost component of the optimal cost in \cref{eq:optimal_cost} as $T\rightarrow 0$. The lower bound~\eqref{eq:lower_bound_vol} specifies the minimum coverage of the state space by the ellipsoid $\calE_{\psi}\REVVV{(\bfz, T)}$.

    \REVVV{
    As our linear system \eqref{eq:linear_dynamics} is controllable with an unconstrained pseudo-control input $\bfw$ (see \cref{def:Z_obs}) and the cost weight matrix $\bfR$ in \cref{eq:flat_cost} is symmetric positive definite, we now can apply Theorem IV.6 in \cite{schmerling2015optimal_drift} with the constant $C_\mu$ calculated above to construct a sequence of ellipsoids that generates an $(\varepsilon, \zeta_N, \eta_N)$-close piecewise trace $\bfsigma_\bfz$. While we refer the readers to~\cite{schmerling2015optimal_drift} for a detailed proof, we provide a sketch of the derivation for completeness as follows. 
    
    Consider a sequence of fixed points on the trajectory $\bfpi_\bfz$: $\{\bfz^{\bfpi}_i\}_{i=0}^{M}$ with $\bfz^{\bfpi}_i = \bfpi_\bfz(t^{\bfpi}_i)$, such that the cost of each segment $i \in \{1, \ldots, M\}$, from $\bfz^{\bfpi}_{i-1}$ to $\bfz^{\bfpi}_{i}$, is exactly $\zeta_N /2$. The number of segments $M$ is finite and bounded as $M \leq \lceil 2\calC(\bfpi_\bfz, \bfw_{\bfpi}, T_{\bfpi}) /\zeta_N\rceil$. Following~\cite{schmerling2015optimal_drift}, we construct two sequences of ellipsoids centered around~$\{\bfz^{\bfpi}_i\}_{i=0}^{M}$: $\calT = \{\calT_i\}_{i=0}^M$ with $\calT_i = \calE_{\frac{1}{6}\sqrt{\zeta_N /2}}(\bfz^{\bfpi}_i, \zeta_N/6)$ and $\calS = \{\calS_i\}_{i=0}^{M}$ with $\calS_i = \calE_{\frac{1}{12}\varepsilon\sqrt{\zeta_N /2}}(\bfz^{\bfpi}_i, \zeta_N/6)$.
    Let $\calA'_N$ be the event that every ellipsoid in $\calT$ and at least $(1-\frac{\varepsilon}{4})$ fraction of the ellipsoids in $\calS$ contains a sample in $\bbV_\bfz$. Using random geometric graph theory~\cite{penrose2003random}, Schmerling et al.~\cite{schmerling2015optimal_drift} shows that $\calA'_N$ occurs with high probability as $N\rightarrow \infty$:
    \begin{equation}
        1 - P(\calA'_N) = O\left(N^{-\kappa/D} \log^{-1/D}{N}\right).
    \end{equation}
    Given that $\calA'_N$ occurs, we can choose a sequence of points $\{\bar{\bfz}_i\}_{i=0}^M$ from $\bbV_\bfz$ by picking $\bar{\bfz}_i$ as a sample in the ellipsoid $\calS_i$ if possible, and in the ellipsoid $\calT_i$ otherwise. 
   In other words,~we compose a new sequence of ellipsoids, whose $(1-\frac{\varepsilon}{4})$ fraction is from $\calS$ and $\frac{\varepsilon}{4}$ fraction is from $\calT$.
    By connecting $\{\bar{\bfz}_i\}_{i=0}^M$ via the BVP solutions in \cref{eq:optimal_traj}, we obtain a trajectory $\bfsigma_\bfz(t) = \left\{(\bfz_i(t), t_i)\right\}^{M}_{i=1}$. Intuitively, the ellipsoids $\calS_i$'s have a smaller $O(\varepsilon)$ extent by definition (\ie closer to $\bfpi_\bfz$ with small enough~$\varepsilon$) and cover $1-O(\varepsilon)$ fraction of the segments~$\bfz_i(t)$'s. Meanwhile, the ellipsoids $\calT_i$'s have a larger $O(1)$ extent in terms of $\varepsilon$ (\ie allowing farther distance from $\bfpi_\bfz$) but only include an $O(\varepsilon)$ fraction of the segments $\bfz_i(t)$'s. This is the main intuition to show that the constructed $\bfsigma_\bfz(t)$ is an $(\varepsilon, \zeta_N, \eta_N)$-close piecewise trace of $\bfpi_\bfz$ by Theorem IV.6 in~\cite{schmerling2015optimal_drift}.
    }
    Therefore, the event $\calA'_N$ implies the event $\calA_N$ and as $N\rightarrow \infty$, the probability that $\calA_N$ \emph{does not} occur is bounded~by:
    \begin{equation*}
        P(\bar{\calA}_N) = 1 - P(\calA_N) = O\left(N^{-\kappa/D} \log^{-1/D}{N}\right).
    \end{equation*}

    \item \textbf{Probabilistic exhaustivity in the original state space}:
    As the conversion $\bfx = \bfalpha(\bfz)$ is Lipschitz-continuous on $\calZ$ with a Lipschitz constant $L_{\bfalpha}$, we have: 
    \begin{equation*}
        \Vert \bfx_1 - \bfx_2 \Vert = \Vert \bfalpha(\bfz_1) - \bfalpha(\bfz_2) \Vert \leq L_{\bfalpha}\Vert \bfz_1 - \bfz_2 \Vert, \forall \bfz_1, \bfz_2 \in \cal Z
    \end{equation*}
    
    Let $\bfsigma = \bfalpha(\bfsigma_\bfz)$ be the corresponding trajectory in the original state space of the $(\varepsilon, \zeta_N, \eta_N)$-close piecewise trace $\bfsigma_\bfz$.
    Given a time $t_{\bfsigma} \in [0, T_{\bfsigma}]$, define $t^*_{\bfpi} = \argmin_{t_{\bfpi} \in [0, T_{\bfpi}]}\Vert\bfsigma_\bfz (t_{\bfsigma}) - \bfpi_\bfz(t_{\bfpi})\Vert$. Then, we have:
        \begin{equation*}
       \Vert \bfsigma(t_{\bfsigma}) - \bfpi(t^*_{\bfpi}) \Vert \leq L_{\bfalpha}\Vert \bfsigma_\bfz(t_{\bfsigma}) - \bfpi_\bfz(t^*_{\bfpi}) \Vert \leq  L_{\bfalpha}\eta_N,
    \end{equation*}
    since $\bfsigma_\bfz$ is a $(\varepsilon, \zeta_N, \eta_N)$-close piecewise trace.
    Let $\tau^*_{\bfpi} = \argmin_{\tau_{\bfpi} \in [0, T_{\bfpi}]}\Vert\bfsigma (t_{\bfsigma}) - \bfpi(\tau_{\bfpi})\Vert$. Clearly, we have the following bound that holds \emph{for all $t_{\bfsigma} \in [0, T_{\bfsigma}]$}:
    \begin{equation*}
       \Vert \bfsigma(t_{\bfsigma}) - \bfpi(\tau^*_{\bfpi}) \Vert \leq \Vert \bfsigma(t_{\bfsigma}) - \bfpi(t^*_{\bfpi}) \Vert \leq  L_{\bfalpha}\eta_N.
    \end{equation*}
    By maximizing over $t_{\bfsigma} \in [0, T_{\bfsigma}]$, the maximum distance between a point on $\bfsigma$ to $\bfpi$ in the state space $\calX$ satisfies:
    \begin{equation*}
        \max_{t_{\bfsigma} \in [0, T_{\bfsigma}]} \min_{t_{\bfpi} \in [0, T_{\bfpi}]} \Vert \bfsigma(t_{\bfsigma}) - \bfpi(t_{\bfpi})\Vert \leq L_{\bfalpha}\eta_N.
    \end{equation*}
    As the cost function \eqref{eq:flat_cost} is defined in the flat state space, the trajectory $\bfsigma$ has the same cost as $\bfsigma_\bfz$. Therefore, an $(\varepsilon, \zeta_N, \eta_N)$-close piecewise trace $\bfsigma_\bfz$ implies an $(\varepsilon, \zeta_N, L_{\bfalpha}\eta_N)$-close piecewise trace $\bfsigma$ or in other words, the event $\calA_N$ implies the event $\calB_N$, \ie $P(\calA_N) \leq P(\calB_N)$. Therefore, we have:
    \begin{equation*}
        1 - P(\calB_N) \leq 1 -  P(\calA_N)  = O\left(N^{-\kappa/D} \log^{-1/D}{N}\right). \qedhere
    \end{equation*}
\end{enumerate}

\end{proof}

\REVV{\cref{theorem:prob_exhausitivity} shows that given any trajectory $\bfpi$ with $\delta$ clearance, as the number of samples  $N\rightarrow \infty$, there exists, with high probability, an $(\varepsilon, \zeta_N, \eta_N)$-close piecewise trace $\bfsigma$ of the trajectory $\bfpi$ (illustrated in \cref{fig:proba_exhaus}). More importantly, the bounds $\zeta_N$ and $\eta_N$ are $O\left(\left(\frac{\log{N}}{N}\right)^{1/D}\right)$, which will be shown later to be critical for our \REVVV{asymptotic} optimality analysis.}

\REVVV{
\cref{theorem:prob_exhausitivity} assumes that we can completely check if a trajectory is valid or has $\delta$-clearance, e.g., via continuous-time collision checking. However, the standard practice in sampling-based motion planning, \eg~\cite{lavalle2006planning, kuffner2000rrtconnect}, is to perform this check approximately via sampling for efficiency, as shown in \cref{alg:collision_checking_poly}, with the accuracy up to~the sampling resolution. While integrating continuous-time collision checking can avoid caveats of sampling-based collision checking, \eg tunneling effects with thin obstacles, we leave this direction for future~work.
}

\subsection{Optimality Analysis} \label{subsec:optimality}
\REVV{As our \textbf{\FLASK} framework in \cref{sec:technical_approach} (\cref{alg:main_proposed_alg}) is compatible with any sampling-based motion planners~\cite{orthey2024-review-sampling}, the optimality guarantees of our trajectory depend on the optimality guarantees of the specific planner that integrates our \textsc{FlaskExtend} subroutine.} However, we provide an outline of an optimality analysis based on random geometric graphs~\cite{penrose2003random} which many asymptotically optimal sampling-based motion planners follow (see~\cite{karaman2011optimalmp, webb2013kinodynamicRRT, janson2015fmt} for an example and ~\cite{bekris2020AOSurvey} for an excellent review of this topic). \REVV{Therefore, our framework not only transforms any sampling-based motion planner into its kinodynamic version but also preserves its optimality guarantees.}

\REV{
We define an optimal trajectory solution of \cref{problem:flatspace_planning} with respect to the cost \eqref{eq:flat_cost} in \cref{def:optimal_traj}. We note that while our cost function Eq. \eqref{eq:flat_cost} is defined in the flat state space, it bears physical meaning in the original space as the flat output $\bfy$ often describes certain physical properties of the robot, \eg the position and yaw angle of a quadrotor in Example~\ref{example:quadrotor}. For example, if the pseudo-control input $\bfw = \bfy^{(r)}$ in \eqref{eq:pseudo_control} has $r = 2, 3, 4, \ldots$, the cost \eqref{eq:flat_cost} means that we would like to find a minimum acceleration-time, jerk-time, or snap-time trajectory, respectively (\eg in \cite{mellinger2011minimumsnap, liu2017search, liu2018search}).
\begin{definition}[Optimal Trajectory]\label{def:optimal_traj}
    Denote $\calC_{\bfsigma_\bfz} = \calC(\bfsigma_\bfz(t), \bfw(t), T)$. Assume that there exists a minimum cost $\calC^* =\min_{\bfsigma_\bfz}\calC_{\bfsigma_\bfz}$, defined in Eq. \eqref{eq:flat_cost}, over all trajectories $\bfsigma_\bfz(t)$ that solve \cref{problem:flatspace_planning}. A trajectory $\bfsigma_\bfz^*(t)$ with control $\bfw^*(t)$ and its corresponding trajectory $\bfsigma^*(t) = \bfalpha(\bfsigma_\bfz^*(t), \bfw^*(t))$ with control $\bfu^*(t) = \bfbeta(\bfsigma_\bfz^*(t), \bfw^*(t))$ in the original state space are called optimal if they achieve the minimum cost $\calC^*$.
\end{definition}
}

\REV{
In the flat space, the optimal trajectory $\bfsigma_\bfz^*(t)$, defined in \cref{def:optimal_traj}, might have $0$ clearance as illustrated in \cref{fig:optimality_paths}, \ie it might touch the boundary of $\calZ_{obs}$. However, we assume that there exists a sequence of $\delta_N$-clearance trajectories $\hat{\bfsigma}_\bfz^{(N)}(t)$ with control $\hat{\bfw}^{(N)}$ and duration $T^{(N)}$ such that its cost $\calC_{\hat{\bfsigma}_\bfz^{(N)}} = \calC(\hat{\bfsigma}_\bfz^{(N)}(t), \hat{\bfw}^{(N)}(t), T^{(N)})$ satisfies: 
$\lim_{N\rightarrow \infty}{\calC_{\hat{\bfsigma}_\bfz^{(N)}}} = \calC^*,$
where $N = |\bbV_\bfz|$ denotes the number of nodes/samples in our planning graph $\bbG_\bfz = (\bbV_\bfz, \bbE_{\bfz})$. This is a common assumption in sampling-based motion planning as the $\delta_N$ clearance allows a region around $\hat{\bfsigma}_\bfz^{(N)}(t)$ with nonzero measure for sampling.

For each $\delta_N$-clearance trajectory $\hat{\bfsigma}^{(N)}_\bfz(t)$, \cref{theorem:prob_exhausitivity}(i) shows that as $N\rightarrow \infty$, the probability that there is an $(\varepsilon, \zeta_N, \eta_N)$-close piecewise trace $\bfsigma^{(N)}_\bfz(t)$ of $\hat{\bfsigma}^{(N)}_\bfz(t)$ goes to~$1$. This allows us to cover $\hat{\bfsigma}^{(N)}_\bfz(t)$ by a sequence of ellipsoids of extent $\zeta_N$ so that the piecewise trace $\bfsigma^{(N)}_\bfz(t)$ will stay inside, as illustrated in \cref{fig:optimality_paths}. Meanwhile, to maintain connectivity around $\hat{\bfsigma}^{(N)}_\bfz(t)$, the bound $\zeta_N$ can be used as thresholds for expanding (\eg for picking $\bfz_0$ in \textsc{FlaskExtend} in \cref{alg:steering}) and rewiring the planning graph $\bbG_\bfz$ to find a better cost for each neighboring node, similar to RRT*~\cite{karaman2011optimalmp} or for recursive cost updates via dynamic programming as in FMT*~\cite{janson2015fmt, schmerling2015optimal_drift}.

More importantly, as $N\rightarrow \infty$, the upper bounds $\zeta_N, \eta_N$ in \eqref{eq:exhausitivy_bound}, proportional to $\left(\frac{\log N}{N}\right)^{1/D}$, go to $0$. Therefore, the probability that the piecewise trace $\bfsigma^{(N)}_\bfz(t)$ converges to $\hat{\bfsigma}^{(N)}_\bfz(t)$ with cost $\calC_{\bfsigma^{(N)}_\bfz} \leq (1+\varepsilon) \calC_{\hat{\bfsigma}^{(N)}_\bfz}$ also approaches $1$. 
Furthermore, by choosing a decreasing clearance $\delta_N$ so that $\delta_N \rightarrow 0$ as $N \rightarrow \infty$, \eg $\delta_N = O(\eta_N)$, we have $P\left(\calC_{\bfsigma^{(N)}_\bfz} \leq (1+\varepsilon) \calC^*\right) \rightarrow 1,$ for an arbitrarily small $\varepsilon > 0$, \ie our \textbf{\FLASK} \REVV{framework} maintains asymptotic optimality. This is illustrated in \cref{fig:optimality_paths} with denser samples and smaller clearance $\delta_N$ than those of \cref{fig:proba_exhaus}, making our piecewise trajectory $\bfsigma^{(N)}_\bfz$ closer to the optimal one as $N$ increases.

Finally, the \REVVV{asymptotic} optimality guarantees can be transformed to the original state space following similar steps due to \cref{theorem:prob_exhausitivity}(ii).
This analysis importantly affirms that~our \textbf{\FLASK} framework can be transform any existing sampling-based motion planners into kinodynamic planners without breaking their asymptotic optimality guarantees.

}

\section{Evaluation}
\label{sec:evaluation}

\begin{figure}[t]
\centering
\begin{subfigure}[t]{0.5\textwidth}
        \centering
\includegraphics[width=0.5\linewidth, trim={0.95cm 2cm 0.95cm 0.6cm},clip]{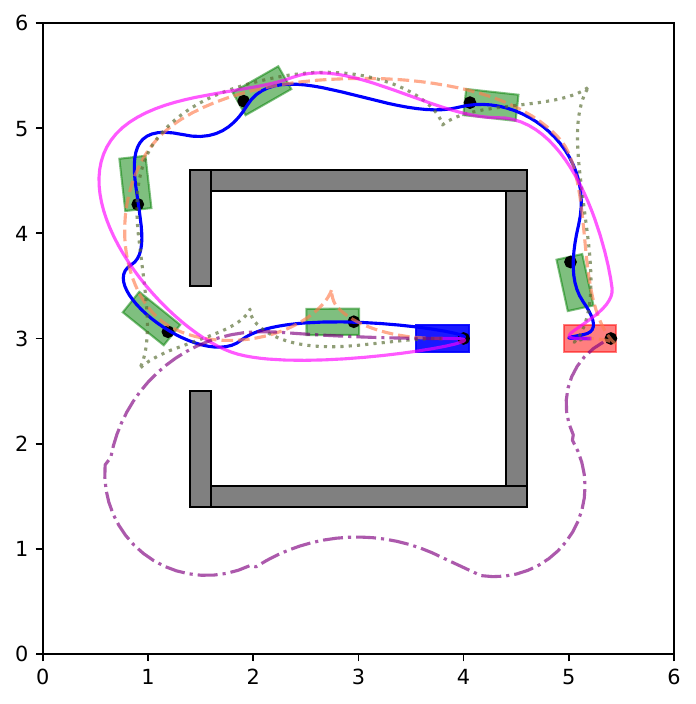}%
\hfill
\includegraphics[width=0.5\linewidth, trim={2cm 2.5cm 3cm 3cm},clip]{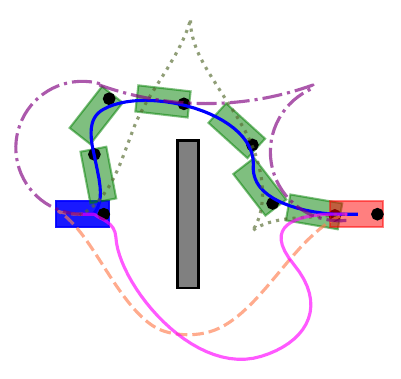}%
        \caption{Unicycle: \emph{Bugtrap} (left) \& \emph{Wall} (right).}
        \label{fig:unicycle_viz}
\end{subfigure}%

\begin{subfigure}[t]{0.5\textwidth}
        \centering
\includegraphics[width=0.5\textwidth, trim={1.2cm 2cm 1cm 0.5cm},clip]{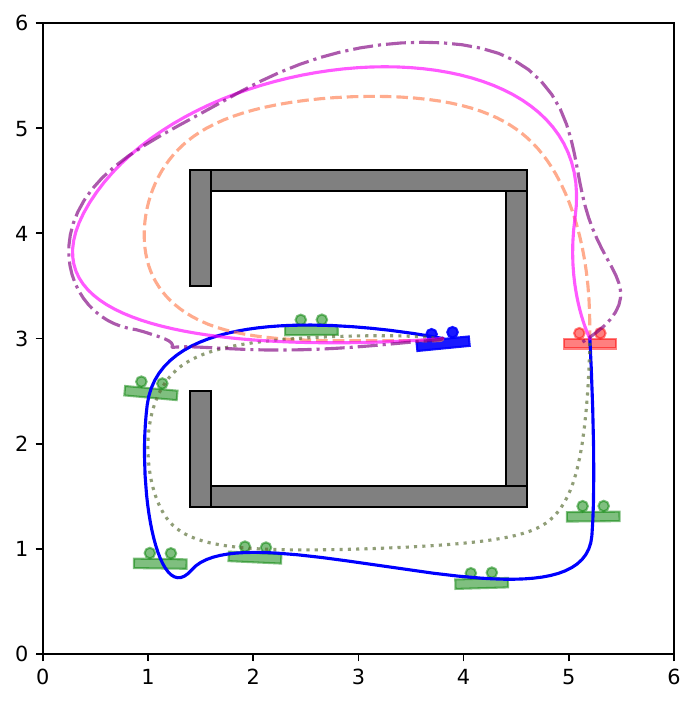}%
\hfill
\includegraphics[width=0.42\textwidth, trim={1cm 1cm 1cm 3.0cm},clip]{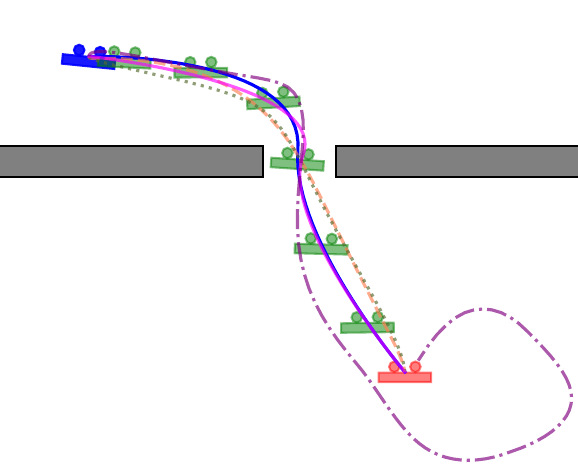}%
        \caption{2D quadrotor: \emph{Bugtrap} (left) \& \emph{Hole} (right).}
        \label{fig:quad2d_viz}
\end{subfigure}%

\begin{subfigure}[t]{0.5\textwidth}
        \centering
\includegraphics[width=0.45\textwidth, trim={4cm 6cm 3.7cm 6cm},clip]{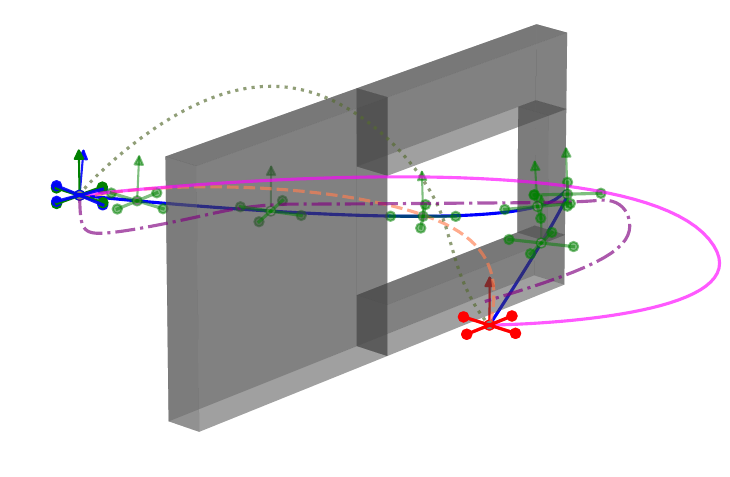}%
\hfill
\includegraphics[width=0.45\textwidth, trim={4cm 5.22cm 3.5cm 5.95cm},clip]{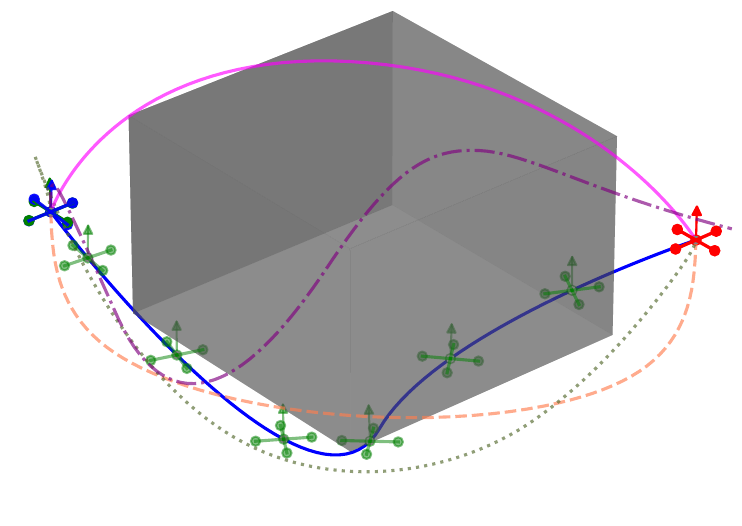}%
        \caption{3D quadrotor: \emph{Window} (left) \& \emph{Obstacle} (right).}
        \label{fig:quad3d_viz}
\end{subfigure}%
\caption{\REVVV{Visualization of our trajectories with DynoBench~\cite{ortizharo2025iDbAstar} benchmarking problems: the trajectories from our kinodynamic \REV{\FLASKRRTC} and our \REV{BVP-augmented \mbox{\FLASKSST}} are plotted as blue and magenta solid curves, respectively, while those from the iDb-A* and SST* baselines are shown as orange dashed and green dotted curves, respectively, \REV{and the Kino-PAX trajectories are plotted as purple dash-dotted curves}.}}
\label{fig:dynobench_viz}
\end{figure}

\begin{table*}[t]
    \caption{Planning performance with DynoBench benchmarks~\cite{ortizharo2025iDbAstar}. \REVVV{$^{\dagger}$For the anytime planners \mbox{iDb-A*}~\cite{ortizharo2025iDbAstar},  SST*~\cite{li2016sst}, the \emph{first solution time} and \emph{final trajectory length} within a $180$-second limit are reported.}}
    \label{table:dynobench_results}
    \centering
\begin{tabular}{ccrrrrrrrr|rrrrrr} 
		\multirow{5}{*}{Robot} & \multirow{5}{*}{Problem} & \multicolumn{2}{c}{\REVV{\texttt{FLASK-}}} & \multicolumn{2}{c}{\mbox{\REV{\FLASKSST}}} & \multicolumn{2}{c}{\REV{\mbox{\FLASKSST}}} & \multicolumn{2}{c}{\REV{\FLASKRRT}} & \multicolumn{2}{c}{iDb-A*\REVVV{$^{\dagger}$}} & \multicolumn{2}{c}{Orig. SST*\REVVV{$^{\dagger}$}} & \multicolumn{2}{c}{\REV{Kino-PAX}} \\
        %%%%%%%%%%%%%%%%%%%%%%%%%%%%%%%%%%%%%%%%%%%%%%%%%%%%%%%%
         & & \multicolumn{2}{c}{\mbox{\REVV{\texttt{RRTConnect}}}} & \multicolumn{2}{c}{\REV{(BVP-aug.)}} & \multicolumn{2}{c}{\REV{(SIMD-only)}} & \multicolumn{2}{c}{\REV{(DP-based)}} & \multicolumn{2}{c}{(DynoBench)} & \multicolumn{2}{c}{(DynoBench)} & \multicolumn{2}{c}{(\REV{GPU-based})} \\
         %%%%%%%%%%%%%%%%%%%%%%%%%%%%%%%%%%%%%%%%%%%%%%%%%%%%%%%%%%%
        \cmidrule(lr){3-4} \cmidrule(lr){5-6} \cmidrule(lr){7-8} \cmidrule(lr){9-10} \cmidrule(lr){11-12} \cmidrule(lr){13-14} \cmidrule(lr){15-16}
        &&Traj. &Plan.&Traj. &Plan.& \REV{Traj.} & \REV{Plan.} & \REV{Traj.}  & \REV{Plan.} & Traj.  & Plan.  & Traj.  & Plan. & \REV{Traj.}  & \REV{Plan.} \\
        &&Len.&Time&Len.&Time& \REV{Len.}& \REV{Time} & \REV{Len.} & \REV{Time} & Len. & Time  & Len. & Time & \REV{Len.} & \REV{Time} \\
        &&(m)&(ms)&(m)&(ms)& \REV{(m)}& \REV{(ms)} & \REV{(m)} & \REV{(ms)} & (m) & (ms) & (m) & (ms) & \REV{(m)} & \REV{(ms)} \\
        % &&&&&& & & & &  & \REVVV{(1st sol.) & & (1st sol.) & & \\
		\hline
		\hline
		\multirow{2}{*}{Uni.} & Bugtrap & $12.39$      &      $\BETTER{6.0}$     & $12.51$      &      $48.6$&  $\MATHREV{12.27}$ & $\MATHREV{407}$ & $\MATHREV{12.25}$ & $\MATHREV{565}$ &   $\BETTER{11.51}$       &       $\MATHREVV{375}$  &  $12.38$ &  $\MATHREVV{126}$  & $\MATHREV{13.77}$ & $\MATHREV{5.8}$\\
		 %& Kink & 11.55      &      10.2$ms$      &     11.65       &       0.6s  &  11.47 &  0.4s  \\
         & Wall & $4.82$      &      $\BETTER{0.6}$     & $5.24$      &      $11.0$&   $\MATHREV{5.27}$ & $\MATHREV{30}$ & $\MATHREV{5.61}$ & $\MATHREV{25}$ &  $\BETTER{3.89}$       &       $\MATHREVV{77}$  &  $4.97$ &  $\MATHREVV{67}$ & $\MATHREV{5.28}$ & $\MATHREV{2.6}$\\
		\hline
		Quad.& Bugtrap & $11.46$      &      $\BETTER{2.7}$     & $11.67$      &      $45.0$ &  $\MATHREV{11.46}$ & $\MATHREV{388}$ & $\MATHREV{11.50}$ & $\MATHREV{909}$ &   $9.69$       &       $\MATHREVV{10746}$  &  $\BETTER{9.64}$ &  $\MATHREVV{28667}$  & $\MATHREV{18.59}$ & $\MATHREV{19.7}$\\
		 %& Kink & 11.55      &      10.2$ms$      &     11.65       &       0.6s  &  11.47 &  0.4s  \\
         (2D) & Hole & $5.16$      &      $\BETTER{0.4}$     & $4.93$      &      $270.0$ &   $\MATHREV{5.14}$ & $\MATHREV{1813}$ & $\MATHREV{5.24}$ & $\MATHREV{2133}$ &  $4.62$       &       $\MATHREVV{353}$  &  $\BETTER{4.59}$ &  $\MATHREVV{17602}$ & $\MATHREV{17.82}$ & $\MATHREV{124.9}$\\
         		\hline
		Quad. & Block & $8.08$      &      $\BETTER{0.9}$    & $8.61$      &      $1.4$  &  $\MATHREV{9.56}$ & $\MATHREV{4752}$ & $\MATHREV{8.75}$ & $\MATHREV{7602}$ &   $\BETTER{7.94 }$     &       $\MATHREVV{4573}$  &  $9.32$ &  $\MATHREVV{46289}$ & $\MATHREV{10.36}$ & $\MATHREV{21.4}$ \\
		 %& Kink & 11.55      &      10.2$ms$      &     11.65       &       0.6s  &  11.47 &  0.4s  \\
         (3D)& Window & $5.49$      &      $\BETTER{0.5}$   & $6.28$      &      $22.6$   & $\MATHREV{6.96}$ & $\MATHREV{629}$ & $\MATHREV{7.98}$ & $\MATHREV{2631}$ &    $\BETTER{5.12 }$      &       $\MATHREVV{2667}$  &  $6.72$ &  $\MATHREVV{45798}$ & $\MATHREV{16.41}$ & $\MATHREV{178.6}$\\
	\end{tabular}
  % \end{adjustbox}
\end{table*}

In this section, we will evaluate the effectiveness of our \textbf{\REVV{\FLASK}} \REVV{framework, applied on different planners} with different robot platforms, from low to high dimensional state spaces in both simulation and real experiments. As our method focuses on fast kinodynamic planning on CPUs, the results from our planners and the baselines are reported on an Intel i7-6900K CPU for a fair comparison. As a reference, we also provide \REV{the results from the GPU-based Kino-PAX~\cite{perrault2025kino} planner, which is run on a GeForce RTX 4070 GPU.}

\subsection{Kinodynamic planning for low-\dof systems} \label{subsec:compare_lowdim}
In this section, we compare our approach with other kinodynamic planning algorithms provided by the \emph{DynoBench} benchmark~\cite{ortizharo2025iDbAstar}, which contains multiple low-\dof robots and benchmarking environments. We consider $3$ robot platforms: a \emph{unicycle} (\cref{example:unicycle}), a $2$D \emph{quadrotor}, and a $3$D \emph{quadrotor} (\cref{example:quadrotor}). The \emph{unicycle} robot has the following constraints: maximum linear velocity $v  \leq v_{max} = \qty{1.0}{\meter\per\second}$, maximum angular velocity $\omega \leq \qty{1.5}{\radian\per\s}$. The $3D$ \emph{quadrotor} has the following parameters: mass $m = \qty{1}{\kg}$, inertia matrix $\bfJ~=~\text{diag}([0.1, 0.1, 0.2])\unit{\kg\m\squared}$, maximum linear velocity $\Vert \bfv \Vert \leq \qty{4}{\m\per\s}$, maximum angular velocity $\Vert \bfomega \Vert \leq \qty{8}{\radian\per\s}$, maximum thrust $f = 1.5mg\;\unit{\newton}$, and maximum torque $\bftau~\leq~\qty{2}{\newton\m}$, where $g \approx \qty{9.81}{\m\per\s\squared}$ is the gravitational acceleration.
The $2$D \emph{quadrotor}'s dynamics is a special case of the $3$D \emph{quadrotor}'s with position $\bfp = [0, y, z]$, yaw angle $\psi = 0$, and the rotation matrix $\bfR~=~\bfR_\phi$ of a pitch angle $\phi$, leading to a simplified flat output $\bfy = [y,z] \in \bbR^2$. The $2$D \emph{quadrotor} is modeled after a Crazyflie~\cite{giernacki2017crazyflie} with mass $m = \qty{0.034}{\kg}$, inertia $J = \qty{1e-4}{\kg\m\squared}$, and arm length $\ell = \qty{0.1}{\m}$. The thrusts $f_1$ and $f_2$, generated from the two onboard motors, satisfy the following constraint: $0 \leq f_1, f_2 \leq 0.65mg\;\unit{\newton}$. The thrust and torque are calculated as: $f = f_1 + f_2$ and $\tau = \ell(f_1 - f_2)$.

For each platform, we consider $2$ DynoBench environments, shown in \cref{fig:dynobench_viz}: \emph{Bugtrap} and \emph{Wall} for the unicycle, \emph{Bugtrap} and \emph{Hole} for the $2$D quadrotor, and \emph{Window} and \emph{Obstacle} for the $3$D quadrotor. To be compatible with \scSIMD-only collision checking, \emph{e.g.},~\cite{thomason2024vamp}, the obstacles in each environment are represented by a set of spheres, while the unicycle and quadrotor's geometries are both modeled as single spheres.

We \REVV{apply our \textbf{\REVV{\FLASK}} framework to obtain} two planners: \REV{\textbf{\FLASKRRTC}} with BVP solutions in \cref{subsubsec:solving_flat_bvp} and \REV{three variants of \textbf{\REV{\FLASKSST}}} with dynamics propagation in \cref{subsubsec:flat_dyn_prop}.
Our kinodynamic \mbox{\textbf{\FLASKRRTC}} replaces a line segment connecting two nodes in an RRT tree by our closed-form BVP solution (local ``flat" path) in the \REV{\textsc{FlaskExtend}} subroutine (\cref{alg:steering}), using \cref{eq:optimal_traj} with an optimal $T^*$ from \cref{eq:optimal_time} and $\rho = 1$. 
\REV{For \REVV{\textbf{\mbox{\FLASKSST}}}, our first variant, \textbf{BVP-augmented \mbox{\FLASKSST}}, uses \cref{eq:poly_traj_sampled_control} to propagate robot dynamics in closed forms and takes advantage of our local ``flat" path in \eqref{eq:optimal_traj} to check whether there is a direct connection from an existing node to the goal. Meanwhile, our second variant is a \textbf{SIMD-only \mbox{\FLASKSST}} that keeps expanding the tree using vectorized collision checking via SIMD, \emph{without} leveraging the BVP solutions. To highlight the ``wandering" effect in propagation-based planners, our third variant disables best-state selection and tree pruning in SST*, effectively leading to a \textbf{DP-based \FLASKRRT}, where dynamics propagation (DP) with a sampled control is used to extend the RRT tree. 
The resulting trajectories are simplified using \cref{alg:traj_simplification}.} We compare our approach with \REV{three baseline planners: \mbox{\textbf{iDb-A*}}~\cite{ortizharo2025iDbAstar},  \textbf{SST*}~\cite{li2016sst} from OMPL~\cite{sucan2012the-open-motion-planning-library}, \REV{both without parallelized collision checking}, and a GPU-based coarse-grained parallelized kinodynamic planner, \textbf{Kino-PAX}~\cite{perrault2025kino}}. We maintain the aforementioned robot state and control constraints in our planners, enforced via our parallelized collision checking algorithm in \cref{alg:collision_checking_poly}, and in the baselines for a fair comparison.

\begin{figure*}[t]
\centering
\begin{subfigure}[t]{0.22\textwidth}
    \includegraphics[width=\textwidth]{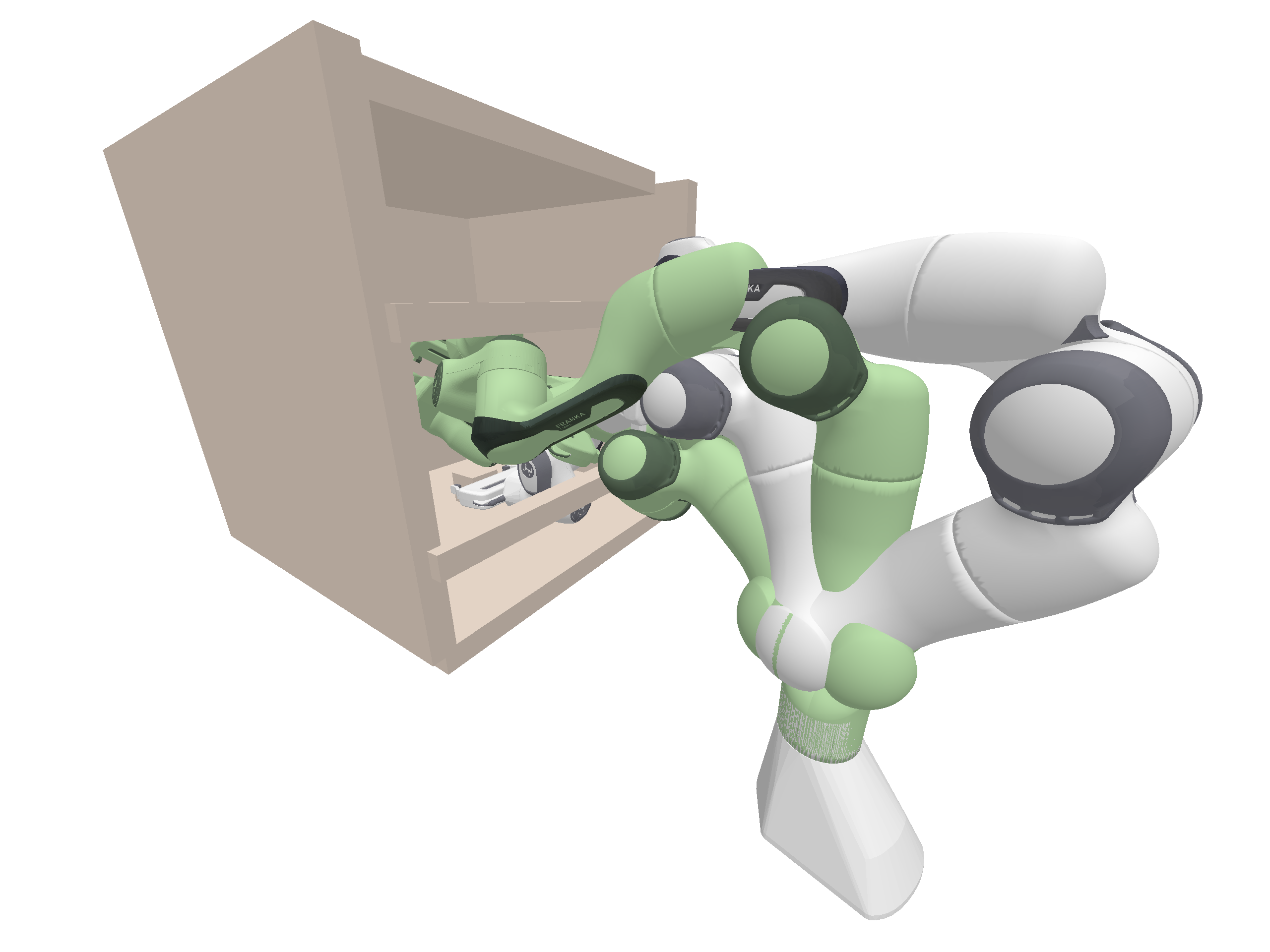}
    %\caption{}
\end{subfigure}
\hfill
\begin{subfigure}[t]{0.22\textwidth}
    \includegraphics[width=\textwidth]{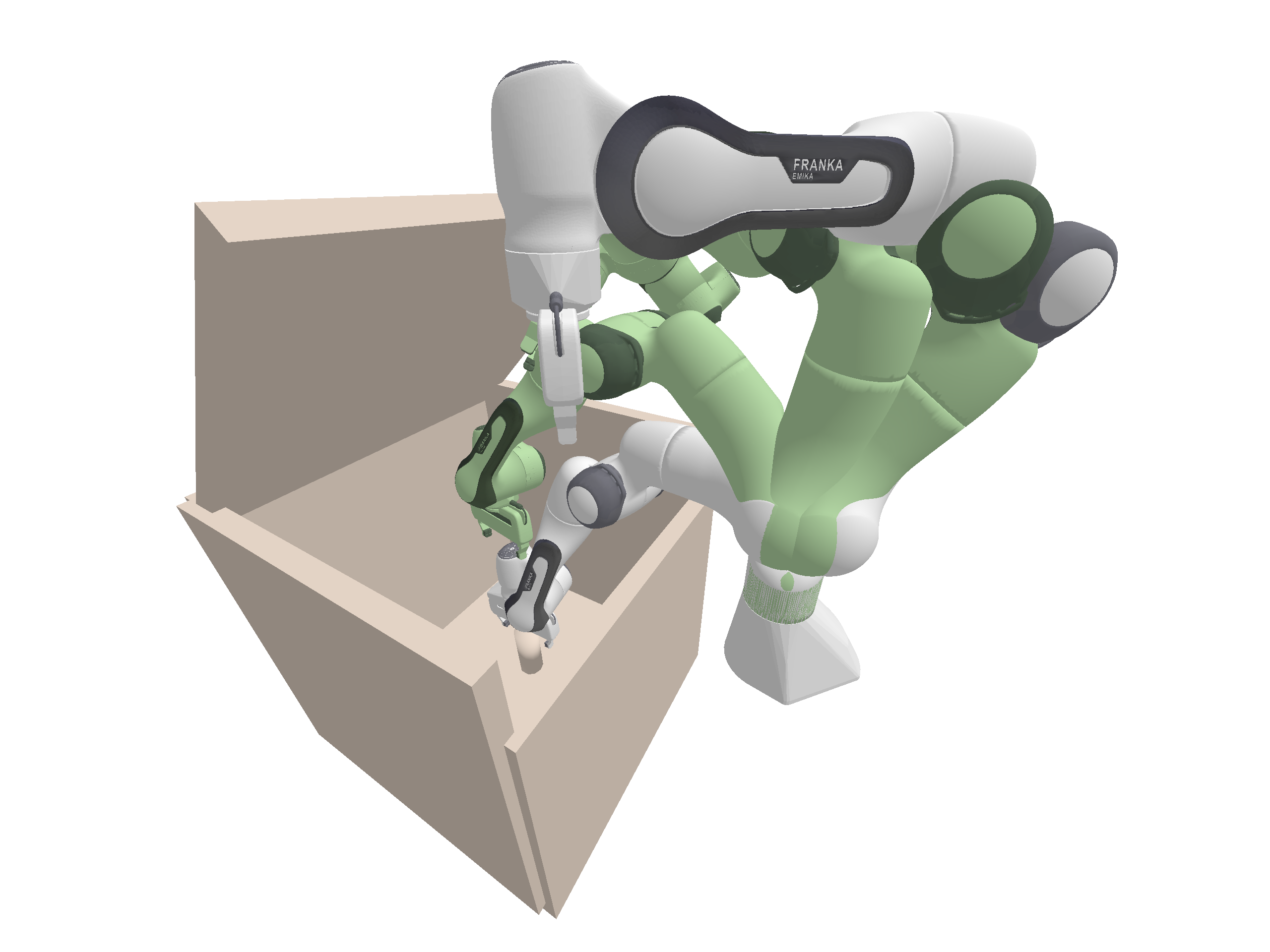}
    %\caption{}
\end{subfigure}
\hfill
\begin{subfigure}[t]{0.22\textwidth}
    \includegraphics[width=\textwidth]{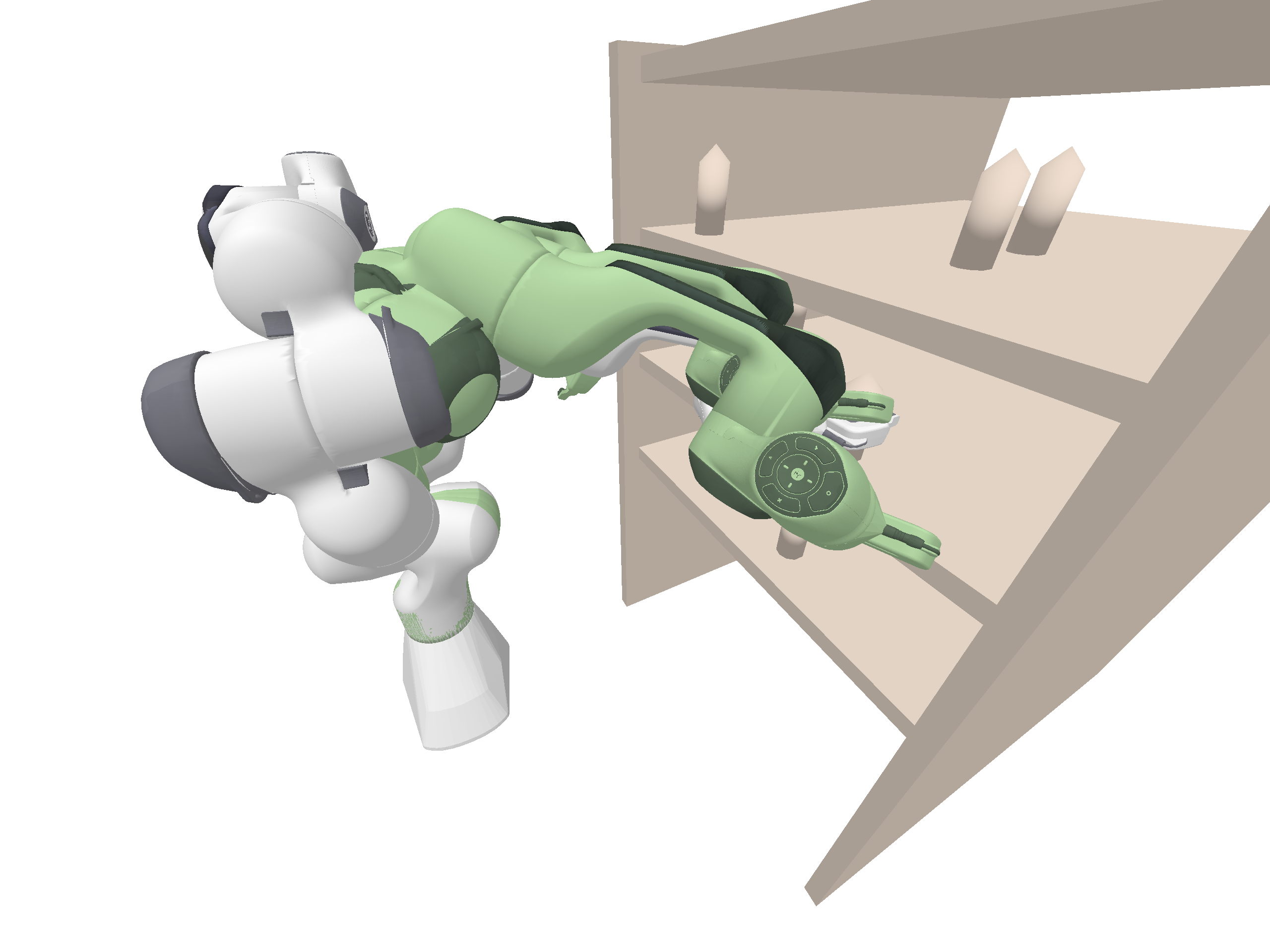}
    %\caption{}
\end{subfigure}
\hfill
\begin{subfigure}[t]{0.22\textwidth}
    \includegraphics[width=\textwidth]{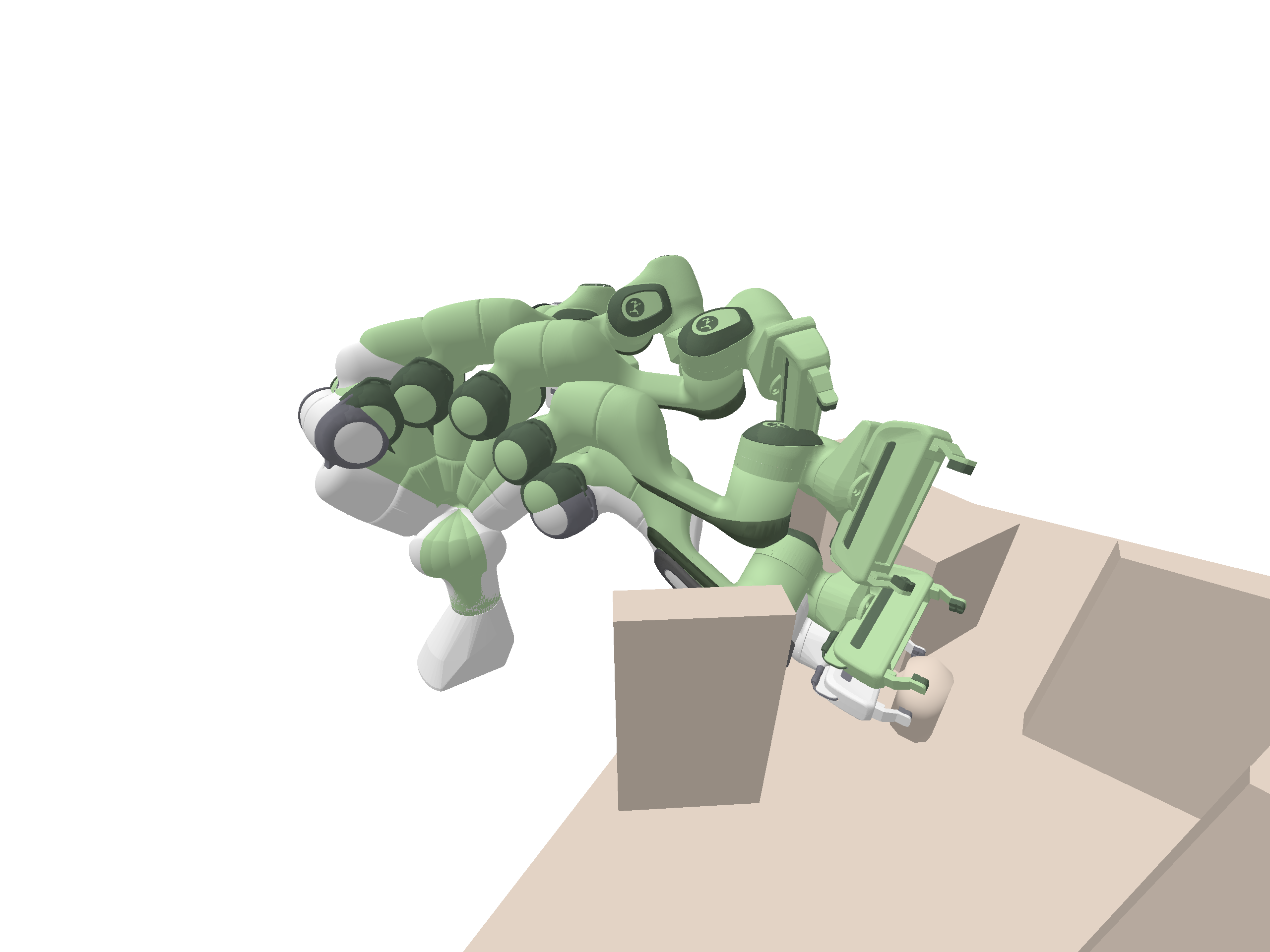}
    %\caption{}
\end{subfigure}
\\
\begin{subfigure}[t]{0.22\textwidth}
    \includegraphics[width=\textwidth]{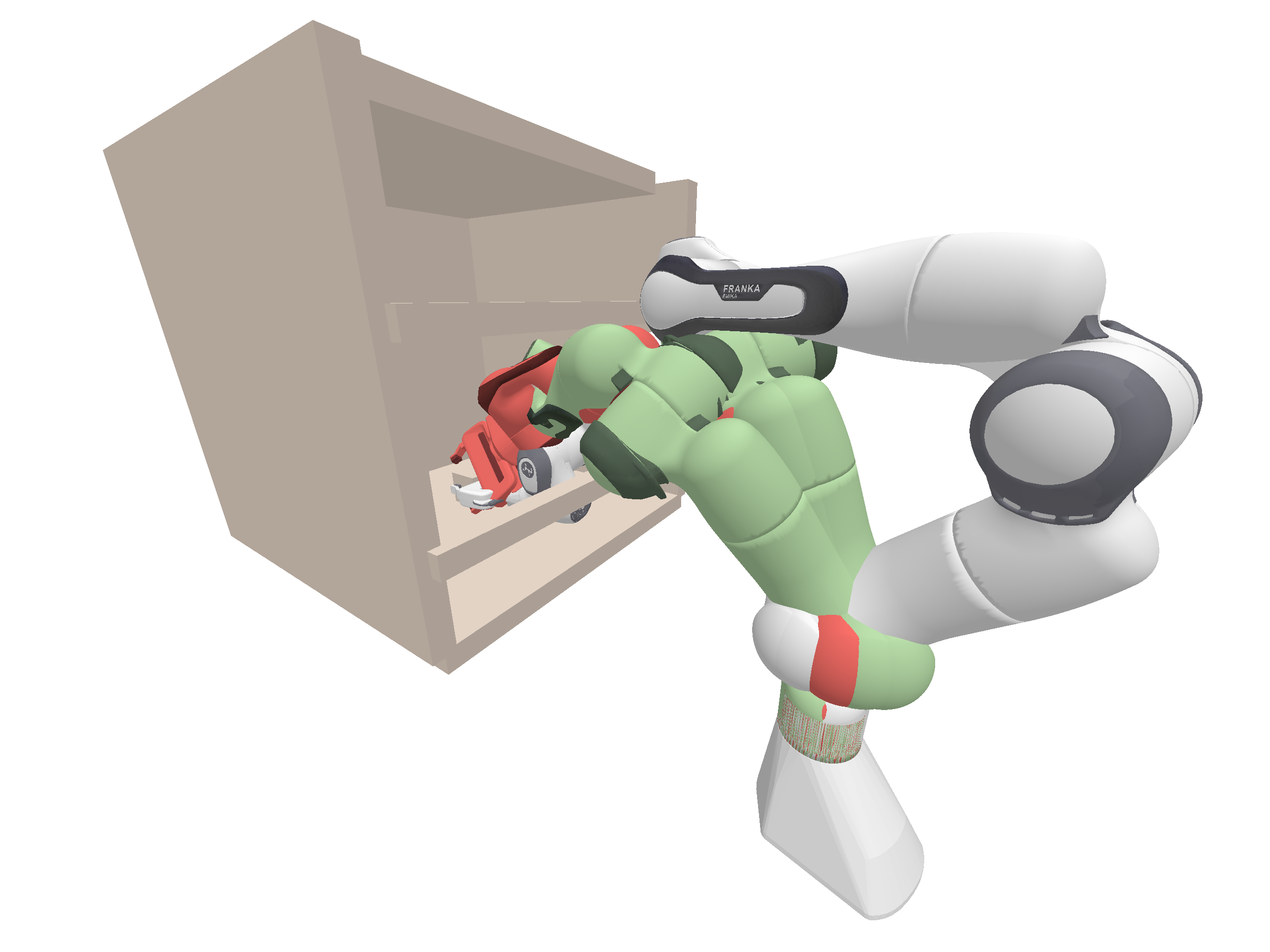}
    \caption{}
\end{subfigure}
\hfill
\begin{subfigure}[t]{0.22\textwidth}
    \includegraphics[width=\textwidth]{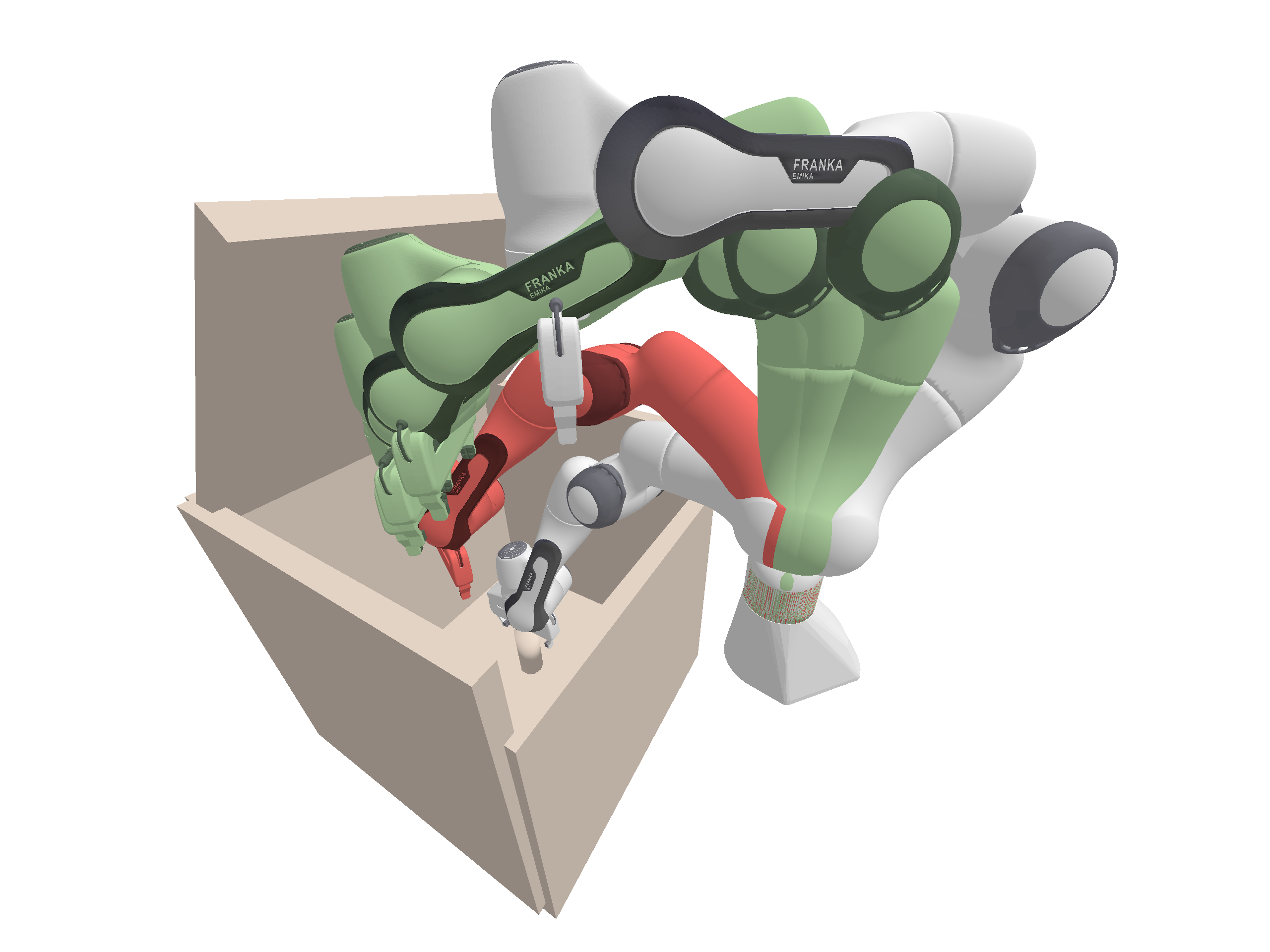}
    \caption{}
\end{subfigure}
\hfill
\begin{subfigure}[t]{0.22\textwidth}
    \includegraphics[width=\textwidth]{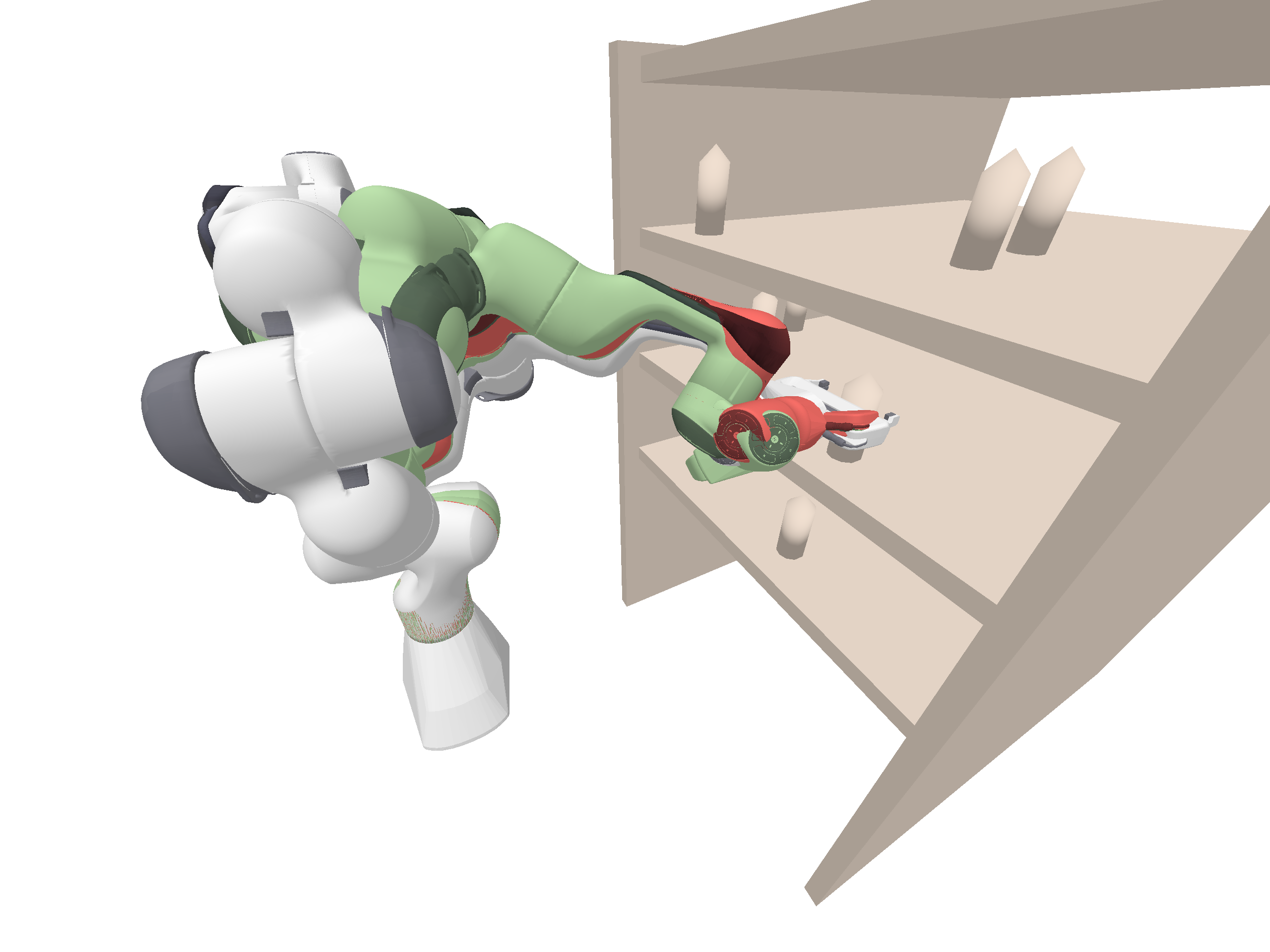}
    \caption{}
\end{subfigure}
\hfill
\begin{subfigure}[t]{0.22\textwidth}
    \includegraphics[width=\textwidth]{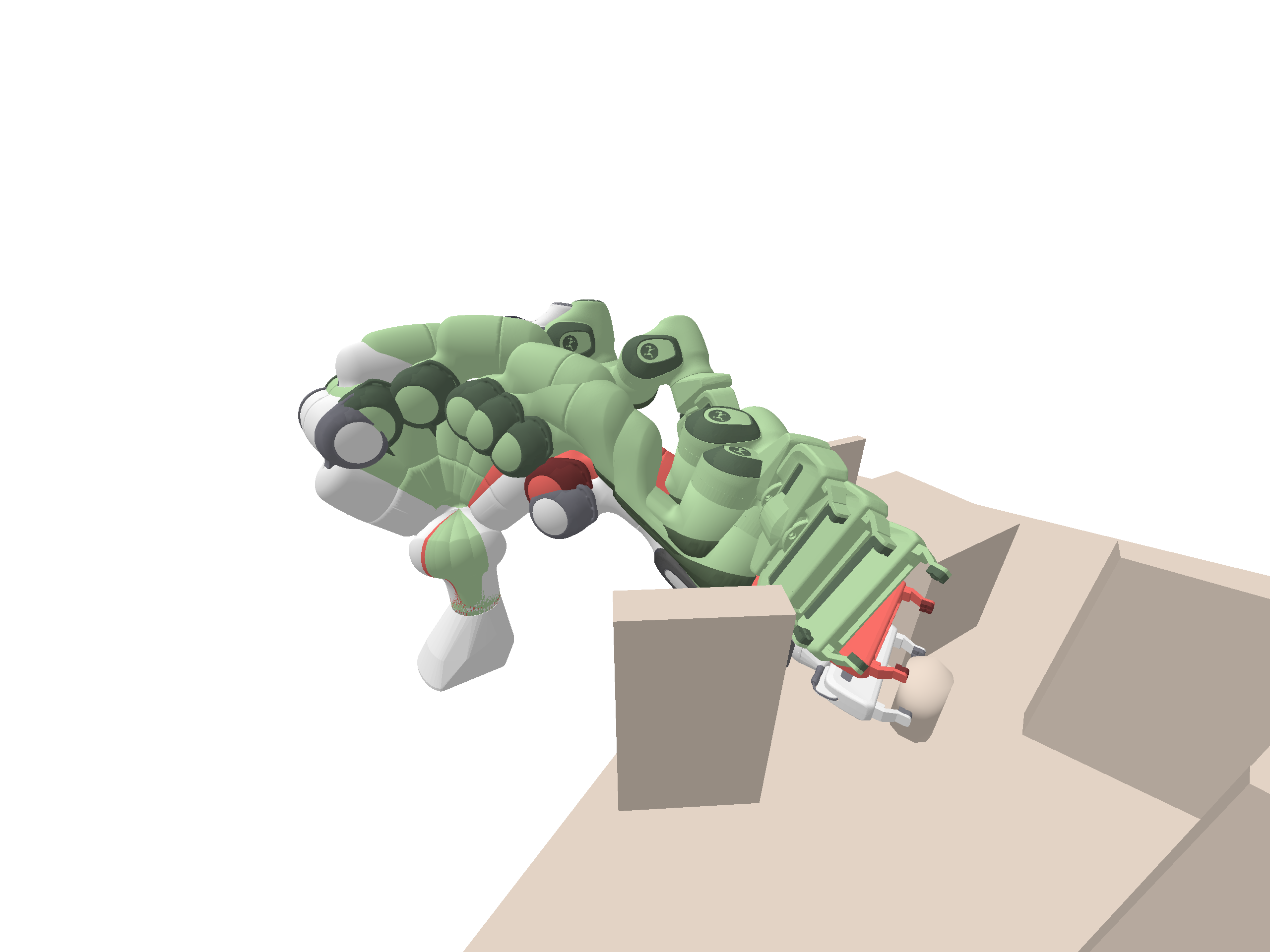}
    \caption{}
\end{subfigure}
\caption{Visualization of trajectories generated by our \mbox{\FLASKRRTC} (top) and by the baseline planner, geometric \mbox{RRTConnect} + TOPP-RA (bottom), for a Franka Emika Panda robot in: (a) \emph{cage}, (b) \emph{box}, (c) \emph{bookshelf thin}, and (d) \emph{table pick} environments.
For each environment, the bottom image shows a trajectory from the baseline with collided configurations in red, while the top image demonstrates that our trajectory is collision-free.}
\label{fig:panda_exp}
\end{figure*}

As each method has its own definition of trajectory costs, \REV{it is challenging to compare the quality of the trajectories. However, most cost functions describe a combination of trajectory duration, length, control efforts, and even velocity, acceleration, or higher-order terms, all of which intuitively and indirectly prioritize shorter trajectories. Therefore,} we choose the trajectory length and planning time as common metrics for comparison, shown in \cref{table:dynobench_results}. \REV{For anytime planners such as \mbox{iDb-A*}~\cite{ortizharo2025iDbAstar},  SST*~\cite{li2016sst}, we report the time of the first solution and the length of the final solution within a $180$-second limit}. Qualitatively, \cref{fig:dynobench_viz} plots the trajectories generated by our \mbox{\FLASKRRTC} (blue solid curves), our \REV{BVP-augmented} \mbox{\FLASKSST} (magenta solid curves), iDb-A* (orange dashed curves), SST* (green dotted curves), \REV{and Kino-PAX (dash-dotted purple curves)}. 
\subsubsection{Comparison to the baselines}
The baselines iDb-A* and SST* generate shorter trajectories but often take several seconds to plan, especially for complicated robot dynamics such as quadrotors. \REV{Kino-PAX~\cite{perrault2025kino} achieves planning times in the millisecond range by parallelizing the tree expansion \emph{on GPUs} but returns much longer trajectories, possibly due to the ``wandering" effect" from dynamics propagation with sampled suboptimal control. While it achieves similar planning times for simple dynamics such as unicycles, it is $\sim$8--300 times slower than our \FLASKRRTC for more complicated systems such as quadrotors.} 
For all $3$ robot platforms, our \mbox{\FLASKRRTC} achieves significantly faster planning times than the baselines, in just a few milliseconds, offering the capability of real-time kinodynamic planning. Our \REV{BVP-augmented} \mbox{\FLASKSST} variant is slightly slower than our \mbox{\FLASKRRTC} since it is restricted to a constant control input for each local ``flat" path instead of directly calculating a time-parameterized optimal control value. However, it is still significantly faster than the baselines in most cases. Meanwhile, our methods maintain comparable trajectory lengths, within 10--20\% of the shortest length across the benchmarking problems. \REV{Intuitively, this is because the cost function \eqref{eq:flat_cost} prioritizes shorter trajectories, which leads to lower total control effort and trajectory time}. While our generated trajectories can be further optimized by combining with other optimization-based planners, \emph{e.g.}, as initial solutions, this is out of the scope of our paper and left for future work.

\begin{table*}[t]
    \caption{Planning performance with a Franka Emika Panda robot over $100$ runs of our \mbox{\FLASKRRTC} with different environments in MotionBenchMaker~\cite{chamzas2022-motion-bench-maker}. For a baseline, we use a collision-free geometric path, provided by a SIMD-parallelized geometric \mbox{RRTConnect}~\cite{thomason2024vamp}, followed by time parameterization (TOPP-RA~\cite{pham2018toppra}). Overall, our kinodynamic \mbox{\FLASKRRTC} generates a dynamically feasible trajectory faster than the baseline in almost all environments. Our trajectory is slightly longer but verified to be collision-free while the baseline's trajectory is colliding in approximately $30\%$ of the cases, due to the time parameterization. Better metrics are shown in bold, while the collision risk is \REVVV{underlined} in red if positive.}
    \label{table:panda_results}
    \centering
% {\rowcolors{5}{gray!10}{gray!25}
\begin{tabular}{ll|rrrrrr|rrrrrr} 
		\multirow{2}{*}{Env.} & \multirow{2}{*}{Metrics} & \multicolumn{6}{c}{Our kinodynamic \mbox{\FLASKRRTC}} & \multicolumn{6}{c}{Geometric \mbox{RRTConnect} (VAMP) + TOPP-RA} \\
        \cmidrule(lr){3-8} \cmidrule(lr){9-14}
        &&Mean&SD & $25\%$&$50\%$&$75\%$ &$95\%$& Mean&SD &$25\%$&$50\%$&$75\%$ &$95\%$\\
		\hline
		\hline
		\multirow{3}{*}{\vtop{\hbox{\strut Bookshelf}\hbox{\strut (thin)}}} & Total planning time (ms) & $\BETTER{0.61}$  & $1.5$     &      $\BETTER{0.10}$     & $\BETTER{0.20}$      &      $\BETTER{0.46}$&     $\BETTER{3.50}$       &       $3.96$    & $1.02$   &      $3.14$     & $3.70$      &      $4.60$&     $5.47$  \\
         & \REV{\hspace{1em} $\triangleright$\mbox{RRTConnect} time}& $\MATHREV{0.58}$    & $\MATHREV{1.25}$   &      $\MATHREV{0.07}$     & $\MATHREV{0.16}$      &      $\MATHREV{0.42}$&     $\MATHREV{3.45}$      &       $\MATHREV{\BETTER{0.16}}$     & $\MATHREV{0.59}$  &      $\MATHREV{\BETTER{0.05}}$     & $\MATHREV{\BETTER{0.07}}$      &      $\MATHREV{\BETTER{0.11}}$&     $\MATHREV{\BETTER{0.27}}$ \\
         & \REV{\hspace{1em} $\triangleright$Post./TOPP-RA time}& $\MATHREV{\BETTER{0.03}}$   & $\MATHREV{0.02}$   & $\MATHREV{\BETTER{0.02}}$  &      $\MATHREV{\BETTER{0.03}}$     & $\MATHREV{\BETTER{0.04}}$      &      $\MATHREV{\BETTER{0.08}}$&     $\MATHREV{3.80}$      &       $\MATHREV{0.84}$      &      $\MATHREV{3.04}$     & $\MATHREV{3.62}$      &      $\MATHREV{4.43}$&     $\MATHREV{5.28}$ \\
         &Final traj. length (rad)& $5.25$   & $1.08$  & $4.53$   &      $5.15$     & $5.90$      &      $7.63$&     $\BETTER{4.66}$      &       $0.79$      &      $\BETTER{4.15}$     & $\BETTER{4.68}$      &      $\BETTER{5.05}$&     $\BETTER{5.92}$ \\
         & Collision risk& $\BETTER{0}$    & $0$ & $\BETTER{0}$   &      $\BETTER{0}$     & $\BETTER{0}$      &      $\BETTER{0}$&     $\RED{0.32}$     &       $0.47$      &      $0$     & $0$      &      $\RED{1}$ &     $\RED{1}$\\
         %%%%%%%%%%%%%%%%%%%%%%%%%%%%%
		\hline
		\multirow{3}{*}{\vtop{\hbox{\strut Bookshelf}\hbox{\strut (tall)}}} & Total planning time (ms) & $\BETTER{1.09}$  & $3.55$     &      $\BETTER{0.07}$     & $\BETTER{0.15}$      &      $\BETTER{0.36}$&     $\BETTER{5.2}$       &       $3.91$    & $0.99$   &      $3.21$     & $3.57$      &      $4.62$&     $5.72$  \\
         & \REV{\hspace{1em} $\triangleright$\mbox{RRTConnect} time}& $\MATHREV{1.06}$    & $\MATHREV{3.55}$   &      $\MATHREV{0.05}$     & $\MATHREV{0.12}$      &      $\MATHREV{0.34}$&     $\MATHREV{5.19}$      &       $\MATHREV{\BETTER{0.14}}$     & $\MATHREV{0.29}$  &      $\MATHREV{\BETTER{0.04}}$     & $\MATHREV{\BETTER{0.06}}$      &      $\MATHREV{\BETTER{0.11}}$&     $\MATHREV{\BETTER{0.47}}$ \\
         & \REV{\hspace{1em} $\triangleright$Post./TOPP-RA time}& $\MATHREV{\BETTER{0.03}}$   & $\MATHREV{0.02}$   & $\MATHREV{\BETTER{0.02}}$  &      $\MATHREV{\BETTER{0.03}}$     & $\MATHREV{\BETTER{0.04}}$      &      $\MATHREV{\BETTER{0.06}}$&     $\MATHREV{3.77}$      &       $\MATHREV{0.96}$      &      $\MATHREV{3.07}$     & $\MATHREV{3.38}$      &      $\MATHREV{4.44}$&     $\MATHREV{5.43}$ \\
         &Final traj. length (rad)& $5.30$   & $1.01$  & $4.56$   &      $5.19$     & $5.96$      &      $7.01$&     $\BETTER{4.86}$      &       $0.69$      &      $\BETTER{4.34}$     & $\BETTER{4.85}$      &      $\BETTER{5.28}$&     $\BETTER{6.08}$ \\
         & Collision risk& $\BETTER{0}$    & $0$ & $\BETTER{0}$   &      $\BETTER{0}$     & $\BETTER{0}$      &      $\BETTER{0}$&     $\RED{0.34}$      &       $0.48$      &      $0$     & $0$      &      $\RED{1}$&     $\RED{1}$ \\
         %%%%%%%%%%%%%%%%%%%%%%%%%%%%%
		\hline
		\multirow{3}{*}{\vtop{\hbox{\strut Bookshelf}\hbox{\strut (small)}}} & Total planning time (ms) & $7.80$  & $35.63$     &      $\BETTER{0.08}$     & $\BETTER{0.17}$      &      $\BETTER{0.64}$&     $38.49$       &       $\BETTER{4.11}$    & $1.04$   &      $3.19$     & $3.85$      &      $4.93$&     $5.74$  \\
         & \REV{\hspace{1em} $\triangleright$\mbox{RRTConnect} time}& $\MATHREV{7.77}$    & $\MATHREV{35.63}$   &      $\MATHREV{0.06}$     & $\MATHREV{0.14}$      &      $\MATHREV{0.57}$ &     $\MATHREV{38.45}$      &       $\MATHREV{\BETTER{0.14}}$     & $\MATHREV{0.27}$  &      $\MATHREV{\BETTER{0.04}}$     & $\MATHREV{\BETTER{0.06}}$      &      $\MATHREV{\BETTER{0.12}}$&     $\MATHREV{\BETTER{0.45}}$\\
         & \REV{\hspace{1em} $\triangleright$Post./TOPP-RA time}& $\MATHREV{\BETTER{0.03}}$  & $\MATHREV{0.15}$  & $\MATHREV{\BETTER{0.02}}$ &      $\MATHREV{\BETTER{0.03}}$    & $\MATHREV{\BETTER{0.04}}$     &      $\MATHREV{\BETTER{0.06}}$&     $\MATHREV{3.97}$     &       $\MATHREV{0.95}$     &      $\MATHREV{3.16}$    & $\MATHREV{3.77}$     &      $\MATHREV{4.66}$&     $\MATHREV{5.30}$\\
         &Final traj. length (rad)& $5.50$   & $1.31$    &      $4.69$     & $5.31$      &      $6.17$ & $8.23$ &     $\BETTER{4.83}$      &       $0.91$      &      $\BETTER{4.34}$     & $\BETTER{4.83}$      &      $\BETTER{5.32}$ &     $\BETTER{6.05}$ \\
         & Collision risk& $\BETTER{0}$    & $0$ & $\BETTER{0}$   &      $\BETTER{0}$     & $\BETTER{0}$      &      $\BETTER{0}$ &       $\RED{0.32}$      &       $0.47$      &      $0$     & $0$      &      $\RED{1}$ &     $\RED{1}$ \\
         %%%%%%%%%%%%%%%%%%%%%%%%%%%%%
		\hline
		\multirow{3}{*}{Cage} & Total planning time (ms) & $12.96$  & $27.54$     &      $\BETTER{1.83}$     & $6.35$      &      $14.05$ &     $41.32$       &       $\BETTER{5.81}$    & $1.80$   &      $4.36$     & $\BETTER{5.32}$      &      $\BETTER{7.21}$ &     $\BETTER{8.92}$  \\
         & \REV{\hspace{1em} $\triangleright$\mbox{RRTConnect} time}& $\MATHREV{12.89}$   & $\MATHREV{27.53}$  &      $\MATHREV{1.78}$    & $\MATHREV{6.28}$     &      $\MATHREV{13.97}$&     $\MATHREV{41.13}$     &       $\MATHREV{\BETTER{0.60}}$    & $\MATHREV{0.45}$ &      $\MATHREV{\BETTER{0.27}}$    & $\MATHREV{\BETTER{0.50}}$     &      $\MATHREV{\BETTER{0.88}}$&     $\MATHREV{\BETTER{1.66}}$\\
         & \REV{\hspace{1em} $\triangleright$Post./TOPP-RA time}& $\MATHREV{\BETTER{0.07}}$  & $\MATHREV{\BETTER{0.04}}$  & $\MATHREV{\BETTER{0.05}}$ &      $\MATHREV{\BETTER{0.06}}$    & $\MATHREV{\BETTER{0.09}}$     &      $\MATHREV{\BETTER{0.14}}$&     $\MATHREV{5.21}$     &       $\MATHREV{1.62}$     &      $\MATHREV{3.76}$    & $\MATHREV{4.89}$     &      $\MATHREV{6.53}$&     $\MATHREV{7.85}$\\
         &Final traj. length (rad)& $10.07$   & $3.61$  & $7.17$   &      $9.27$     & $12.30$      &      $16.60$ &     $\BETTER{7.09}$      &       $1.67$      &      $\BETTER{5.55}$     & $\BETTER{6.54}$      &      $\BETTER{8.52}$ &     $\BETTER{9.63}$ \\
         & Collision risk& $\BETTER{0}$    & $0$ & $\BETTER{0}$   &      $\BETTER{0}$     & $\BETTER{0}$      &      $\BETTER{0}$ &      $\RED{0.29}$      &       $0.46$      &      $0$     & $0$      &      $\RED{1}$ &     $\RED{1}$ \\
         %%%%%%%%%%%%%%%%%%%%%%%%%%%%%
		\hline
		\multirow{3}{*}{Box} & Total planning time (ms) & $\BETTER{0.86}$  & $4.25$     &      $\BETTER{0.08}$     & $\BETTER{0.25}$      &      $\BETTER{0.59}$ &     $\BETTER{1.40}$       &       $4.22$    & $0.93$   &      $3.35$     & $4.25$      &      $4.97$ &     $5.64$  \\
         & \REV{\hspace{1em} $\triangleright$\mbox{RRTConnect} time}& $\MATHREV{0.81}$   & $\MATHREV{4.25}$  &      $\MATHREV{0.05}$    & $\MATHREV{0.20}$     &      $\MATHREV{0.55}$&     $\MATHREV{1.35}$     &       $\MATHREV{\BETTER{0.14}}$    & $\MATHREV{0.14}$ &      $\MATHREV{\BETTER{0.07}}$    & $\MATHREV{\BETTER{0.12}}$     &      $\MATHREV{\BETTER{0.17}}$&     $\MATHREV{\BETTER{0.26}}$\\
         & \REV{\hspace{1em} $\triangleright$Post./TOPP-RA time}& $\MATHREV{\BETTER{0.04}}$  & $\MATHREV{0.02}$  & $\MATHREV{\BETTER{0.03}}$ &      $\MATHREV{\BETTER{0.04}}$    & $\MATHREV{\BETTER{0.05}}$     &      $\MATHREV{\BETTER{0.07}}$&     $\MATHREV{4.07}$     &       $\MATHREV{0.91}$     &      $\MATHREV{3.24}$    & $\MATHREV{4.07}$     &      $\MATHREV{4.84}$&     $\MATHREV{5.38}$\\
         &Final traj. length (rad)& $5.64$   & $1.20$  & $4.79$   &      $5.26$     & $6.14$      &      $8.10$ &     $\BETTER{4.68}$      &       $0.82$      &      $\BETTER{4.07}$     & $\BETTER{4.45}$      &      $\BETTER{5.02}$ &     $\BETTER{6.35}$ \\
         & Collision risk& $\BETTER{0}$    & $0$ & $\BETTER{0}$   &      $\BETTER{0}$     & $\BETTER{0}$      &      $\BETTER{0}$ &      $\RED{0.13}$      &       $0.34$      &      $0$     & $0$      &      $0$ &     $\RED{1}$ \\
         %%%%%%%%%%%%%%%%%%%%%%%%%%%%%
		\hline
		\multirow{3}{*}{\vtop{\hbox{\strut Table}\hbox{\strut under}\hbox{\strut pick}}} & Total planning time (ms) & $\BETTER{0.34}$  & $0.45$     &      $\BETTER{0.17}$     & $\BETTER{0.25}$      &      $\BETTER{0.35}$ &     $\BETTER{0.63}$       &       $5.13$    & $0.99$   &      $4.41$     & $5.10$      &      $5.89$ &     $6.45$  \\
         & \REV{\hspace{1em} $\triangleright$\mbox{RRTConnect} time}& $\MATHREV{0.28}$   & $\MATHREV{0.45}$  &      $\MATHREV{0.11}$    & $\MATHREV{0.19}$     &      $\MATHREV{0.29}$&     $\MATHREV{0.52}$     &       $\MATHREV{\BETTER{0.09}}$    & $\MATHREV{0.09}$ &      $\MATHREV{\BETTER{0.06}}$    & $\MATHREV{\BETTER{0.07}}$     &      $\MATHREV{\BETTER{0.10}}$&     $\MATHREV{\BETTER{0.15}}$\\
         & \REV{\hspace{1em} $\triangleright$Post./TOPP-RA time}& $\MATHREV{\BETTER{0.06}}$  & $\MATHREV{0.03}$  & $\MATHREV{\BETTER{0.04}}$ &      $\MATHREV{\BETTER{0.05}}$    & $\MATHREV{\BETTER{0.07}}$     &      $\MATHREV{\BETTER{0.11}}$&     $\MATHREV{5.04}$     &       $\MATHREV{0.98}$     &      $\MATHREV{4.34}$    & $\MATHREV{5.03}$     &      $\MATHREV{5.71}$&     $\MATHREV{6.36}$\\
         &Final traj. length (rad)& $8.46$   & $1.89$  & $7.07$   &      $8.48$     & $9.55$      &      $11.82$ &     $\BETTER{6.80}$      &       $1.39$      &      $\BETTER{5.61}$      &      $\BETTER{6.75}$ &   $\BETTER{7.71}$     &   $\BETTER{9.01}$ \\
         & Collision risk& $\BETTER{0}$    & $0$ & $\BETTER{0}$   &      $\BETTER{0} $    & $\BETTER{0}$      &      $\BETTER{0}$ &      $\RED{0.40}$      &       $0.49$      &      $0$     & $0$      &      $\RED{1}$ &     $\RED{1}$ \\
         %%%%%%%%%%%%%%%%%%%%%%%%%%%%%
		\hline
		\multirow{3}{*}{\vtop{\hbox{\strut Table}\hbox{\strut pick}}} & Total planning time (ms) & $\BETTER{0.28}$    & $0.91$   &      $\BETTER{0.08}$     & $\BETTER{0.12}$      &      $\BETTER{0.18}$ &     $\BETTER{0.57}$ & $3.78$    & $0.76$   &      $3.08$     & $3.60$      &      $4.50$ &     $5.00$  \\
         & \REV{\hspace{1em} $\triangleright$\mbox{RRTConnect} time}& $\MATHREV{0.26}$   & $\MATHREV{0.91}$  &      $\MATHREV{0.05}$    & $\MATHREV{0.09}$     &      $\MATHREV{0.15}$&     $\MATHREV{0.55}$     &       $\MATHREV{\BETTER{0.05}}$    & $\MATHREV{0.03}$ &      $\MATHREV{\BETTER{0.04}}$    & $\MATHREV{\BETTER{0.05}}$     &      $\MATHREV{\BETTER{0.06}}$&     $\MATHREV{\BETTER{0.10}}$\\
         & \REV{\hspace{1em} $\triangleright$Post./TOPP-RA time}& $\MATHREV{\BETTER{0.03}}$  & $\MATHREV{0.01}$  & $\MATHREV{\BETTER{0.02}}$ &      $\MATHREV{\BETTER{0.03}}$    & $\MATHREV{\BETTER{0.03}}$     &      $\MATHREV{\BETTER{0.05}}$&     $\MATHREV{3.73}$     &       $\MATHREV{0.75}$     &      $\MATHREV{3.04}$    & $\MATHREV{3.53}$     &      $\MATHREV{4.44}$&     $\MATHREV{4.95}$\\
         &Final traj. length (rad)& $5.24$   & $0.91$  & $4.57$   &      $5.10$     & $5.69$      &      $6.51$ &     $\BETTER{4.76}$      &       $0.71$      &      $\BETTER{4.26}$      &      $\BETTER{4.53}$ &   $\BETTER{5.18}$     &   $\BETTER{6.27}$ \\
         & Collision risk& $\BETTER{0}$    & $0$ & $\BETTER{0}$   &      $\BETTER{0}$     & $\BETTER{0}$      &      $\BETTER{0}$ &      $\RED{0.42}$      &       $0.49$      &      $0$     & $0$      &      $\RED{1}$ &     $\RED{1}$ \\
         %%%%%%%%%%%%%%%%%%%%%%%%%%%%%
		\hline
        \hline
		\multirow{5}{*}{Overall} & Total planning time (ms) & $\BETTER{3.49}$      &      $18.00$     & $\BETTER{0.11}$      &      $\BETTER{0.24}$ &     $\BETTER{0.76}$      &       $14.65$      &      $4.43$     & $1.34$      &      $3.36$ &     $4.19$ &      $5.14$ &     $\BETTER{7.16}$ \\
         & \hspace{1em} $\triangleright$ \mbox{RRTConnect} time& $3.45$      &      $18.00$     & $0.08$      &      $0.2$ &     $0.70$      &       $14.54$      &      $\BETTER{0.18}$     & $0.36$      &      $\BETTER{0.05}$ &     $\BETTER{0.07}$  &      $\BETTER{0.15}$ &     $\BETTER{0.83}$ \\
         & \hspace{1em} $\triangleright$ Post./TOPP-RA time& $\BETTER{0.04}$      &      $0.03$     & $\BETTER{0.02}$      &      $\BETTER{0.04}$ &     $\BETTER{0.05}$      &       $\BETTER{0.09}$      &      $4.24$     & $1.22$      &      $3.28$ &     $4.02$  &      $4.99$ &     $6.58$ \\
         &Final traj. length (rad)& $6.50$      &      $2.56$     & $4.89$      &      $5.67$ &     $7.37$      &       $11.73$      &      $\BETTER{5.26}$     & $1.40$      &      $\BETTER{4.30}$ &     $\BETTER{4.88}$  &      $\BETTER{5.74}$ &     $\BETTER{8.43}$ \\
         & Collision risk& $\BETTER{0}$    & $0$ & $\BETTER{0}$   &      $\BETTER{0}$     & $\BETTER{0}$      &      $\BETTER{0}$ &      $\RED{0.33}$      &       $0.47$      &      $0$     & $0$      &      $\RED{1}$ &     $\RED{1}$ \\
         \end{tabular}
\end{table*}

\subsubsection{\REV{The roles of BVP solutions and parallelized collision checking}}
\REV{All of our planners in the experiments use parallelized collision checking but integrate the BVP solution at different levels: \emph{(i)} our {\FLASKRRTC} uses the BVP solution to find all the local paths; \emph{(ii)} our {\textit{BVP-augmented} \mbox{\FLASKSST}} only uses the BVP solution to find a shortcut to the goal; \emph{(iii)} our {\textit{SIMD-only }\mbox{\FLASKSST}} only leverages the vectorized collision checking without using the BVP solution; and \emph{(iv)} our {\textit{DP-based} \FLASKRRT} simply expands the planning tree using dynamics propagation without best-state selection and pruning benefits from SST*.

Parallelized collision checking clearly improves the planning times. Compared to \mbox{iDb-A*}~\cite{ortizharo2025iDbAstar} and the original SST*~\cite{li2016sst}, checking the local path for collision in parallel reduces the planning times from tens of seconds to a few seconds (see \cref{table:dynobench_results}) for robots with complicated dynamics such as quadrotors. This is the case even when the BVP solution is not used as in our {\textit{SIMD-only} \mbox{\FLASKSST}} or there is no best-state selection and tree pruning as in our {\textit{DP-based} \FLASKRRT}. As propagation-based planners, our {\textit{SIMD-only} \mbox{\FLASKSST}} and {\textit{DP-based} \FLASKRRT} still wander through the state space, a few small time steps at a time, causing longer planning times and trajectory lengths, as shown in \cref{table:dynobench_results}. This effect is even worse in {\textit{DP-based} \FLASKRRT} as the best-state selection and pruning techniques used in SST* are not available.

However, our unique integration of BVP solutions and parallelized collision checking is the key to driving the planning times further down to the (sub-)millisecond range. By employing the BVP solution to find a shortcut from a node to the goal, our {\textit{BVP-augmented} \mbox{\FLASKSST}} generates a trajectory in just milliseconds, significantly faster than {\textit{SIMD-only} \mbox{\FLASKSST}}. This illustrates the potential of our \REVV{framework} to improve existing propagation-based planners by addressing the common ``wandering" effects via our BVP solution. 
Taking a further step, our {\FLASKRRTC} fully integrates the BVP solution to find all local paths during tree expansion, and therefore, leads to less than a millisecond of planning time in many cases. This illustrates our approach's ability to transform any existing sampling-based geometric planner into an ultrafast kinodynamic planner, requiring only widely available CPUs.
}

\subsection{Kinodynamic planning for high-\dof manipulators}
\label{subsec:exp_panda}
In this section, we evaluate our algorithms with a simulated $7$-\dof Franka Emika Panda robot in PyBullet~\cite{coumans2021pybullet}.
We consider $7$ realistic and challenging environments, including \emph{bookshelf thin}, \emph{bookshelf tall}, \emph{bookshelf small}, \emph{cage}, \emph{box}, \emph{table under pick} and \emph{table pick}, from the MotionBenchMaker benchmark~\cite{chamzas2022-motion-bench-maker}. For each environment, $100$ motion planning problems with $100$ different pairs of start and goal configurations are generated for our experiments. The obstacles and the robot's geometries are represented by sets of spheres, which are compatible to \scSIMD parallelized collision checking as shown in~\cite{thomason2024vamp}. For each problem, the robot's task is to plan a dynamically feasible trajectory from a start to a goal configuration in the free space.

In our experiments, we use our \mbox{\FLASKRRTC} planner, described in \cref{subsec:compare_lowdim} with trajectory simplification in \cref{alg:traj_simplification}. Our baseline is a common approach that generates a time-parameterized trajectory, \emph{e.g.}, using TOPP-RA~\cite{pham2018toppra}, from a collision-free geometric path, under kinematic constraints such as velocity and acceleration limits. The geometric path is provided by a \scSIMD-parallelized geometric \mbox{RRTConnect} planner from VAMP~\cite{thomason2024vamp}. For each environment, we use both our planner and the baseline to solve the pre-generated benchmarking problems and report the results in \cref{table:panda_results}. Our planner's total planning time is calculated as the sum of the \mbox{RRTConnect} time and trajectory simplification time. Meanwhile, the baseline's total planning time includes the \mbox{RRTConnect} time, path simplification time, and TOPP-RA time. As our trajectory and the baseline's are optimized for different cost functions, we use their length as a common metrics, similar to \cref{subsec:compare_lowdim}. As TOPP-RA~\cite{pham2018toppra} only fits a time-parameterized trajectory to a geometric path, it does not consider any collision avoidance constraints. While the geometric path is collision-free, there is no guarantee that the time-parameterized trajectory from TOPP-RA will be valid. Therefore, we sample the resulting trajectories, check for collision and compare the collision risk, measured by the percentage of collided trajectories.

\cref{table:panda_results} compares our kinodynamic planner and the baseline, in terms of the total planning time, the final trajectory's length, and the collision risk, for each environment and overall. On average, we are able to generate a valid trajectory in around $3.5ms$, and in less than $1ms$ for $75\%$ of the problems, while the baseline's total planning time is consistently around $4.5ms$ for almost all problems. While VAMP can generate a collision-free geometric path fast, the baseline's total planning time is dominated by the \mbox{TOPP-RA} time. Our approach, instead, spends most of the time in the \mbox{\FLASKRRTC} planner, as we directly enforce the dynamics constraints in the RRT tree construction. Across all the problems, the baseline offers a shorter trajectory but, more importantly, poses a collision risk of $\sim 30\%$ on average, highlighting the benefit of our approach with guarantees of collision-free trajectories up to the sampling resolution. \cref{fig:panda_exp} visualizes the trajectories from our \FLASKRRTC (top) with no collisions and from the baseline (bottom) with collided configurations (red). 

For the \emph{table under pick} and \emph{table pick} environments, our total planning time is approximately $\sim 15$ times faster than that of the baseline, while maintaining $\sim 4$ times shorter time for the \emph{box}, \emph{bookshelf thin} and \emph{bookshelf tall}  environments. For the \emph{bookshelf small} and \emph{cage} environments, our approach is slightly slower on average. Yet, it is still $\sim 8$ times faster for $75\%$ of the \emph{bookshelf small} problems, suggesting that there are certain challenging cases in those environments that our approach takes slightly longer than the baseline to plan, but still achieves milliseconds of planning time. We emphasize that even though the baseline can be faster in those cases, it still leads to a significant risk of collision, around $30\%$ of the \emph{bookshelf small} and \emph{cage} problems, as shown in \cref{table:panda_results}.

\begin{figure*}[t]
\centering
\begin{subfigure}[t]{0.30\textwidth}
        \centering
\includegraphics[width=\textwidth]{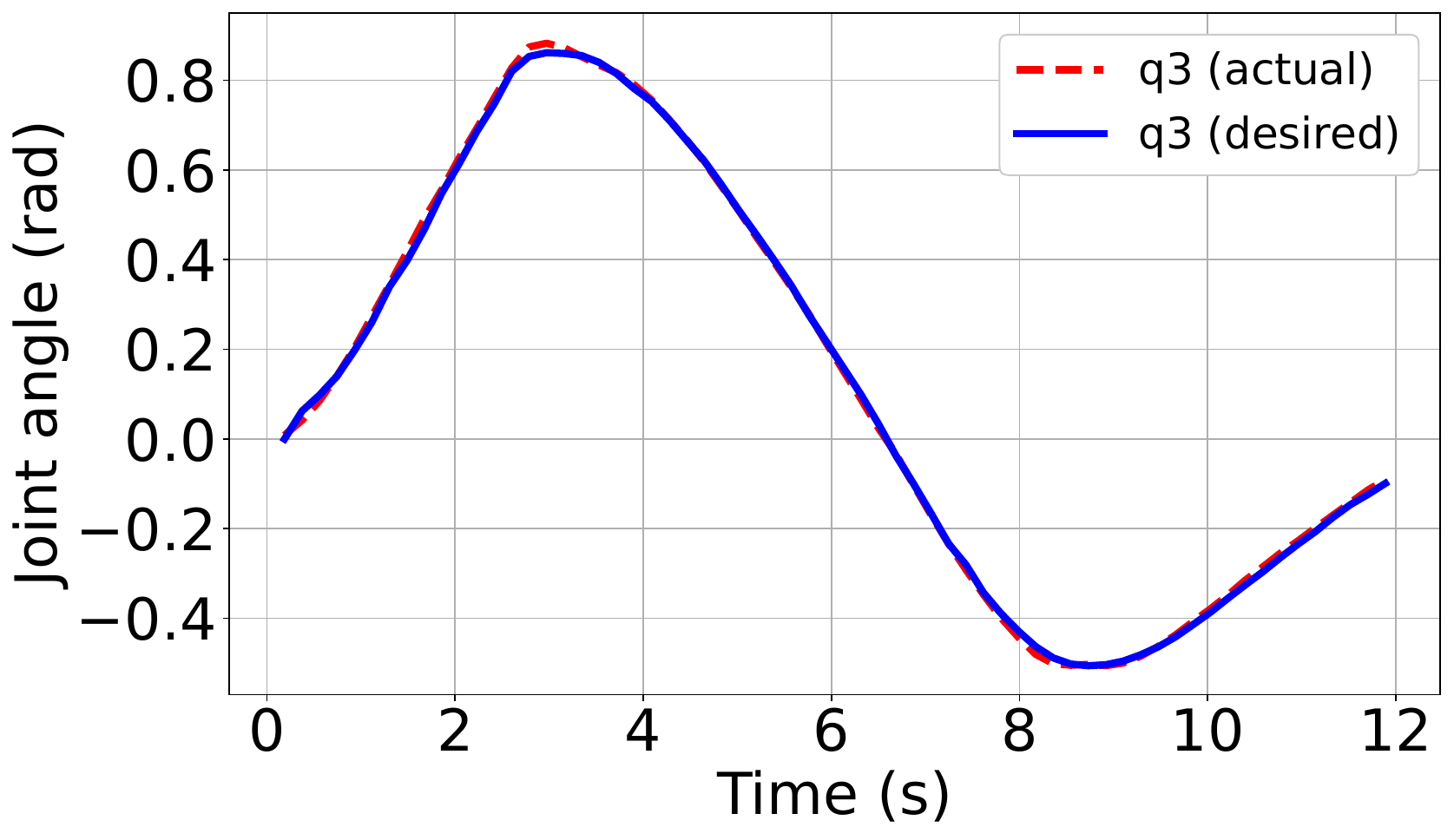}%
\caption{Tracking our trajectory}
\label{fig:ur5_tracking_kino}
\end{subfigure}%
\hfill
\begin{subfigure}[t]{0.30\textwidth}
        \centering
\includegraphics[width=\textwidth]{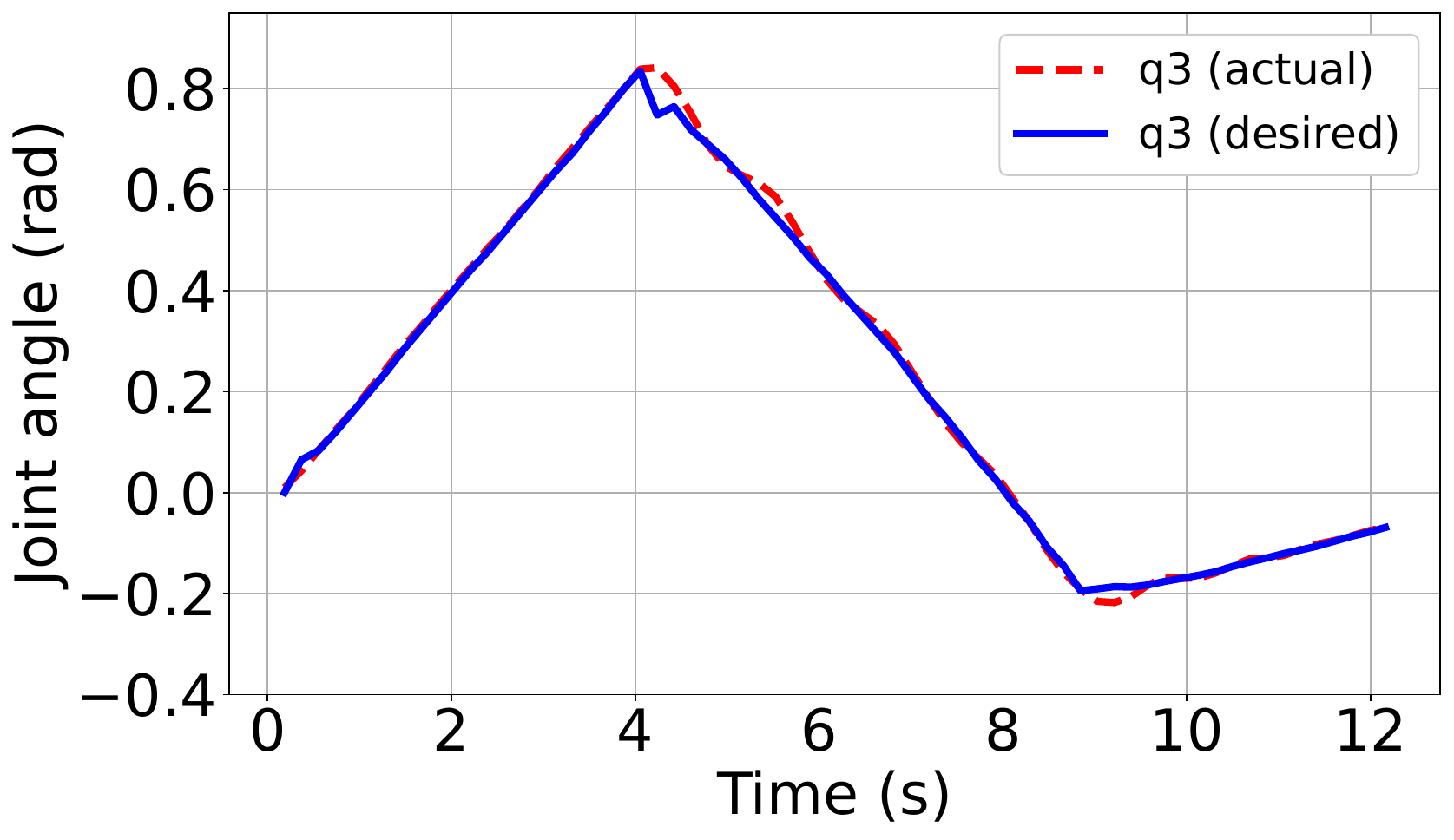}%
\caption{Tracking a geometric path}
\label{fig:ur5_tracking_geo}
\end{subfigure}%
\hfill
\begin{subfigure}[t]{0.30\textwidth}
        \centering
\includegraphics[width=\textwidth]{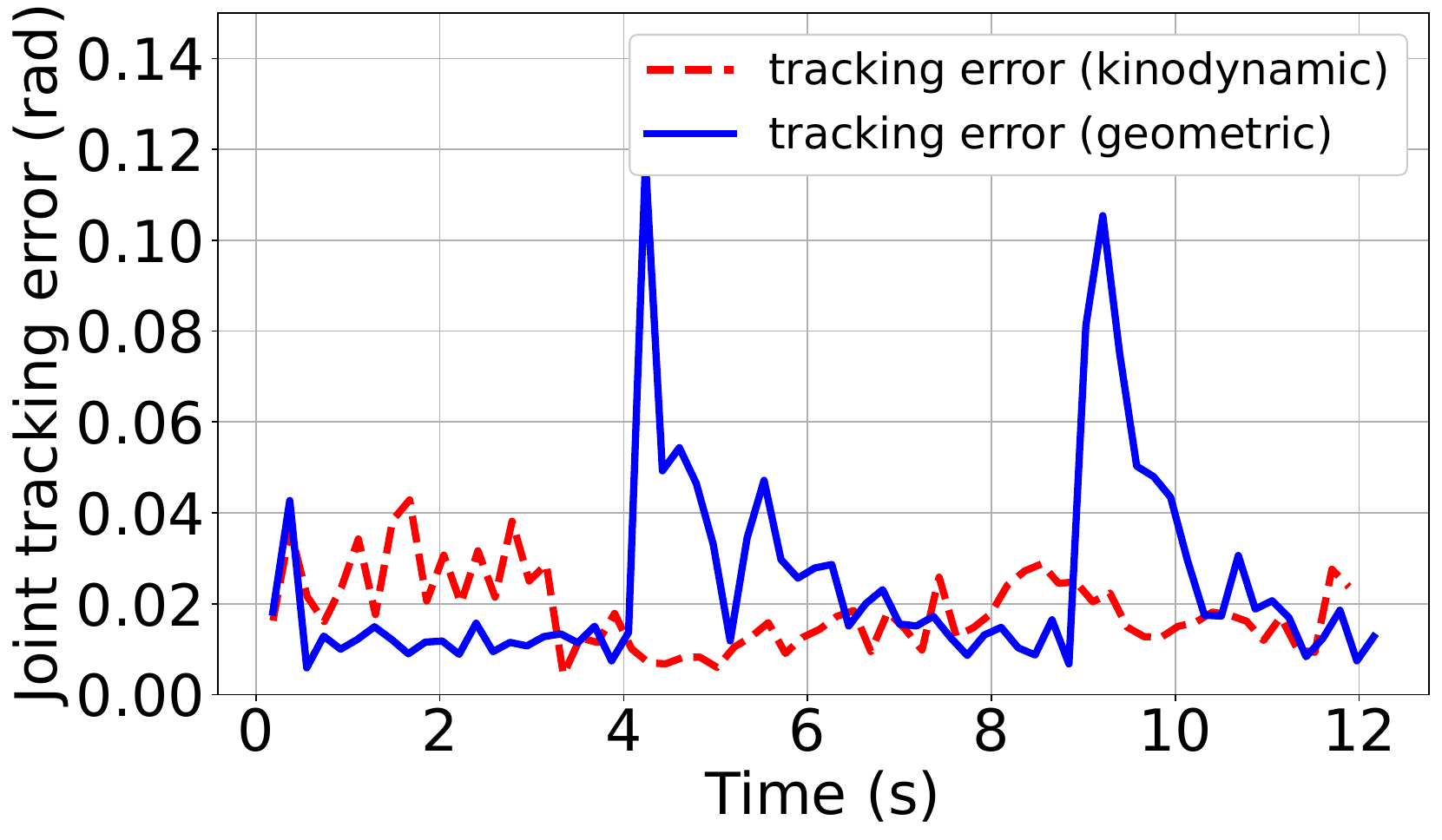}%
\caption{Tracking error}
\label{fig:ur5_tracking_error}
\end{subfigure}%
\caption{Our trajectory is smoother and dynamically feasible, leading to better tracking performance. For example, the third joint angle $q_3$ in the our UR5's configuration stays close to the desired value from our trajectory (a)  while there are overshoots and fluctuates with a geometric path~\cite{thomason2024vamp} (b), causing collision in \cref{fig:ur5_pickplace_geo}. The tracking error $\Vert q_{actual} - q_{desired} \Vert$ also shows two overshooting peaks with the geometric path in \cref{fig:ur5_tracking_error}.}
\label{fig:ur5_tracking_perf}
\end{figure*}

\begin{figure*}[t]
\centering
\begin{subfigure}[t]{0.5\textwidth}
        \centering
\includegraphics[width=0.49\textwidth]{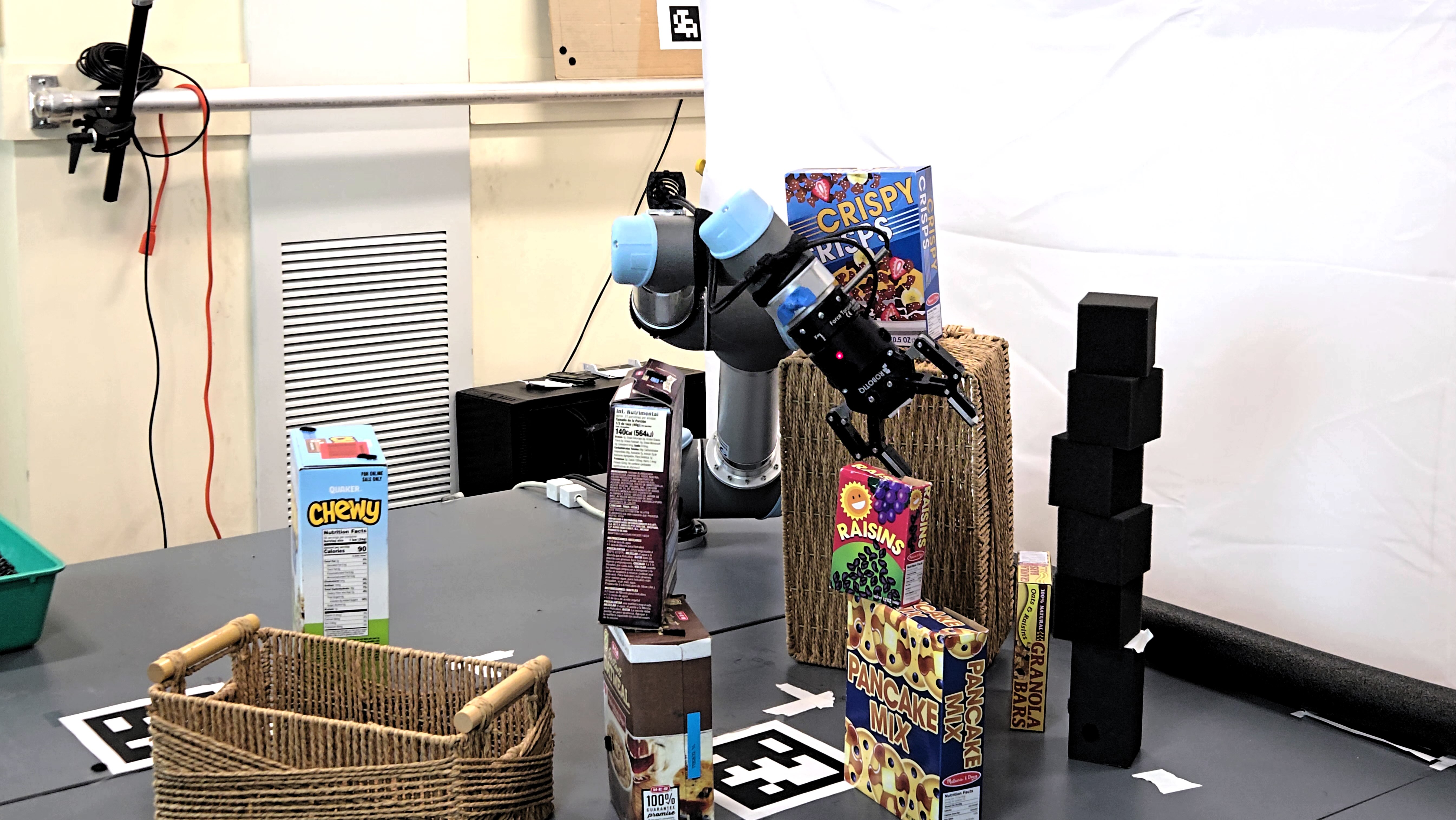}%
\includegraphics[width=0.49\textwidth]{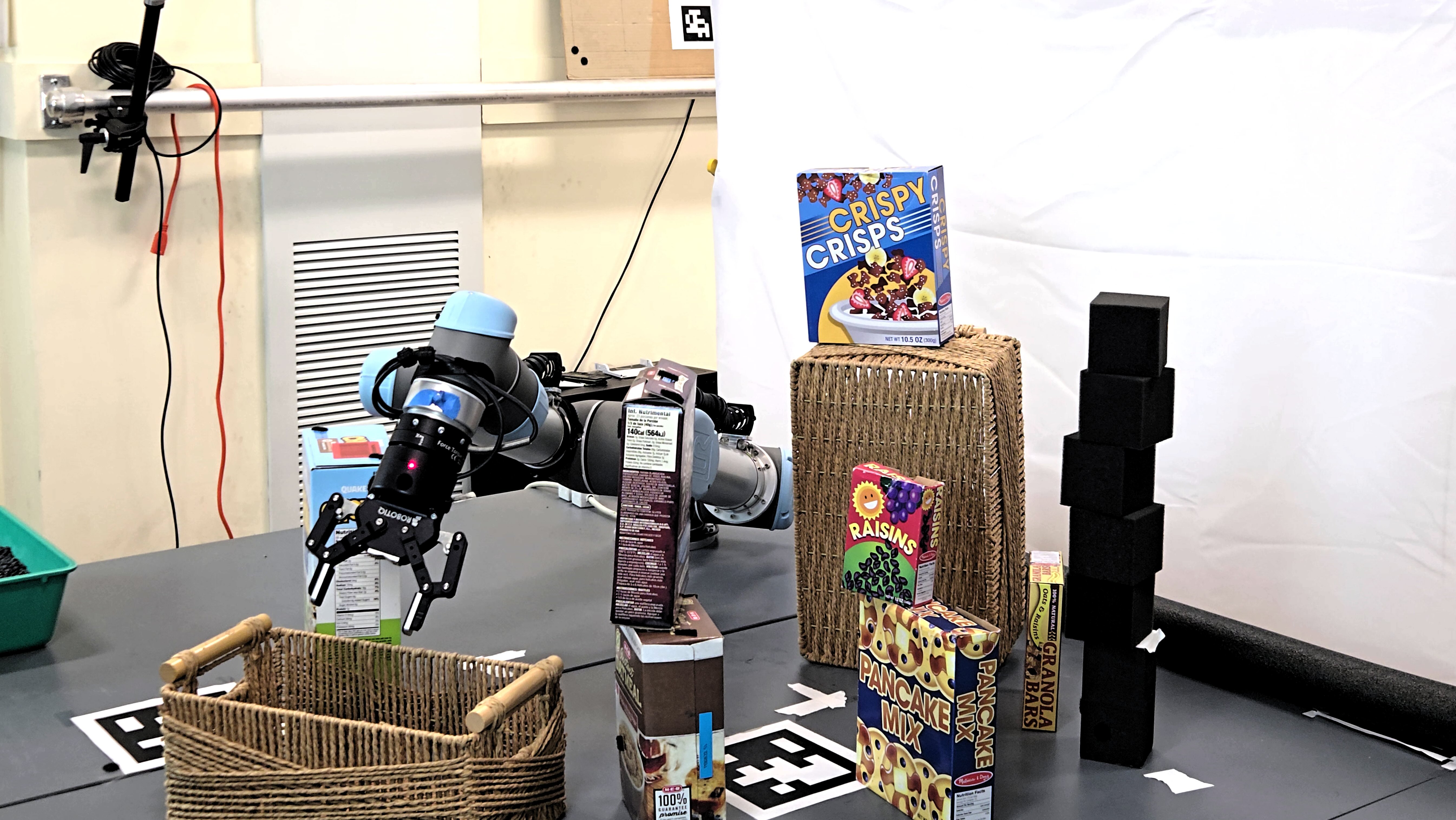}%
\caption{}
\label{fig:ur5_pickplace_kino}
\end{subfigure}%
\begin{subfigure}[t]{0.5\textwidth}
        \centering
\includegraphics[width=0.49\textwidth]{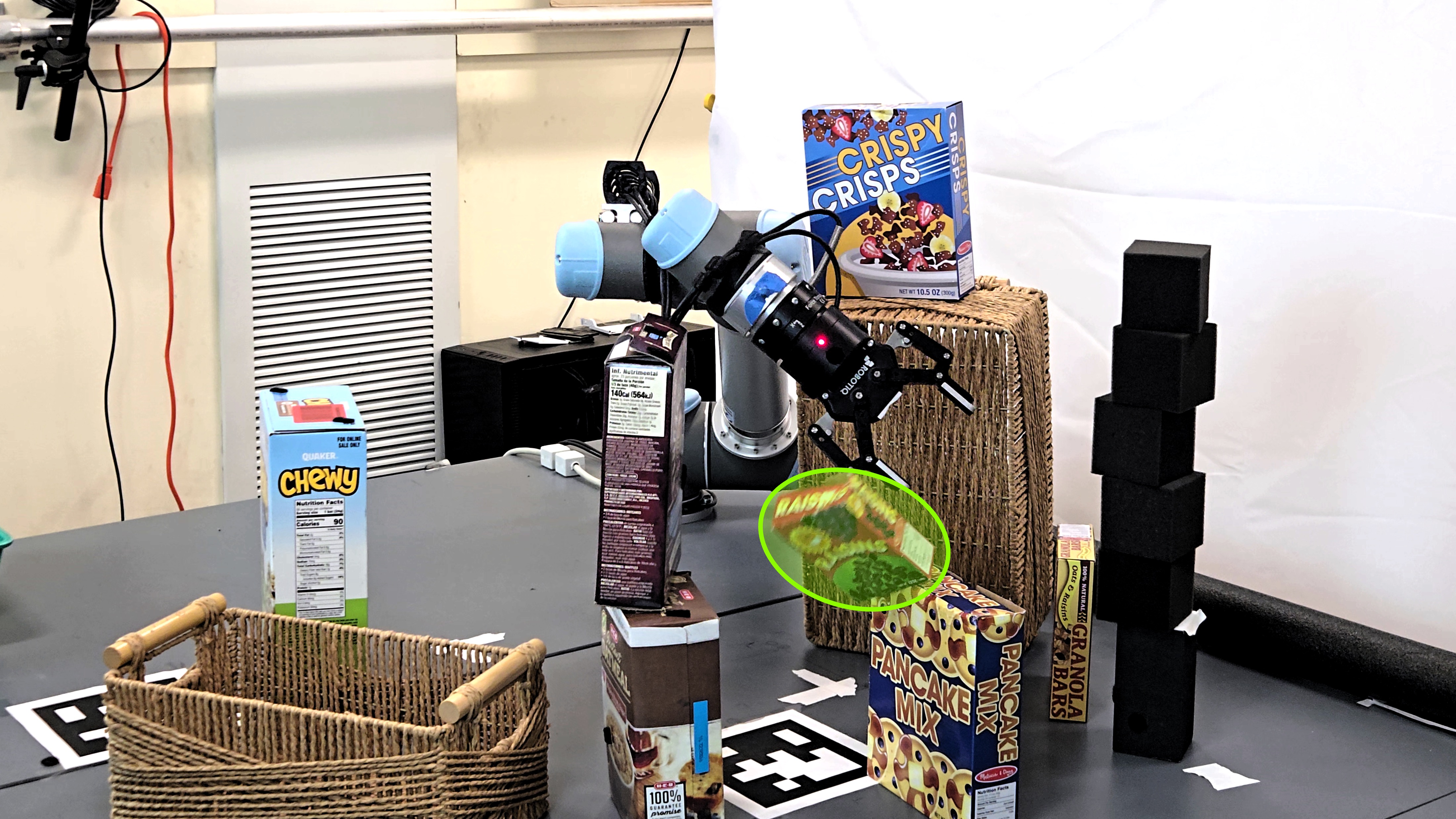}%
\includegraphics[width=0.49\textwidth]{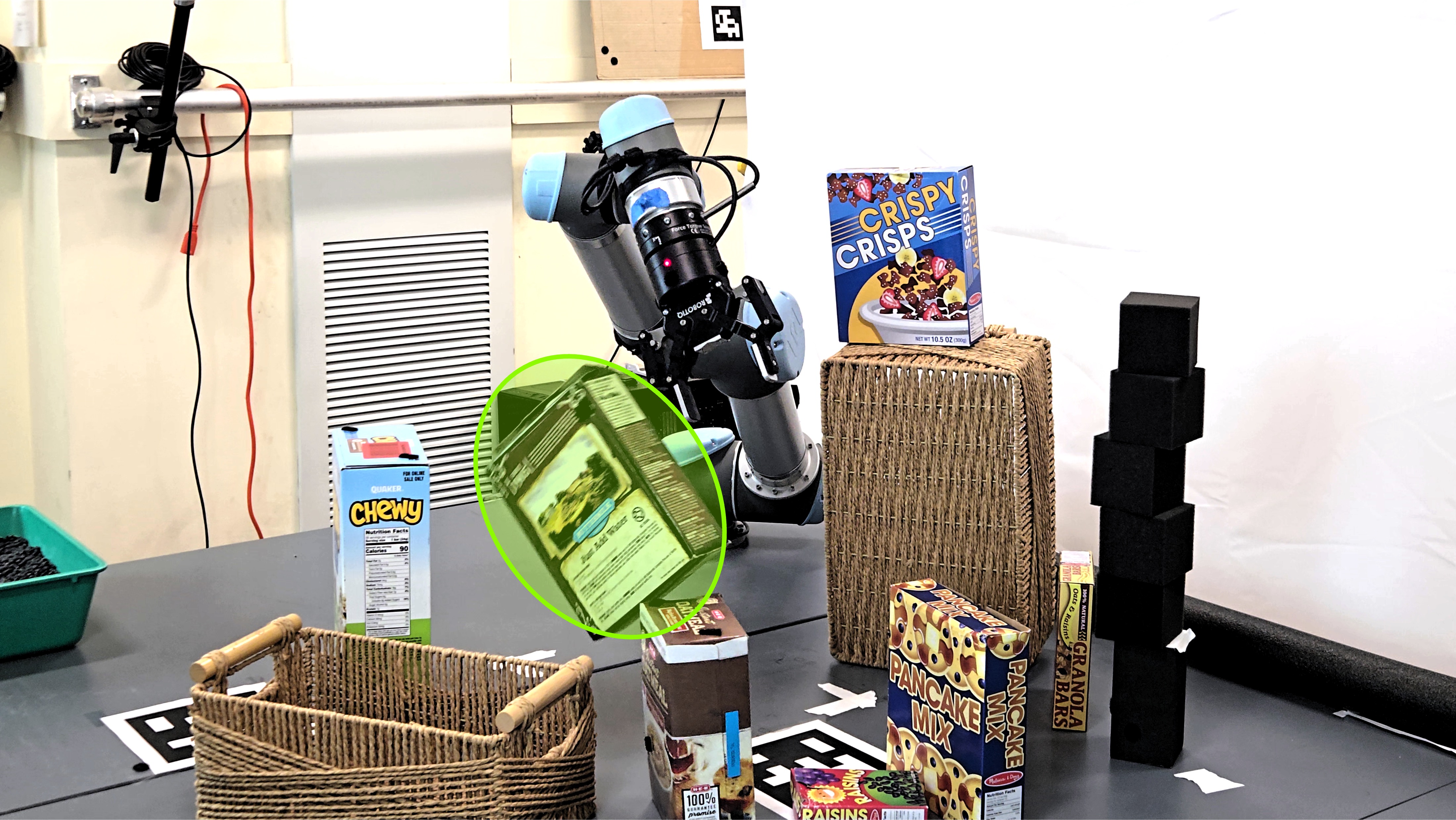}%
\caption{}
\label{fig:ur5_pickplace_geo}
\end{subfigure}%
\caption{The ``pick and place" task with UR5 robot in a cluttered environment with narrow passages: (a) our kinodynamic planner successfully finishes the task as it generates a smooth and dynamically feasible trajectory that the robot can track accurately; (b) the robot collides with obstacles when trying to track a geometric path~\cite{thomason2024vamp} due to overshoots caused by the lack of dynamics constraints.}
\label{fig:ur5_pickplace}
\end{figure*}

\subsection{Real experiments}
Finally, we verify the benefit of our approach for online motion planning on a real Universal Robots (UR5) platform in a cluttered environment. The environment consists of multiple static and dynamic objects, perceived by an Intel Realsense camera with a top-down view. The depth image from the camera is used to generate a set of spheres stored in a collision-affording point tree (CAPT) compatible with \scSIMD-parallelized motion planning~\cite{ramsey2024-capt}. The spheres represent the obstacles in the environment in our \cref{alg:collision_checking_poly}. \mbox{Our \mbox{\FLASKRRTC}} planner in \cref{subsec:compare_lowdim} is used to find a trajectory from the start configuration to the goal.

\subsubsection{Kinodynamic Planning versus Geometric Planning}

\begin{figure}[t]
\centering
\begin{subfigure}[t]{0.5\textwidth}
        \centering
\includegraphics[width=0.48\textwidth]{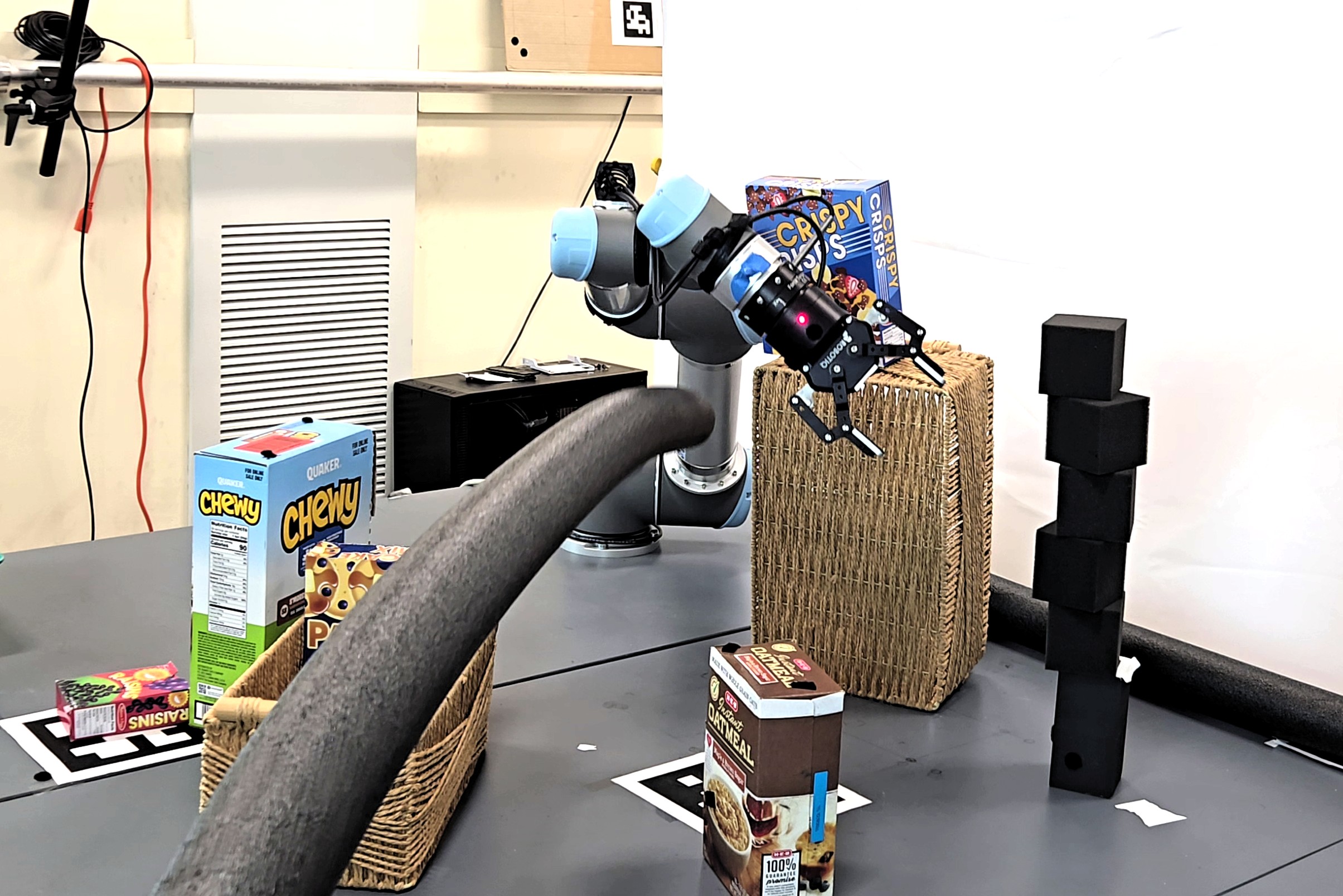}%
\hspace{0.5mm}
% \hfill
\includegraphics[width=0.48\textwidth]{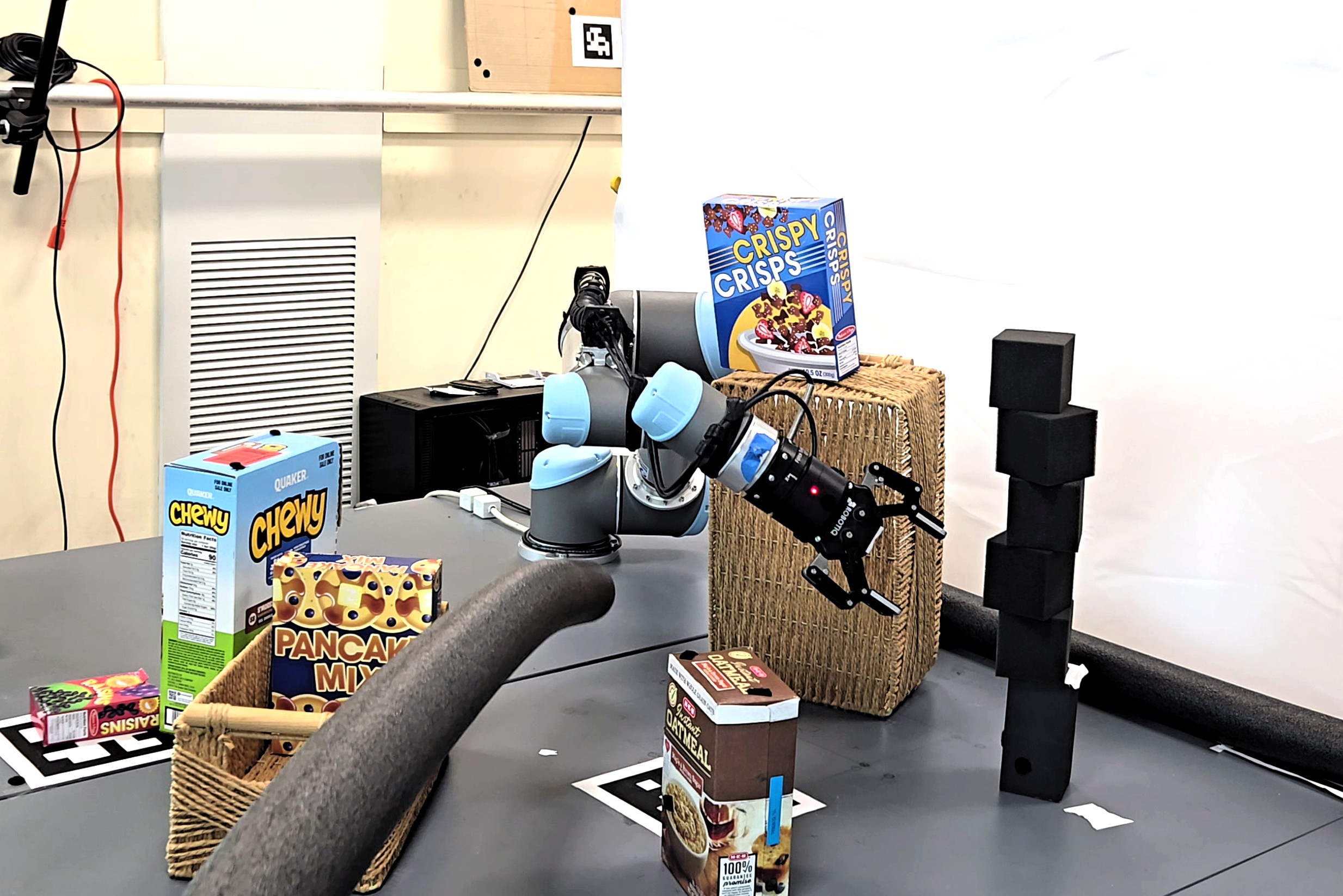}%
\vspace{1mm}
\includegraphics[width=0.48\textwidth]{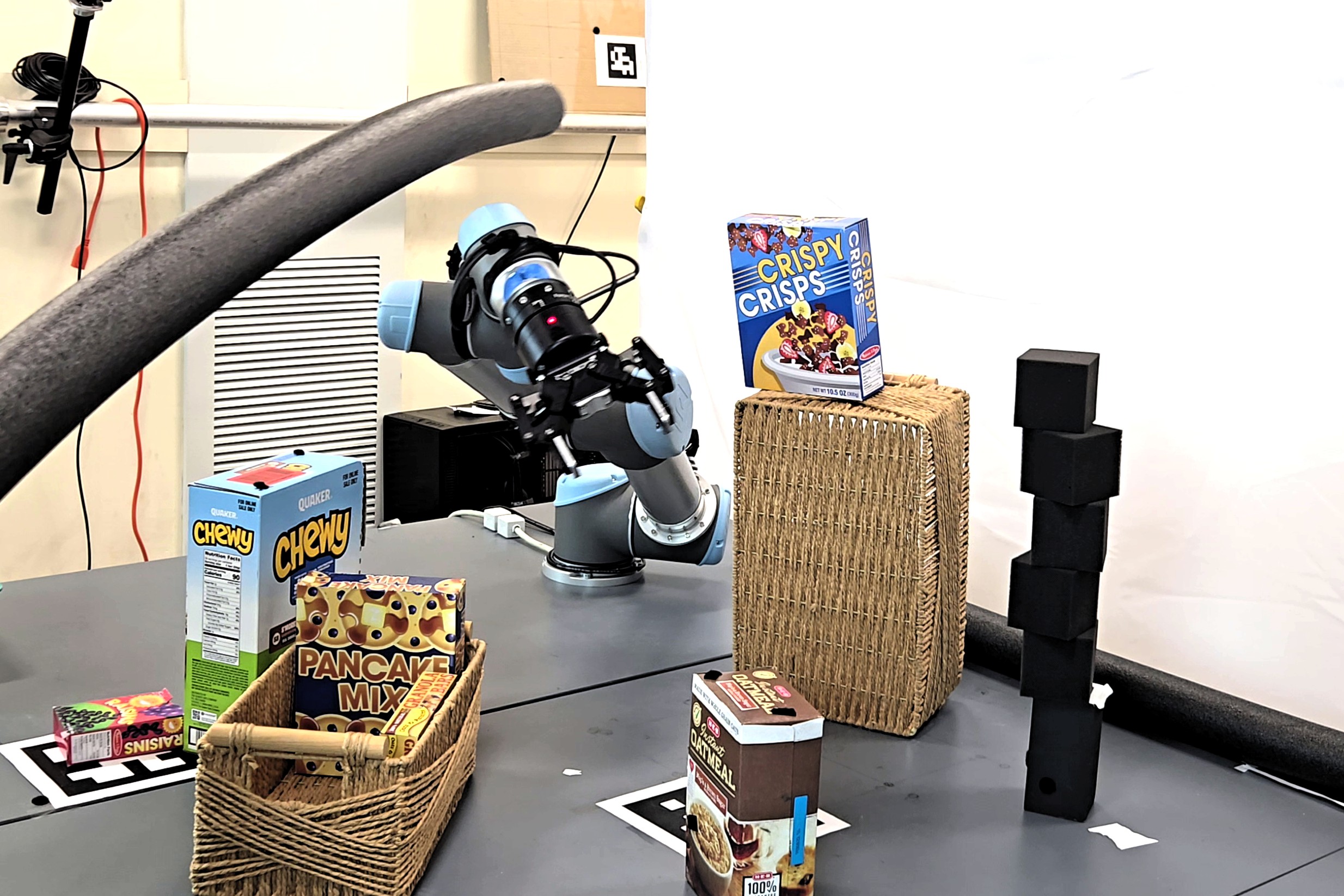}%
\hspace{0.5mm}
% \hfill
\includegraphics[width=0.49\textwidth]{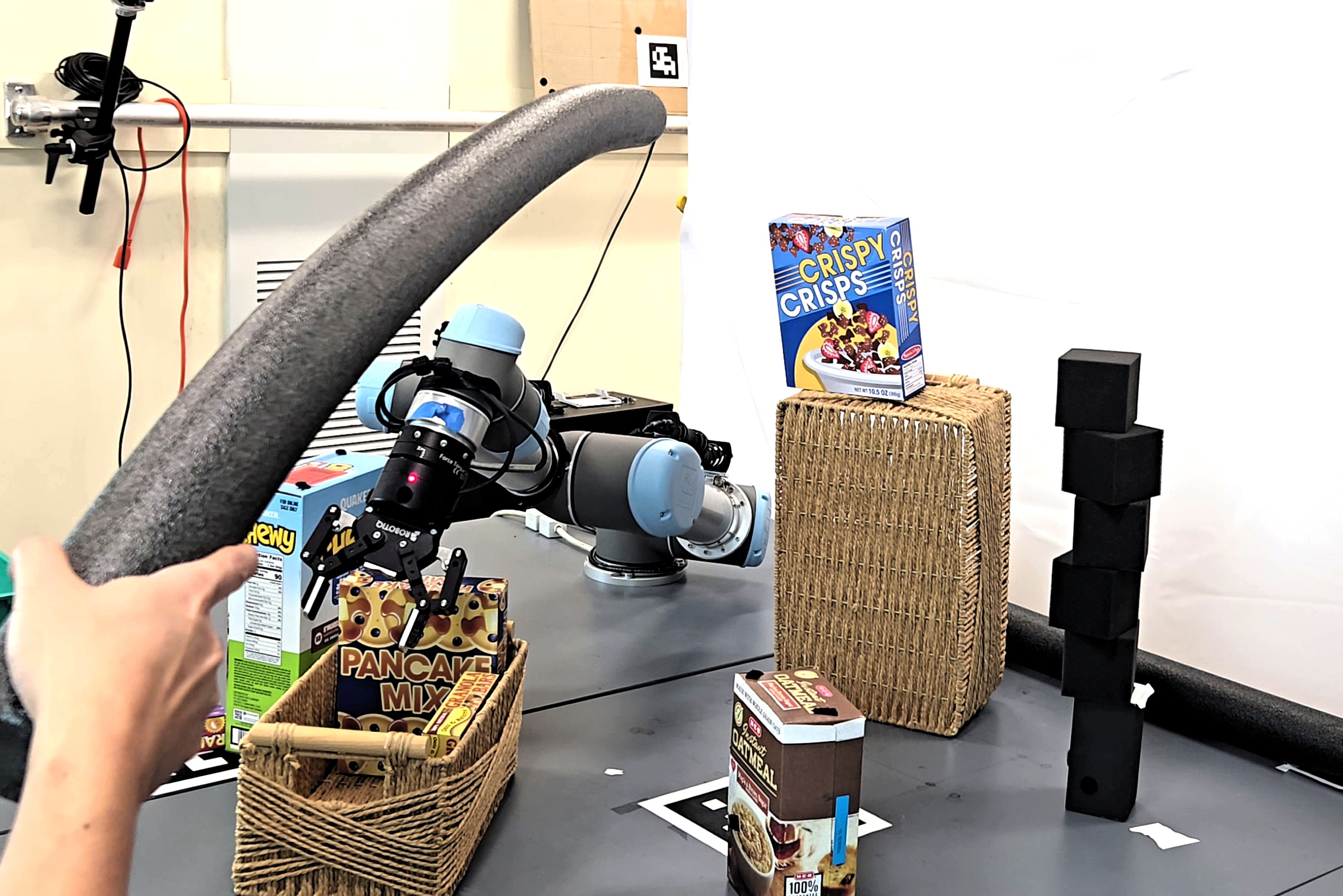}%
\vspace{1mm}
\includegraphics[width=0.48\textwidth]{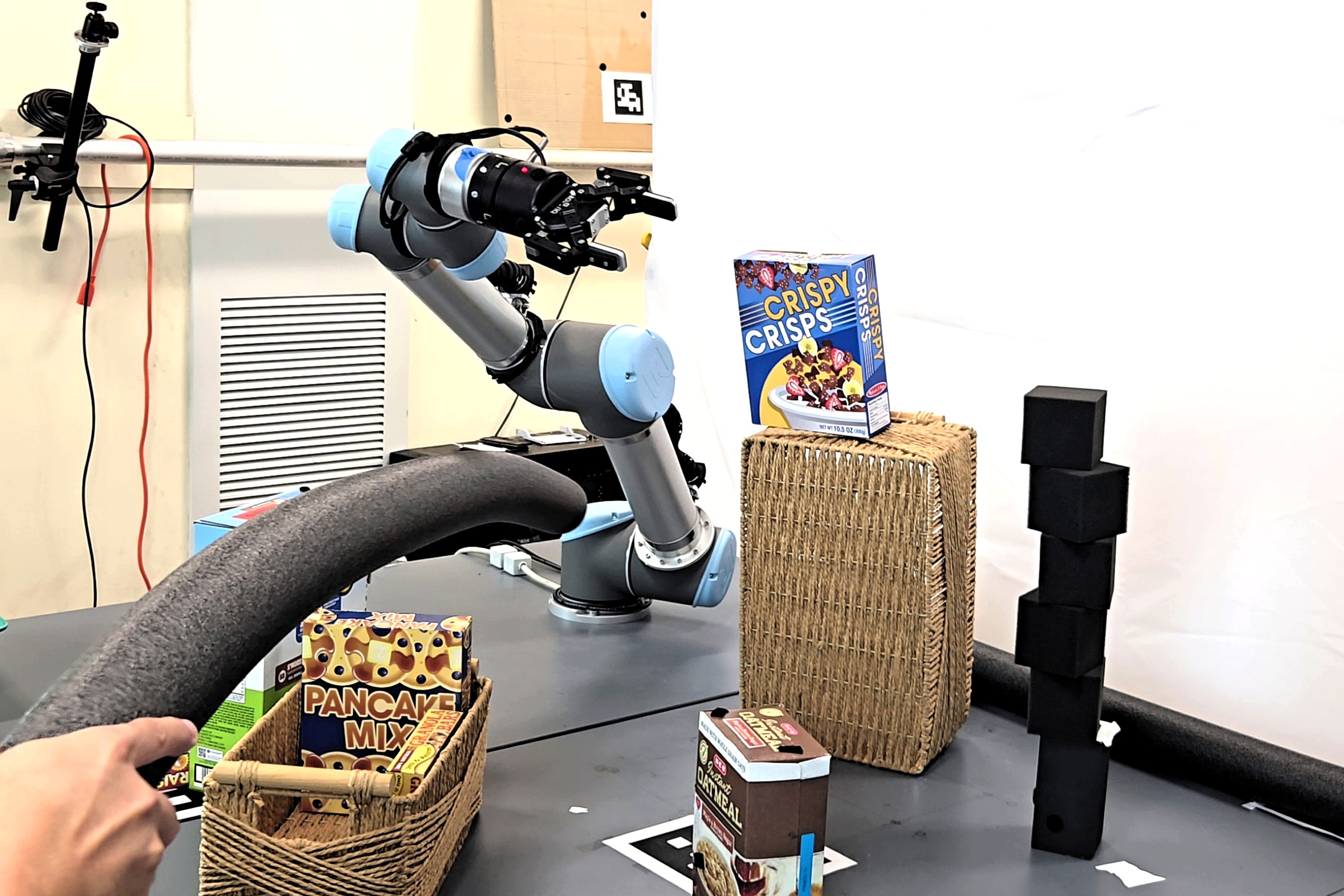}%
\hspace{0.5mm}
% \hfill
\includegraphics[width=0.48\textwidth]{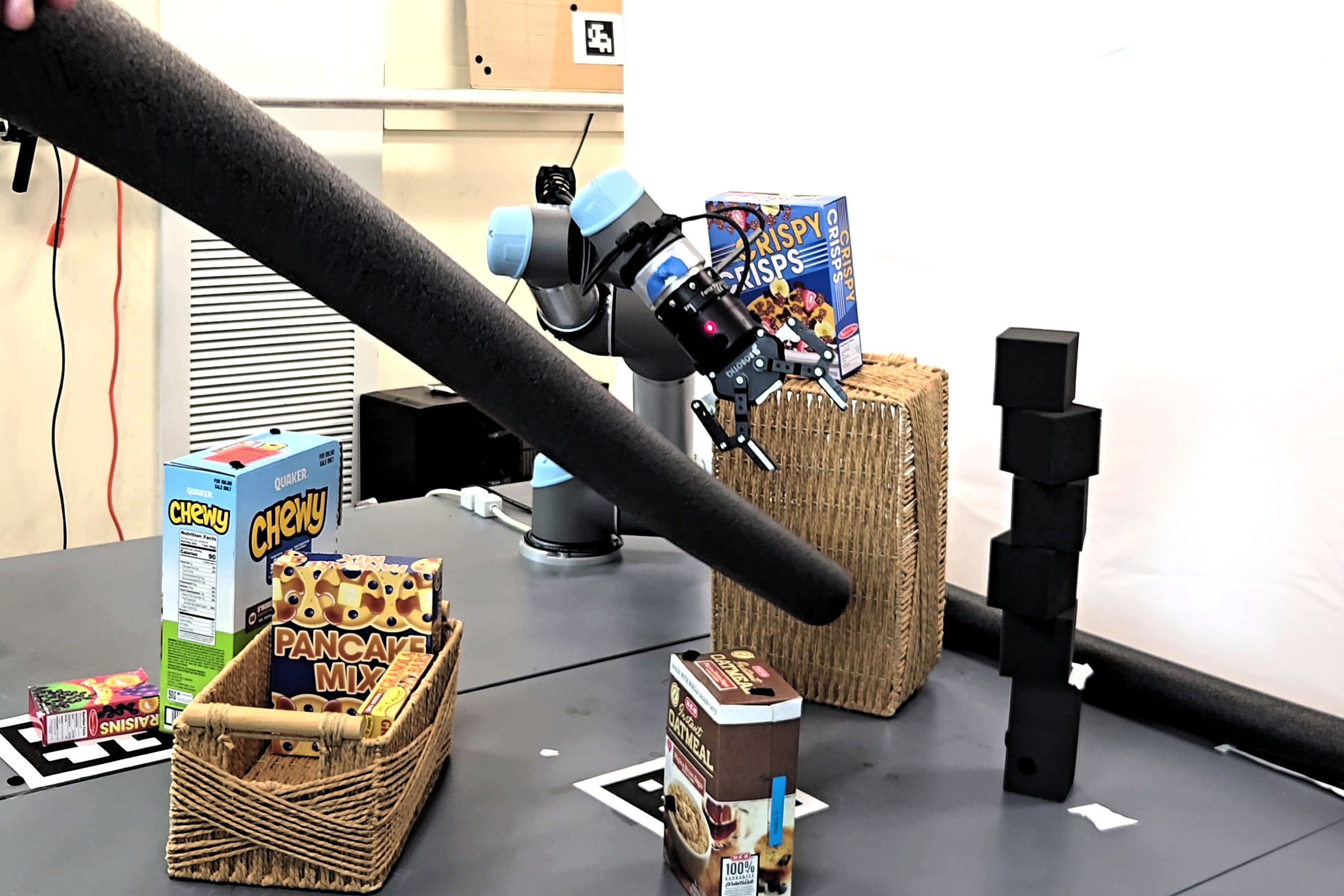}%
\end{subfigure}%
\caption{Reactive planning with moving obstacles: our UR5 robot successfully performs a ``\emph{pick, place, and reset}" loop with our \mbox{\FLASKRRTC} planner without colliding with any of the static or moving obstacles. The obstacles are observed by an overhead Intel Realsense camera. }
\label{fig:ur5_reactive_exp}
\end{figure}
\begin{figure}[t]
\centering
\begin{subfigure}[t]{0.4\textwidth}
        \centering
\includegraphics[width=\textwidth]{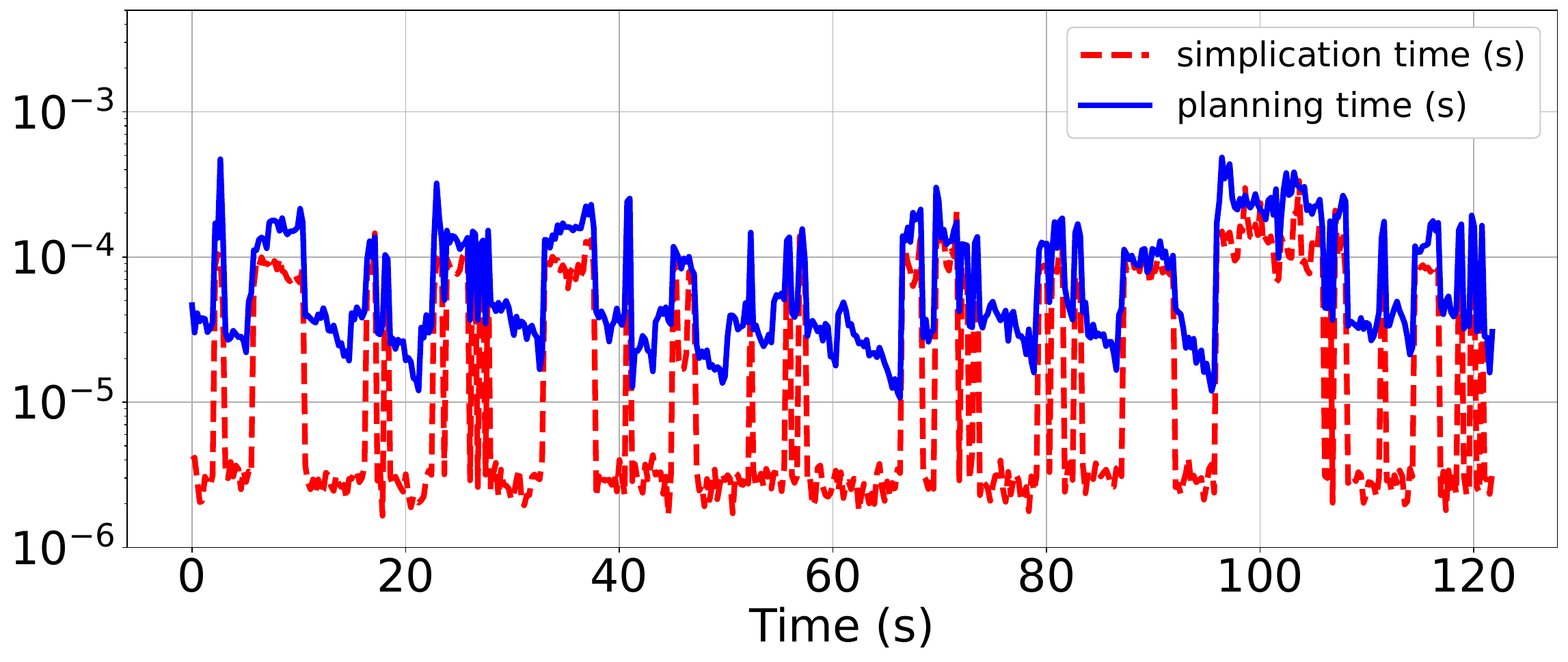}%
\end{subfigure}%
\caption{Planning and simplification time in our ``\emph{pick, place, and reset}" loop: our planner takes $\sim 90\mu s$ and $\sim 40 \mu s$ for planning and simplification, respectively, illustrating the reactiveness of our approach for real-time applications.}
\label{fig:ur5_reactive_exp_time}
\end{figure}

We first consider a ``\emph{pick and place}" scenario in \cref{fig:ur5_pickandplace_annotated}, where the robot generates a motion plan from its start position to a~\emph{``pick"} position near a granola box, and then moves to a \emph{``place"} position above a basket. The environment is cluttered with multiple objects, creating narrow passages such that a minor deviation from the planned trajectory can cause collisions. Therefore, accurate tracking is crucial for the safe realization of this ``\emph{pick and place}" task.
The baseline is a geometric path, generated by a SIMD-parallelized \mbox{RRTConnect} planner from VAMP~\cite{thomason2024vamp}. The path is sampled with a constant open-loop velocity, chosen such that the task's duration is similar to ours for a fair comparison. A closed-loop velocity control input~is computed by Ruckig~\cite{berscheid2021ruckig} and sent to the UR5 robot to execute.

\cref{fig:ur5_tracking_perf} plots the tracking error during the experiments. Our trajectory is dynamically feasible, allowing the robot to smoothly and accurately execute the task without any collisions (\cref{fig:ur5_tracking_kino,fig:ur5_pickplace_kino}). Meanwhile, the geometric path has sharp turns, leading to overshoots and fluctuations (\cref{fig:ur5_tracking_geo}). As a result, the robot collides with two nearby boxes while trying to turn, as seen in \cref{fig:ur5_pickplace_geo}. The overshoots can be observed in \cref{fig:ur5_tracking_error} with two spikes in total tracking error $\Vert \bfq_{actual} - \bfq_{desired} \Vert$, at \mbox{times $t \approx 4$s and $t \approx 9$s,} where $\bfq_{actual}$ and $\bfq_{desired}$ denote the actual and desired values of the joint angles from the motion plans. Meanwhile, our tracking error stays low without any spikes, highlighting the benefits of having a dynamically feasible trajectory for accurate task executions.

\subsubsection{Reactive Kinodynamic Planning}
To examine the reactiveness of our kinodynamic planning approach, we consider a ``\emph{pick, place, and reset}" loop, where the robot keeps planning its trajectory to perform a ``\emph{pick}" action near a pile of blocks, a ``\emph{place}" action above the basket, and then ``\emph{reset}" to its original configuration. During the experiment, we move a foam obstacle in the environment and verify that our kinodynamic motion planner is able to quickly react to object changes and safely achieve the task as shown in \cref{fig:ur5_reactive_exp}. On average, our kinodynamic planner takes $\sim 90\mu s$ to plan and $\sim 40 \mu s$ to simplify the trajectory (\cref{fig:ur5_reactive_exp_time}), leading to a total planning time of $0.13 ms$. The results illustrate that our approach is fast and suitable for online and reactive planning with dynamic environments, while satisfying the dynamics constraints.

\section{Conclusions}
This paper develops \REVV{\textbf{\FLASK}}, an ultrafast sampling-based kinodynamic planning \REVV{framework}, by solving the two-point boundary value problem and dynamics propagation in closed forms for a common class of differentially flat robot platforms, and performing extremely fast forward kinematics and collision checking via \scSIMD parallelism, available on most consumer CPUs. Our approach is exact, general and \REV{can be applied to any sampling-based planners while offering theoretical guarantees on probabilistic exhaustivity and asymptotic optimality}. It is able to generate a dynamically feasible trajectory in the range of microseconds to milliseconds, which is suitable for online and reactive planning in dynamic environments. Our method outperforms common sampling-based and optimization-based kinodynamic planners as well as time-parameterization of a geometric path in terms of planning times, offering a fast, reliable, and feasible solution for trajectory generation.

Our approach unfolds a promising and exciting research area where it is possible to transform existing sampling-based motion planners \REV{with theoretical guarantees} into fast kinodynamic versions while only requiring general-purpose and widely available CPUs. \REV{As we have closed-form BVP solutions for a large class of nonlinear differentially flat systems, we can largely bypass the propagation-based planners and can even enable roadmap-based kinodynamic planning}. This ability potentially can lead to a wide application of our method on different robots, in different settings, and for various tasks. \REV{Our method can quickly bootstrap optimization-based planners with quality and dynamically feasible initial solutions, especially when their objective function differs from our cost.} \REV{Moreover, the ability to specify a duration for our trajectory offers an exciting potential integration task and motion planning with temporal constraints such as task deadlines.}

\section{Acknowledgement}
We would like to thank J. Arden Knoll, Emiliano Flores, and Yitian Gao for their help with the baselines.
\section{Appendices}
\label{sec:appendix}
\subsection{Derivation of closed-form solutions in Example \ref{example:min_acc_time}}
\label{subsec:min_time_derivation}
The optimal control input \eqref{eq:optimal_control} becomes:
\begin{equation*}
\begin{aligned}
    \bfw_{loc}(t) 
    &= \begin{bmatrix}
    \bf0 & \bfI_n
\end{bmatrix} \left( \bfI_{2n} + \bfA^T (T-t) \right) \bfG^{-1}_T \bfd_T \\
    &=  \scaleMathLine[0.7]{\begin{bmatrix}
    \bf0 & \bfI_n
\end{bmatrix} \begin{bmatrix}
                    \bfI_n & \bf0 \\
                    (T-t)\bfI_n & \bfI_n
                \end{bmatrix} \begin{bmatrix}
                    \frac{12}{T^3}\bfI_n & -\frac{6}{T^2}\bfI_n \\
                    -\frac{6}{T^2}\bfI_n & \frac{4}{T}\bfI_n
                \end{bmatrix} \bfd_T}, \\
    &=\begin{bmatrix}
        -\frac{12\bfI_n}{T^3}  &  \frac{6\bfI_n}{T^2}
    \end{bmatrix} \bfd_T t + \begin{bmatrix}
        \frac{6\bfI_n}{T^2} &  - \frac{2\bfI_n}{T}
    \end{bmatrix} \bfd_T.
\end{aligned}
\end{equation*}
This leads to a minimum-acceleration optimal motion $\bfz_{loc}(t)$: 
\begin{equation*}%\label{eq:optimal_acc_traj_detailed}
\begin{aligned}
        \bfz_{loc}(t) 
& = \scaleMathLine[0.85]{\begin{bmatrix}
                    \frac{t^3}{3}\bfI_n & \frac{t^2}{2}\bfI_n \\
                    \frac{t^2}{2}\bfI_n & t\bfI_n
                \end{bmatrix} \begin{bmatrix}
                    \bfI_n & \bf0 \\
                    (T-t)\bfI_n & \bfI_n
                \end{bmatrix} \begin{bmatrix}
                    \frac{12}{T^3}\bfI_n & -\frac{6}{T^2}\bfI_n \\
                    -\frac{6}{T^2}\bfI_n & \frac{4}{T}\bfI_n
                \end{bmatrix} \bfd_T } \\
               & \qquad \qquad \scaleMathLine[0.27]{+ \begin{bmatrix}
    \bfI_n & t\bfI_n \\ \bf0 & \bfI_n 
\end{bmatrix} \begin{bmatrix}
    \bfy_0 \\ \dot{\bfy}_0
\end{bmatrix}}, \\
& = \scaleMathLine[0.85]{\begin{bmatrix}
    \begin{bmatrix}
        -\frac{2\bfI_n}{T^3}  &  \frac{\bfI_n}{T^2}
    \end{bmatrix} \bfd_T t^3 + \begin{bmatrix}
        \frac{3\bfI_n}{T^2} &  - \frac{\bfI_n}{T}
    \end{bmatrix} \bfd_T t^2 + \dot{\bfy}_0 t + \bfy_0 \\
    \begin{bmatrix}
        -\frac{6\bfI_n}{T^3}  &  \frac{3\bfI_n}{T^2}
    \end{bmatrix} \bfd_T t^2 + \begin{bmatrix}
        \frac{6\bfI_n}{T^2} &  - \frac{2\bfI_n}{T}
    \end{bmatrix} \bfd_T t + \dot{\bfy}_0
\end{bmatrix}.}
\end{aligned}
\end{equation*}
In other words, the optimal flat output trajectory is: 
\begin{equation*} %\label{eq:optimal_acc_traj}
    \bfy_{loc}(t) = \begin{bmatrix}
        -\frac{2\bfI_n}{T^3}  &  \frac{\bfI_n}{T^2}
    \end{bmatrix} \bfd_T t^3 + \begin{bmatrix}
        \frac{3\bfI_n}{T^2} &  - \frac{\bfI_n}{T}
    \end{bmatrix} \bfd_T t^2 + \dot{\bfy}_0 t + \bfy_0.
\end{equation*}

The optimal cost function \eqref{eq:optimal_cost} is expanded as:
\begin{equation*}
\begin{aligned}
\calC_{loc}(T) 
    &= \scaleMathLine[0.85]{\begin{bmatrix}
    \bfy_f - \bfy_0 - T\dot{\bfy}_0 \\ \dot{\bfy}_f - \dot{\bfy}_0
\end{bmatrix}^\top \begin{bmatrix}
                    \frac{12}{T^3}\bfI_n & -\frac{6}{T^2}\bfI_n \\
                    -\frac{6}{T^2}\bfI_n & \frac{4}{T}\bfI_n
                \end{bmatrix} \begin{bmatrix}
    \bfy_f - \bfy_0 - T\dot{\bfy}_0 \\ \dot{\bfy}_f - \dot{\bfy}_0
\end{bmatrix} + \rho T} \\
    &= \scaleMathLine[0.75]{12\Vert \bfy_f - \bfy_0 \Vert^2/T^3 - 12 (\dot{\bfy}_f + \dot{\bfy}_0)^\top (\bfy_f - \bfy_0)/T^2}\\
    & \qquad \qquad + \scaleMathLine[0.56]{4(\Vert \dot{\bfy}_0 \Vert^2 + \dot{\bfy}_0^\top \dot{\bfy}_f + \Vert \dot{\bfy}_f \Vert^2)/T + \rho T}.
\end{aligned}
\end{equation*}

%%%%%%%%%%%%%%%%%%%%%%%%%%%%%%%%%%%%%%%%%%%%%%%%%%%%%%%%%%%%%%%%%%%%%%%%%%%%%%%%

%%%%%%%%%%%%%%%%%%%%%%%%%%%%%%%%%%%%%%%%%%%%%%%%%%%%%%%%%%%%%%%%%%%%%%%%%%%%%%%%
\printbibliography{}

\begin{IEEEbiography}[{\includegraphics[width=1in,height=1.25in,clip,keepaspectratio]{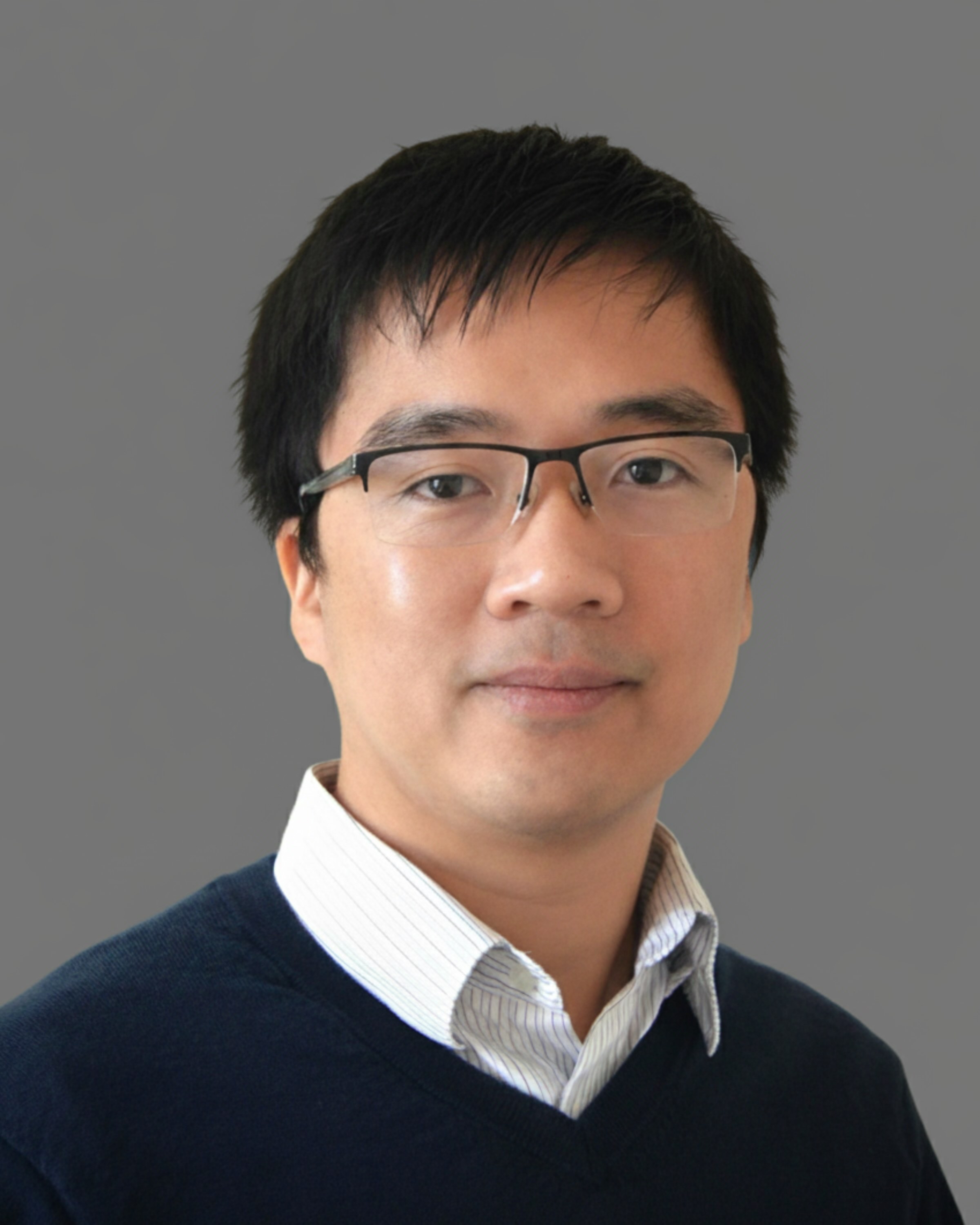}}]{Thai Duong}
	received his Ph.D. degree in Electrical and Computer Engineering from the University of California San Diego, La Jolla, CA, USA, in 2024. 

    He is currently a postdoctoral research associate at Rice University, Houston, TX, USA.
    His research interests include machine learning with applications to robotics, robot dynamics learning, motion planning and control.
\end{IEEEbiography}

\begin{IEEEbiography}[{\includegraphics[width=1in,height=1.25in,clip,keepaspectratio]{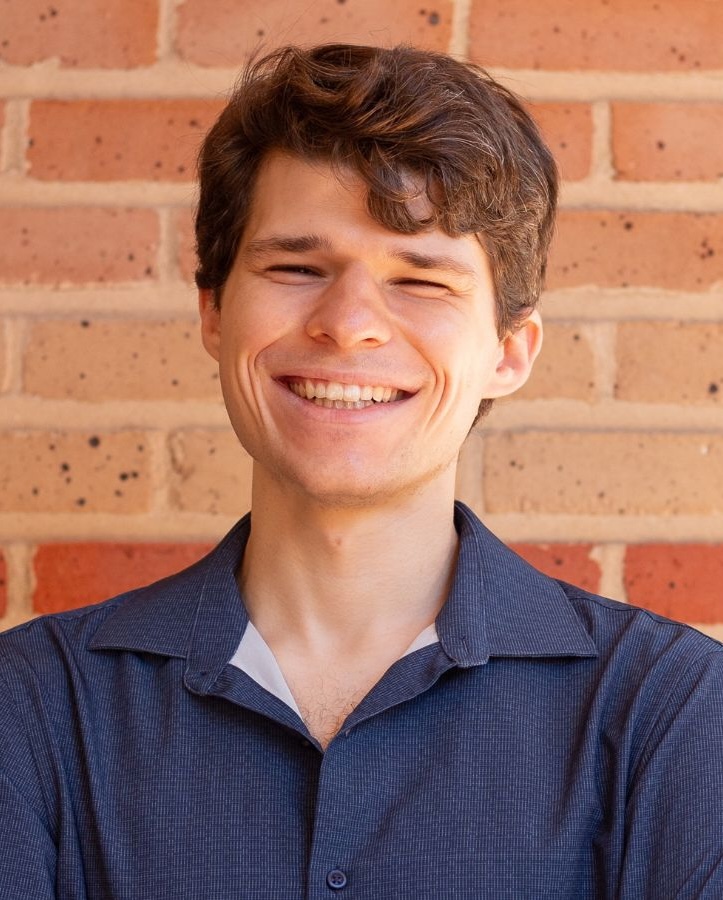}}]{Clayton W. Ramsey} is a 4th-year Ph.D. student under Dr. Lydia E. Kavraki at Rice University. 

He is a NASA NSTGRO Research Fellow, collaborating with the Dexterous Robotics team at NASA Johnson Space Center. His research specializes in hardware acceleration for robot planning, especially for long-horizon planning in dynamic environments.
\end{IEEEbiography}

\begin{IEEEbiography}[{\includegraphics[width=1in,height=1.25in,clip,keepaspectratio]{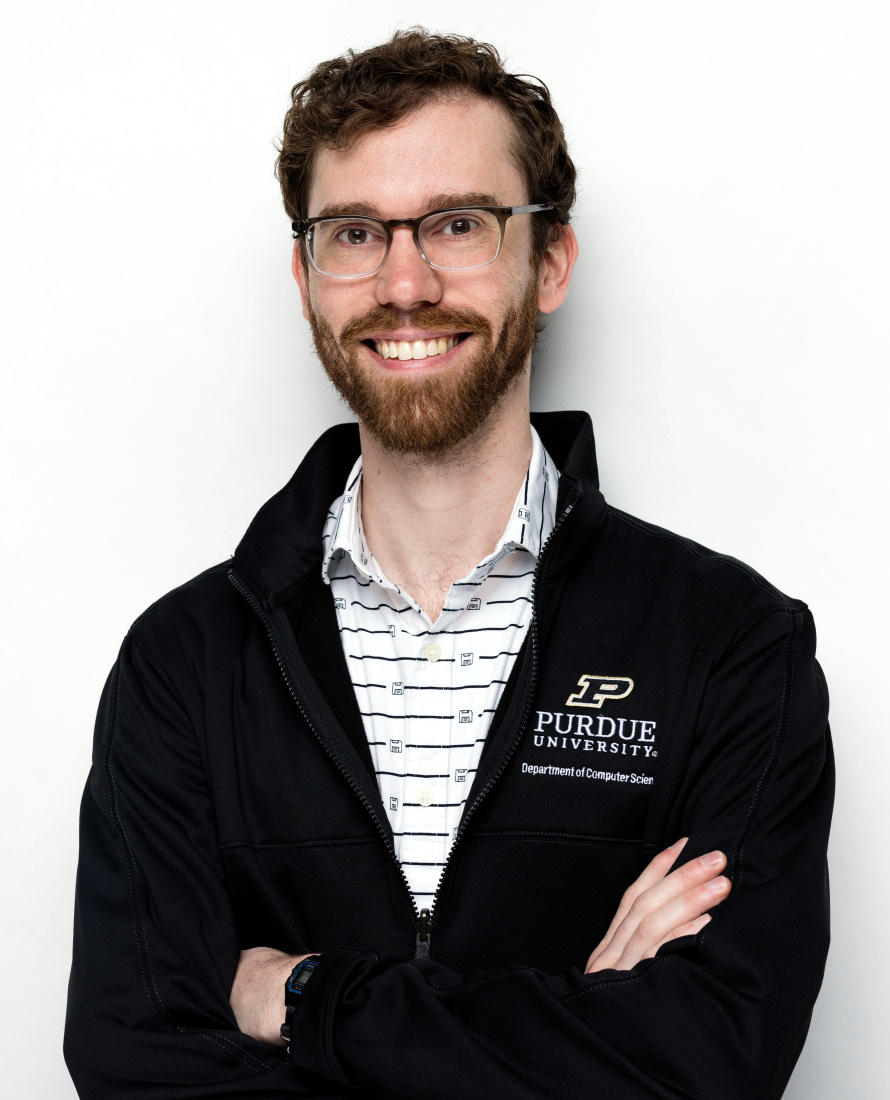}}]{Zachary Kingston} received his Ph.D. degree in Computer Science from Rice University, Houston, TX, USA, in 2021 under the supervision of Dr. Lydia E. Kavraki.

He is currently an Assistant Professor of Computer Science at Purdue University, West Lafayette, IN, USA. His research interests include robot motion planning, hardware-accelerated planning, and planning for systems under constraints.
\end{IEEEbiography}

\begin{IEEEbiography}[{\includegraphics[width=1in,height=1.25in,clip,keepaspectratio]{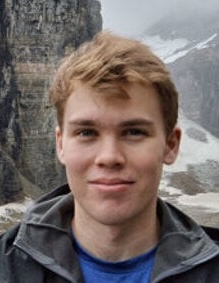}}]{Wil Thomason} received the Ph.D. degree in Computer Science from Cornell University, Ithaca, NY, USA, in 2021.

He was funded by the National Defense Science and Engineering Graduate Fellowship for his Ph.D. He is a research scientist at the Robotics and AI Institute. He was a Postdoctoral Computing Research Association/Computing Community Consortium Computing Innovation Fellow working under~the direction of Dr. Kavraki with the Department of Computer Science, Rice University, Houston, TX, USA.
\end{IEEEbiography}

\begin{IEEEbiography}[{\includegraphics[width=1in,height=1.25in,clip,keepaspectratio]{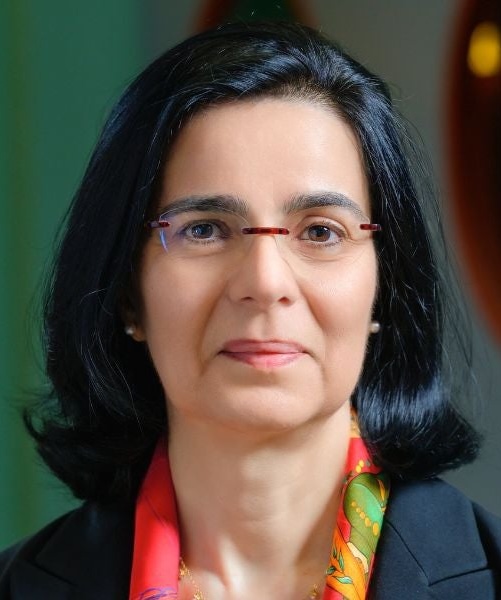}}]{Lydia E. Kavraki} (Fellow, IEEE) received the Ph.D. degree in computer science from Stanford University, Stanford, CA, USA, in 1995.

She is the Kenneth and Audrey Kennedy University Professor of Computing at Rice University, Houston, TX, USA, and Director of the Ken Kennedy Institute. She coauthored \textit{Principles of Robot Motion} (MIT Press, 2005), pioneered sampling-based motion planning, and leads development of the Open Motion Planning Library. Her research interests include motion planning, human-robot collaboration, and robotics/AI methods for computational biomedicine.

Dr. Kavraki is a Fellow of AAAI, AIMBE, ACM, and AAAS, and a Member of the National Academies of Engineering, Sciences, and Medicine.
\end{IEEEbiography}

%%%%%%%%%%%%%%%%%%%%%%%%%%%%%%%%%%%%%%%%%%%%%%%%%%%%%%%%%%%%%%%%%%%%%%%%%%%%%%%%
\end{document}